\documentclass[10pt,journal,compsoc]{IEEEtran}
\usepackage{subfigure}

\usepackage{booktabs}
\usepackage{enumitem}
\usepackage{makecell}
\usepackage{algorithm}
\usepackage{amsfonts, amssymb}
\usepackage{subfiles}
\usepackage{algorithmic}
\usepackage{mathtools}
\usepackage{amsthm}
\usepackage{bm}
\usepackage{graphicx}
\usepackage{multirow}
\usepackage{amssymb} 
\usepackage{ulem}
\usepackage{hyperref}
\usepackage{cuted}
\usepackage[utf8]{inputenc}
\usepackage{ragged2e}
\usepackage[english]{babel}
\newtheorem{theorem}{Theorem}

\newtheorem{proposition}{Proposition}

\hypersetup{hidelinks,
	colorlinks=true,
	allcolors=black,
	pdfstartview=Fit,
	breaklinks=true}

\usepackage{color}

\ifCLASSOPTIONcompsoc
  \usepackage[nocompress]{cite}
\else
  \usepackage{cite}
\fi
\ifCLASSINFOpdf
\else
\fi
\begin{document}
%
\title{Identifying Semantic Component for Robust Molecular Property Prediction}
%
%
%
%

\author{Zijian Li,
        Zunhong Xu, 
        Ruichu Cai*,Zhenhui Yang, Yuguang Yan, Zhifeng Hao ~\IEEEmembership{Senior Member,~IEEE,} Guangyi Chen and Kun Zhang
\IEEEcompsocitemizethanks{\IEEEcompsocthanksitem
Zijian Li is with the School of Computing, Guangdong University of Technology, Guangzhou China, 510006, and Mohamed bin Zayed University of Artificial Intelligence (MBZUAI), Abu Dhabi, UAE.\protect
E-mail: leizigin@gmail.com\protect
\IEEEcompsocthanksitem Zunhong Xu is with the School of Computer Science, Guangdong University of Technology, Guangzhou China, 510006.\protect Email: zunhongxu@gmail.com
\IEEEcompsocthanksitem Ruichu Cai is with the School of Computer Science, Guangdong University of Technology, Guangzhou, China, 510006 and Peng Cheng Laboratory, Shenzhen, China, 518066.
Email: cairuichu@gmail.com\protect
\IEEEcompsocthanksitem Zhenhui Yang is with the School of Computer Science, Guangdong University of Technology, Guangzhou China, 510006.\protect Email: zhenhuiyang@gmail.com\protect
\IEEEcompsocthanksitem Yuguang Yan is with the School of Computer Science, Guangdong University of Technology, Guangzhou China, 510006.\protect Email: ygyan@outlook.com\protect
\IEEEcompsocthanksitem Zhifeng Hao is with the School of Computer Science, Guangdong University of Technology, Guangzhou China, 510006.\protect Email: ZhifengHao@gmail.com\protect
\IEEEcompsocthanksitem Guangyi Chen is with Carnegie Mellon University, Pittsburgh, USA, and Mohamed bin Zayed University of Artificial Intelligence (MBZUAI), Abu Dhabi, UAE.\protect Email: guangyichen1994@gmail.com\protect
\IEEEcompsocthanksitem Kun Zhang is with Carnegie Mellon University, Pittsburgh, USA, and Mohamed bin Zayed University of Artificial Intelligence (MBZUAI), Abu Dhabi, UAE. E-mail: kunz1@cmu.edu\protect
}
\thanks{*Ruichu Cai is the Corresponding author.}}

%
%

\markboth{Journal of \LaTeX\ Class Files,~Vol.~14, No.~8, August~2015}%
{Shell \MakeLowercase{\textit{et al.}}: Bare Demo of IEEEtran.cls for Computer Society Journals}
%



\IEEEtitleabstractindextext{%
\begin{abstract}\justifying
Although graph neural networks have achieved great success in the task of molecular property prediction in recent years, their generalization ability under out-of-distribution (OOD) settings is still under-explored. Different from existing methods that learn discriminative representations for prediction, we propose a generative model with semantic-components identifiability, named \textbf{SCI}. We demonstrate that the latent variables in this generative model can be explicitly identified into semantic-relevant (SR) and semantic-irrelevant (SI) components, which contributes to better OOD generalization by involving minimal change properties of causal mechanisms.
Specifically, we first formulate the data generation process from the atom level to the molecular level, where the latent space is split into SI substructures, SR substructures, and SR atom variables. Sequentially, to reduce misidentification, we restrict the minimal changes of the SR atom variables and add a semantic latent substructure regularization to mitigate the variance of the SR substructure under augmented domain changes. Under mild assumptions, we prove the block-wise identifiability of the SR substructure and the comment-wise identifiability of SR atom variables.
Experimental studies achieve state-of-the-art performance and show general improvement on 21 datasets in 3 mainstream benchmarks. Moreover, the visualization results of the proposed \textbf{SCI} method provide insightful case studies and explanations for the prediction results. The code is available at: https://github.com/DMIRLAB-Group/SCI.

\end{abstract}

\begin{IEEEkeywords}
Molecular Property Prediction, Causal Generation Process, Identification, Graph Out-Of-Distribution\end{IEEEkeywords}}

\maketitle

\IEEEdisplaynontitleabstractindextext

%
\IEEEpeerreviewmaketitle

\IEEEraisesectionheading{\section{Introduction}\label{sec:introduction}}

\label{submission}
\begin{figure*}[h]
		\centering
  \includegraphics[width=2.0\columnwidth]{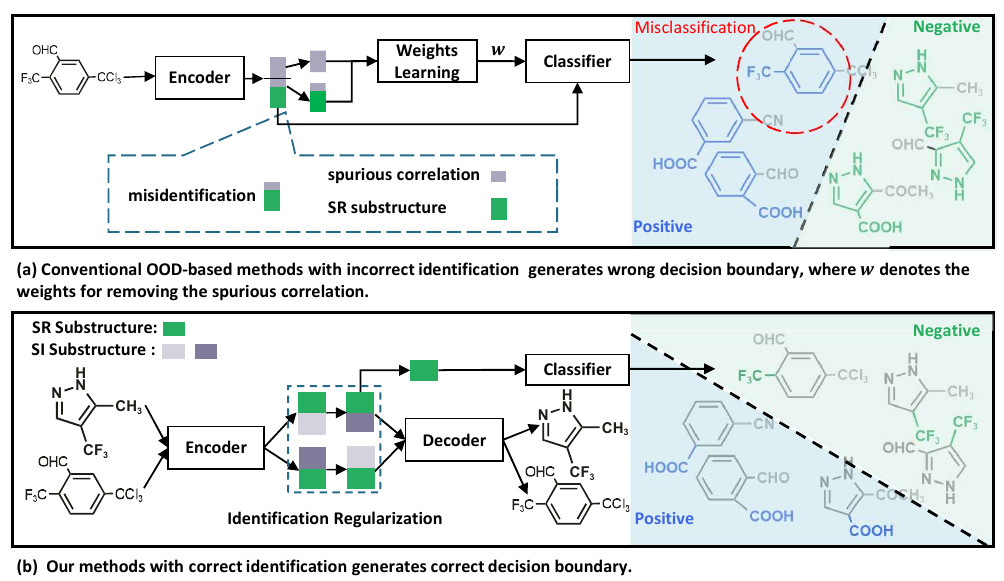}
		\caption{A molecule classification example with the functional groups $(\small{\textbf{-COOH}}_2)$ and $(\small{\textbf{-CF}}_3)$. (a) The discriminative model considers the spurious correlation substructures into the SR substructures and generates the wrong decision boundary that judges if a molecule contains ``Benzene Rings''. (b) The generative model precisely identifies the SR substructures and enjoys ideal generalization performance. 
  \textit{(Best view in color.)}
  }
		\label{fig:model_1}
\end{figure*}

Artificial intelligence has an immense impact on the fields of science via dramatically effective computational algorithms and insightful experimental results \cite{jumper2021highly, atz2021geometric,de2019synthetic,xia2022pre}. As one of the most important cases, molecular property prediction has achieved pioneering applications in various fields such as drug discovery \cite{vamathevan2019applications,ji2023drugood,lu2022tankbind,liu2023interpretable,xu2023graph} and chemical synthesis \cite{butler2018machine}. 

%
In the field of chemistry, density functional theory \cite{neese2009prediction} is a traditional method for molecular prediction, but it is computationally costly. 
With the excessive computation power and the eye-catching performance of deep learning, the GNN-based methods \cite{velivckovic2017graph,battaglia2018relational,lu2019molecular, li2022geomgcl,hu2020pretraining,rong2020self} have become the mainstream methods, which
treat a molecule as a topological graph by treating atoms as nodes. \textcolor{black}{For example, DiffPOOL \cite{ying2018hierarchical} leverages a differentiable graph pooling module to extract the molecule-level representation for property prediction. However, since the molecules are usually sampled from different environments, this leads to the distribution shift between the training and test datasets. To solve this problem, typical methodologies explored in the literature are to extract the invariant disentangled representation for molecular data. Li et.al \cite{li2022ood} leverage the random Fourier features to decorrelate the graph representation. And Liu et.al \cite{liu2022graph} utilize the environment-based augmentation to address the graph OOD challenges, but it cannot theoretically guarantee that the environment or the invariant substructures are correct i.e., identifiable. Recently, several researchers have addressed the graph OOD \cite{li2022out} and molecular property prediction problem from the perspective of causality. Fan et.al \cite{fan2023generalizing} use the technique of stable learning and mitigate the bad influence of distribution shifts on molecular data. 
Fang et.al \cite{fang2022invariant} leverage the prior causal structures and extract the invariant causal substructure with the help of risk extrapolation. 
}

These graph OOD-based methods essentially aim to encourage the independence between SI and SR by \textcolor{black}{modeling the data generation process and reconstructing the latent variables}, but they implicitly assume that the latent variables have been well identified. 
Without identification guarantees, some causality-based methods \cite{wu2022discovering,fan2023generalizing}, which conduct \textit{do-calculus} \cite{pearl2009causality} on latent variables, might result in inaccurate intervention distribution estimation and further poor generalization performance. 
To illustrate this, Figure \ref{fig:model_1} (a) shows that a typical causality-based model without identification guarantees of SR may generate a wrong decision boundary when the spurious correlation substructures are falsely identified as the semantic-relevant components (i.e., the Benzene Rings and the truth SR are identified together). Meanwhile, as shown in Figure \ref{fig:model_1} (b), we can obtain ideal generalization performance by explicitly identifying the semantic relevant (SR) function groups (i.e., $\small{\textbf{-CF}}_3$) with a semantic-relevant substructure identification regularization. Therefore, it is necessary to identify the SR components for robust molecular property prediction.

To achieve identifiability, we propose a \textbf{S}emantic-\textbf{C}omponent \textbf{I}dentifiable model (\textbf{SCI}) to mitigate the negative effect of distribution shift by explicitly disentangling the latent SR and SI components as well as involving the minimal changes property of causal mechanism. 
Technologically, we first formulate the generation process of molecule data in a hierarchical manner, where the latent space between atoms and molecules is split into SI substructures, SR substructures, and SR atom variables. 
Based on this generation process, we reformulate the OOD generalization problem of molecular property prediction as an identifiability problem of latent semantic components, i.e., SR substructures and SR atom variables. 
Sequentially, we identify the distribution of SR atom variables by modeling the data generation process with variational inference with minimal changes property and identify the distribution of the SR substructures by using a data-augmented invariant regularization. 
Under mild assumptions, we theoretically prove that the SR atom variables are component-wise identifiable and the SR substructures are block-wise identifiable.
Extensive experimental studies not only demonstrate the state-of-the-art generalization performance on 21 mainstream datasets of molecular property prediction but also provide insightful visualization results of interpretation.

\textcolor{black}{We summarize our key contribution as follows:}
\begin{itemize}
    \item \textcolor{black}{We propose a general causal generation process for molecular data, which includes several latent variables with identification guarantees.}
    \item \textcolor{black}{Based on the theoretical results, we devise the semantic component identification model (\textbf{SCI}), which identifies the semantic-relevant and semantic-irrelevant substructure with identification guarantees..}
    \item \textcolor{black}{We evaluate the proposed method on three mainstream molecular property prediction benchmarks, which contain 21 datasets, and the proposed \textbf{SCI} achieves state-of-the-art performance.}
    \item \textcolor{black}{We show insightful visualization results and case studies, which provide meaningful motivation and potential help for researchers from the field of chemistry.}
\end{itemize}

\textcolor{black}{The rest of the paper is organized as follows. Section 2 reviews the existing studies on molecular property prediction, graph out-of-distribution classification,    and identification of generation models. Then, we propose the data generation process for molecule data in Section 3. Based on the data generation process, we provide the identification guarantees for the latent variables in Section 4. Section 5 shows the implementation details of the proposed semantic-component identification model.  And section
6 shows the experimental results as well as the insightful
visualization. We conclude the paper with future works
discussion in section 7.}

\section{Related Work}
In this section, we review the existing methods of molecular property prediction, OOD on graph data, and identification of generative models.

\subsection{Molecular Property Prediction}
Molecular property prediction \cite{liu2021transferable,zhang2021motif,lee2023exploring,stark20223d,li2022kpgt} is an important research problem in the fields of physics and chemistry \cite{butler2018machine}. 
Since recent years have witnessed the success of GNNs on structural data, several algorithms have applied GNNs to model molecules \cite{wieder2020compact}, which can be categorized into fine-grained and coarse-grained methods according to different abstract scales. 
Fine-grained methods mainly leverage low-order attributes like atoms or bonds. \cite{doi:10.1021/ci100050t} learns the representation from molecular descriptors or chemical fingerprints. Lu et.al propose the MGCN method \cite{lu2019molecular}, which employs the atom, pair-wise, and triple-wise interaction to learn the representation of molecules. 
Recently, Li et.al propose the GeomGCL \cite{li2022geomgcl} to learn the molecular representation with geometric information in a Node-Edge interactive manner. 
The coarse-grained methods mainly leverage high-order information. Considering the hierarchical structures of molecules, Zhang et.al \cite{10.1093/bioinformatics/btab195} develop a fragment-oriented multi-scale graph attention network for molecular property prediction. Hu et.al \cite{hu2020pretraining} introduce graph-level supervised and structural similarity restriction to pre-training GNNs. And Grover \cite{rong2020self} integrates the graph-level information into a Transformer-based framework. However, these methods might fail to address the distribution shift challenges. 
 

\subsection{Out-Of-Distribution Graph classification}
Our work is also related to the graph out-of-distribution problem \cite{fan2023generalizing,yang2022learning,guo2020graseq,liu2022graph,chen2022learning,10027780,cai2021graph}. 
Existing works on out-of-distribution (OOD) \cite{shen2021towards} mainly focus on the fields of computer vision \cite{zhang2021deep,zhang2022multi} and natural language processing \cite{chen2021hiddencut}, but the OOD challenge on graph-structured data receives less attention. Considering that the existing GNNs lack out-of-distribution generalization \cite{li2022graphde,zhao2020uncertainty,liu2023flood,sui2022causal,sun2022does,li2022ood}, Li et. al propose the OOD-GNN \cite{li2021ood} to address the OOD challenge by eliminating the statistical dependence between relevant and irrelevant graph representations. Considering that the spurious correlations lead to the poor generalization of GNNs, Fan et.al propose the StableGNN \cite{fan2023generalizing}, which extracts causal representation for GNNs with the help of stable learning. Aiming to mitigate the selection bias behind graph-structured data, Wu et.al propose the DIR strategy \cite{wu2022discovering} to extract the invariant causal rationales via intervention. These methods essentially employ causal effect estimation to make SI and SR independent, but they cannot theoretically guarantee that the real semantic-relevant information can be identified. \textcolor{black}{Liu et.al \cite{10.1145/3534678.3539347} employs augmentation to improve the robustness and decompose the observed graph into the environment part and the rationale part. Though using the augmentation, it is worth mentioning that our method formulates the causal generation process of the molecular data with different types of latent variables, i.e. the atom latent variables and the semantic-relevant substructures, and benefits from the flexible interaction of these latent variables. Moreover, compared with other causality-based methods, our method enjoys the advantages of the identification guarantees for latent variables.
}

\subsection{Identification of Generative Models}
Causal representation learning is gaining increasing attention as it seeks to provide deeper insights and broader applicability to deep generative models. This approach, as represented by research in \cite{scholkopf2021toward,kumar2017variational,locatello2019challenging,locatello2019disentangling,zheng2022identifiability,trauble2021disentangled}, aims to uncover the underlying factors and model the latent generation process. A conventional method for acquiring causal representation is through independent component analysis (ICA), as illustrated in works like \cite{hyvarinen2002independent,hyvarinen2013independent,zhang2008minimal,zhang2007kernel,xiemulti,comon1994independent}. ICA supposes a linear mixture model for the generative process. 
In recent work, Aapo .et.al  \cite{hyvarinen2016unsupervised,hyvarinen2017nonlinear,hyvarinen2019nonlinear,khemakhem2020variational,halva2021disentangling,halva2020hidden}, have extended identification theories to the nonlinear scenarios through the introduction of auxiliary variables such as domain indexes, time indexes, and class labels. These approaches often rely on the assumption that latent variables exhibit conditional independence and adhere to exponential families. 
Recently, Zhang et.al  \cite{kong2022partial,xiemulti}, have overcome the constraint of the exponential families assumption. They have proposed component-wise identification methods for nonlinear independent component analysis (ICA) involving a specific number of auxiliary variables. Building on these theoretical findings, Yao et.al  \cite{yao2022temporally,yao2021learning} have successfully identified time-delayed latent causal variables and their relationships from sequential data, even in situations involving stationary environments and various distribution shifts. Additionally, Xie et al. \cite{xiemulti} have utilized nonlinear ICA to reconstruct joint distributions of images from diverse domains, while Kong et.al \cite{kong2022partial} have applied the component-wise identification techniques to address challenges in domain adaptation and Li et.al \cite{li2023subspace} extend the component-wise identification to the subspace identification with milder conditions.

\begin{figure}[t]
		\centering
		\includegraphics[width=0.8\columnwidth]{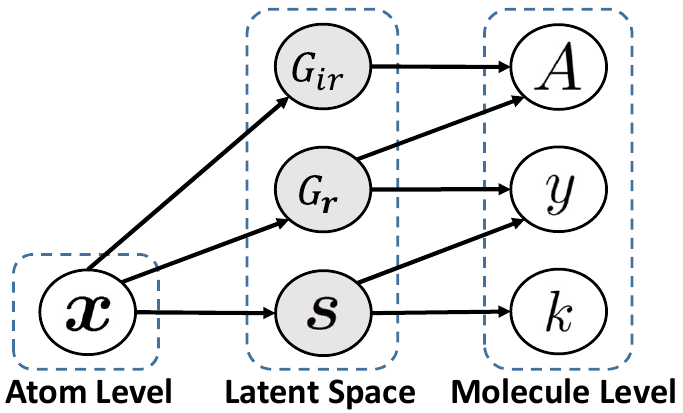}

		\caption{The illustration of the proposed causal generation process for molecule data, connects the atom and molecule views via latent space view. $\bm{s}$, $G_r$ and $G_{ir}$ denote the atom latent variables, semantic-relevant latent substructures, and semantic-irrelevant, respectively.
  }
   \label{fig:gen}
\end{figure}

\section{Data Generation Process for the Molecule Data}
\label{sec:gen}
In the field of quantum chemistry, since the molecular property can be well depicted by the semantic components, we propose a causal generation process by connecting the atom and the molecule views as shown in Figure \ref{fig:gen}.

From the atom view to the latent space, we first let $\bm{x}$ be the observed atom attributes, for example, the atomic chirality, formal charge, and whether the atom is in the ring or not. 
We further let $G_{ir}$, $G_r$ be the semantic-irrelevant substructure (i.e., the adjacency matrix of the semantic-irrelevant substructure) and the semantic-relevant latent substructures, respectively. $\bm{s}$ denote the atom latent variables that encode sufficient low-level information.  
We let $\bm{x}\!\rightarrow \!G_{ir},G_r,\bm{s}$ be the process of how $G_{ir},G_r,\bm{s}$ are generated from $\bm{x}$. 
Specifically, given functions $f_{\bm{s}},f_r,f_{ir}$, we let $\bm{s}=f_{\bm{s}}(\bm{x}), G_r=f_r(\bm{x}),G_{ir}=f_{ir}(\bm{x})$, respectively. 

From the latent space to the molecule, we let $A$ and $y$ be the observed molecule structures and molecule properties, respectively. 
The atomic number $k$ is the supervised signal of $\bm{s}$. 
Since the observed molecule structures are composed of the latent semantic-relevant substructure and the latent semantic-irrelevant substructure, we assume the observed molecule structures $A$ are generated by $G_{ir}$ and $G_r$, which is denoted by $A=g_A(G_{ir}, G_r)$ with a function $g_A$. Since the molecule properties are usually decided by the semantic components, we assume the molecule properties $y$ are controlled by the semantic latent substructures $G_r$ and the atom latent variables $\bm{s}$, which is denoted by $y=g_y(G_r,\bm{s})$. Finally, we further use $k=g_k(\bm{s})$ to represent how the atom number is decided by the node latent variables. 
In summary, we formalize the causal mechanism as follows:
\begin{align}\nonumber
    \bm{s} & = f_{\bm{s}}(\bm{x}), & G_{ir} & = f_{ir}(\bm{x}), & G_r & = f_r(\bm{x}), \\ 
    A & = g_A(G_{ir}, G_r), & y & = g_y(G_r, \bm{s}), & k & = g_k(\bm{s}).
\end{align}
Based on the aforementioned causal generation process, our goal is to learn a robust molecular property prediction model with the help of the training dataset. In other words, we aim to estimate the conditional distribution $P(y|\bm{x}, A, k)$. For a better understanding of this paper, we provide the notation description as shown in Table 1.

\newcommand{\tabincell}[2]{\begin{tabular}{@{}#1@{}}#2\end{tabular}}  
\begin{table}
	\centering
 \caption{Notation and the descriptions.}
	\label{tab:notation}
	\begin{tabular}{c|c}
		\toprule
		\small{Symbols}  & Definitions and Descriptions \\
		\midrule
		$G_{ir}, G_r$            & \tabincell{c}{Semantic-irrelevant substructure and \\ semantic-relevant substructure.} \\
		\midrule
		$x,k$            & \tabincell{c}{Observed atom attributes, \\ and atom number.} \\
		\midrule
		$\bm{s}, \hat{\bm{s}}$  & \tabincell{c}{Ground truth and estimated atom latent variables.} \\
		\midrule
		$f_{\bm{s}}, f_r, f_{ir}$  & \tabincell{c}{The functions that are used to generate atom \\ latent variables, 
   semantic-relevant and \\ semantic-irrelevant substructures, respectively.} \\
		\midrule
		$g_{A}, g_y, g_{k}$  & \tabincell{c}{The functions that are used \\ to reconstruct molecular data, 
   molecular property and \\ atom numbers, respectively.} \\
		\midrule
		$\phi$  & Invertible function.\\
		\midrule
		$\textsc{\textbf{W}}$  & \tabincell{c}{Vectors that are composed of the 1-order \\ and 2-order derivative between the \\ conditional distribution of $\mathbf{s}$ and $\mathcal{s}$.} \\ 
		\midrule
		$\bm{B}_r, \bm{B}_n$   & Parameters of multivariate Bernoulli distribution. \\
		\midrule
		$A, A'$   & Structures of observed and augmented molecule data. \\
		\midrule
		$v$   & Node (Atom) number of a graph (Molecule).\\
		\midrule
		$\psi$ & A smooth mapping, i.e.,  $\hat{\bm{B}}_h=\psi(\bm{B}_h, \bm{B}_n)$ \\
		\midrule
		$d$              & The dimension number of representation. \\
		\midrule
		$\theta_*, \omega_*, W_*$       & The parameters of our SCI model. \\
		\midrule
		$P_{B}(*;\mathcal{B})$       & \tabincell{c}{Multivariate Bernoulli distribution with \\ the parameters $\bm{B}$.} \\
		\midrule
		$E(\cdot)$         & Encoders implemented by neural architectures. \\
		\midrule
		$D(\cdot)$         & Decoders implemented by neural architectures.\\
		\midrule
		$C_h$         & The context representation of $j$ node.\\
		\midrule
		$N^i(j)$    & The $i$-order neighbourhood of the $j$-th node.\\
		\bottomrule
	\end{tabular}
\end{table}

\section{Identifiability of the Semantic Components }

\subsection{OOD Generalization V.S. Identification}

As shown in Figure 2, the spurious correlation between the semantic-irrelevant subgraph and the label occurs since the unobserved $G_{ir}$ are dependent on $y$ when $A$ is given. This is why existing GNN-based methods can hardly address the OOD challenge. 
Based on the aforementioned causal generation process, we find that the OOD generalization problem can be addressed by identifying $G_r$ and $\bm{s}$. This is because $y$ is independent of $G_{ir}$, when $G_r$ is given, and the spurious correlations are removed. Therefore, we can formulate the OOD generalization problem as an identification problem of latent semantic components.

According to the above discussion, to obtain the robust molecular property prediction model, we first identify the distributions of the latent semantic-relevant substructure $G_r$ and the atom latent variables $\bm{s}$ and then estimate $P(y|G_r, \bm{s})$. 
\textcolor{black}{To achieve this, we use causal identification theories to guarantee that the distributions of these latent variables can be estimated from the observed data. }

\subsection{Identifying the atom latent variables $\mathbf{s}$.}

In this subsection, we show that the distribution of atom latent variables $\bm{s}$ can be reconstructed theoretically with the help of a component-wise identification theory.
In detail, given any true atom latent variable $s_{i}$, there exists a corresponding estimated variable $\hat{s}_j$ and an invertible function $\phi:\mathbb{R}\rightarrow\mathbb{R}$, such that $\hat{s}_j=\phi(s_i)$. Based on the proposed data generation process, we can prove that $\hat{s}_i$ is component-wise identifiable, which is shown as follows \textcolor{black}{(proof is provided in Appendix A)}.

\begin{theorem}
\label{the:component}
    \textbf{\textit{(Atom Latent Variables ($\bm{s}$) Identification) }}
    We follow the causal mechanism shown in Figure 2 and make the following assumptions:
    \begin{itemize}
        \item A1 (\underline{Smooth and Positive Density}): The probability density function of atom latent variables is smooth and positive, i.e. $P(\bm{s}|\bm{x}) > 0$.
        \item A2 (\underline{Conditional independence}): Conditioned on $\bm{x}$, each $s_i$ is independent of any other $s_j$ for $i,j\in [n],i\neq j$, i.e, $log P(\bm{s}|\bm{x})=\sum_i^n P(s_i|\bm{x})$.
        \item A3 (\underline{Linear independence}): For any $\bm{s} \in \mathcal{S} \subseteq \mathbb{R}^n$, where $\mathcal{S}$ is the range of $\bm{s}$ and $n$ is the dimension of $\bm{s}$, there exist $2n+1$ values of $\bm{x}$, i.e., $\bm{x}_j$ with $j=0,1,...,2n$, such that the $2n$ vectors $\textbf{\textsc{w}}(\bm{s},x_j)-\textbf{\textsc{w}}(\bm{s},x_0)$ with $j=1,...2n$, are linearly independent, where vector $\textbf{\textsc{w}}(\bm{s},\bm{x})$ is formalized as follows:
        \begin{equation}
        \label{equ:the_ind}
        \begin{split}
             \textbf{\textsc{w}}(\bm{s},x_j)=(&\frac{\partial\log P(s_0|\bm{x})}{\partial s_0},... \frac{\partial\log P(s_n|\bm{x})}{\partial s_n},...\\&\frac{\partial^2\log P(s_0|\bm{x})}{\partial^2 s_0},...\frac{\partial^2\log P(s_n|\bm{x})}{\partial^2 s_n}).
        \end{split}
        \end{equation}
    \end{itemize}
    If a learned generative model $(\hat{f}_{ir},\hat{f}_r,\hat{f}_{ir},\hat{g}_A,\hat{g}_y,\hat{g}_k)$ assumes the same generation process shown in Figure 2 and matches the truth-conditional distribution, i.e. $P(\hat{k}|\bm{x})=P(k|\bm{x})$ and $\hat{k}$ denote the estimated variables, then the identifiability of the atom latent variables $\bm{s}$ is ensured, i.e., the ground-truth atom latent variables can be learned.
\end{theorem}

\textcolor{black}{\textbf{Proof Sketch: }The proof the Theorem \ref{the:component} can be separated into three steps. First, we construct an invertible transformation $h$ between the estimated latent variables $\hat{s}$ and the ground truth latent variables $s$. Second, we use the variance of different types of atoms to build a full-rank linear system. Based on A3, there is only one solution in the aforementioned full-rank linear system, i.e., $\frac{\partial s_i}{\partial \hat{s}_j}=0, i,j \in \{1,2,\cdots, n\}$. Since the Jacobian of $h$ is invertible, for each $s_i, i, j\in \{1,2,\cdots, n\}$, there exist a $h_i$ such that $s_i=h_i(\hat{s}_j)$ and $s_i$ is component-wise identifiable.}

\textcolor{black}{\textbf{Discussion}: The identification results shown in Theorem \ref{the:component} use the same three assumptions in the nonlinear ICA literature \cite{khemakhem2020variational, hyvarinen2017nonlinear}, which are also applicable to molecular property prediction scenarios. A1 holds since each atom has unique properties and if each atom exists in the dataset, then $p(s|x)>0$. A2 holds as properties of atoms i.e., radioactivity and stability are independent. A3 means that $f_{\bm{s}}(\bm{x})$ varies sufficiently over $\bm{x}$, which can be satisfied when we use enough data to model the data generation process. Therefore, based on the data generation process in Figure 2, the molecular property prediction task meets the aforementioned assumptions, and hence the atom latent variables $\bm{s}$ are component-wise identifiable.}

Based on Theorem 1, we can make sure that $\bm{s}$ can be reconstructed theoretically by modeling the causal mechanism shown in Equation (1). We will provide the implementation details in the next section by assuming that $\bm{s}$ follows the factorized Gaussian distributions to meet the assumption.

\subsection{Identifying the distribution of the latent semantic-relevant substructure $\mathbf{G}_{r}$.}

In this subsection, we show that the distribution of $G_r$ can be reconstructed theoretically. 
\textcolor{black}{We first assume that the truth semantic-relevant graphs follow the multivariate Bernoulli distribution with the parameters of $\bm{B}_r \in (0,1)^{v\times v}$, i.e. $G_r \sim P_B(G_r;\bm{B}_r)$, where $v$ is the node number of $G_r$. Since the ground-truth $G_r$ is sampled from $P_B(G_r;\bm{B}_r)$, the identification of the distribution of $G_r$ equals to identify the parameters $\bm{B}_r$. In the following, we will show the identification results of $\bm{B}_r$.
}

Motivated by the block-identification theory \cite{von2021self}, the latent variables can be theoretically identified given the joint distribution of the pairwise samples. Therefore, we can estimate $P(A, A')$ by leveraging $A'=g_A(G_{ir}', G_r)$ to generate the augmented molecular structures $A'$ with different $G_{ir}'$ as input. After estimating $P(A, A')$, we prove that $G_r$ is identifiable with a semantic-relevant regularization term, which is shown as follows (Detailed proof can be found in Appendix B.).

\begin{theorem}
\label{the:block}
\textbf{\textit{(Semantic-relevant Substructure distribution ($\bm{B}_r$) Identification) }}
    We follow the causal generation process shown in Figure 3 and make the following assumptions:
    \begin{itemize}
        \item  A1 (\underline{Smooth and Invertible Generation Process}): $g_A:G_{ir}, G_r \rightarrow A$ is smooth and invertible.
        \item A2 (\underline{Smooth, Continuous and Positive Density}): $P(G_{ir}, G_r)$ is a smooth, continuous density with $P(G_{ir}, G_r)>0$ almost everywhere.
        \item A3 (\underline{Smooth and Positive Conditional Probability}) The conditional probability density function $P(G_{ir}'|G_{ir})$ is smooth \textit{w.r.t} both $G_{ir}$ and $G_{ir}'$; for any $G_{ir}$, $P(\cdot|G_{ir})>0$ in some open, non-empty subset containing $G_{ir}$.
        \item A4 (\underline{Identical data generation process}) A learned generative model $(\hat{f}_{\bm{s}},\hat{f}_r,\hat{f}_{ir},\hat{g}_A,\hat{g}_y,\hat{g}_k)$ assumes the same generation process shown in Equation (1).
    \end{itemize}
    Let $v$ be the node number and let $\tau: A \rightarrow (0,1)^{v\times v}$ be any smooth function. Then $\bm{B}_{r}$ can be identified by minimizing the following restriction:
    \begin{equation}
    \label{equ:block_regu}
\mathcal{L}_r=\mathbb{E}_{(A,A')\sim \hat{P}(A,A')}\left[||\tau(A)-\tau(A')||^2_2\right]-H(\tau(A)),
    \end{equation}
    where $A'$ denotes the augmented sample from the proposed causal generation process and $H(\cdot)$ denotes the differential entropy of the random variables $\tau_x(A)$. When $\bm{B}_{r}$ is identified, we can obtain $G_{r}$ by sampling from $P(G_{r};\bm{B}_{r})$.
\end{theorem}

\textcolor{black}{\textbf{Proof Sketch: }The proof of Theorem \ref{the:block} can be separated into three steps. First, we demonstrate that the $\hat{\bm{B}}_r$ extracted by a smooth function by minimizing Equation (\ref{equ:block_regu}) is related to the ground-truth $\bm{B}_r, \bm{B}_{ir}$ through a smooth mapping $\psi$, i.e., $\hat{\bm{B}}_r=\psi(\bm{B}_r, \bm{B}_{ir})$. Second, we show that $\hat{\bm{B}}_r=\psi(\bm{B}_r, \bm{B}_{ir})$ can only depend on the true $\bm{B}_r$ and not on $\bm{B}_{ir}$, i.e., $\hat{\bm{B}}_r=\psi(\bm{B}_r)$ by contradiction. Third, we show that $\psi$ is bijection by using the results from \cite{zimmermann2021contrastive}. }

\textcolor{black}{\textbf{Discussion:} The identification results of Theorem \ref{the:block} means that $\bm{B}_{r}$ can be reconstructed so distribution $P(G_r;\bm{B}_{r})$ can be learned. Based on Theorem 2, we can make sure that $G_{ir}$ can be reconstructed theoretically with the help of the restriction term in Equation (3). Note that we also employ the same assumptions in block-identification \cite{von2021self}, which are also applicable to molecular property prediction scenarios. Since the ground truth $G_r, G_{ir}$ can be extracted by expert priors and manually synthesized new molecules with them, then A1 holds. A2 holds when each functional group exists in the molecule dataset, which can be easily satisfied when the dataset is large enough. Since $G_{ir}$ can be replaced by $G_{ir}'$ given any molecule (functional groups replacement in chemistry), so $P(G_{ir}'|G_{ir})>0$ and $P(G_{ir}', G_{ir})>0$, which make A3 satisfied. Finally, A4 holds since we can model the data generation process with enough data. 
In summary, with the restriction in Equation (3), the molecular property prediction task meets these assumptions, and hence the distribution of the latent semantic-relevant substructure is block-wise identifiable.}

By combining Theorem 1 and 2, we can prove that the semantic components in the generation process of molecule data are identifiable.

\begin{figure*}[t]
		\centering
	\label{fig:model}	\includegraphics[width=2.0\columnwidth]{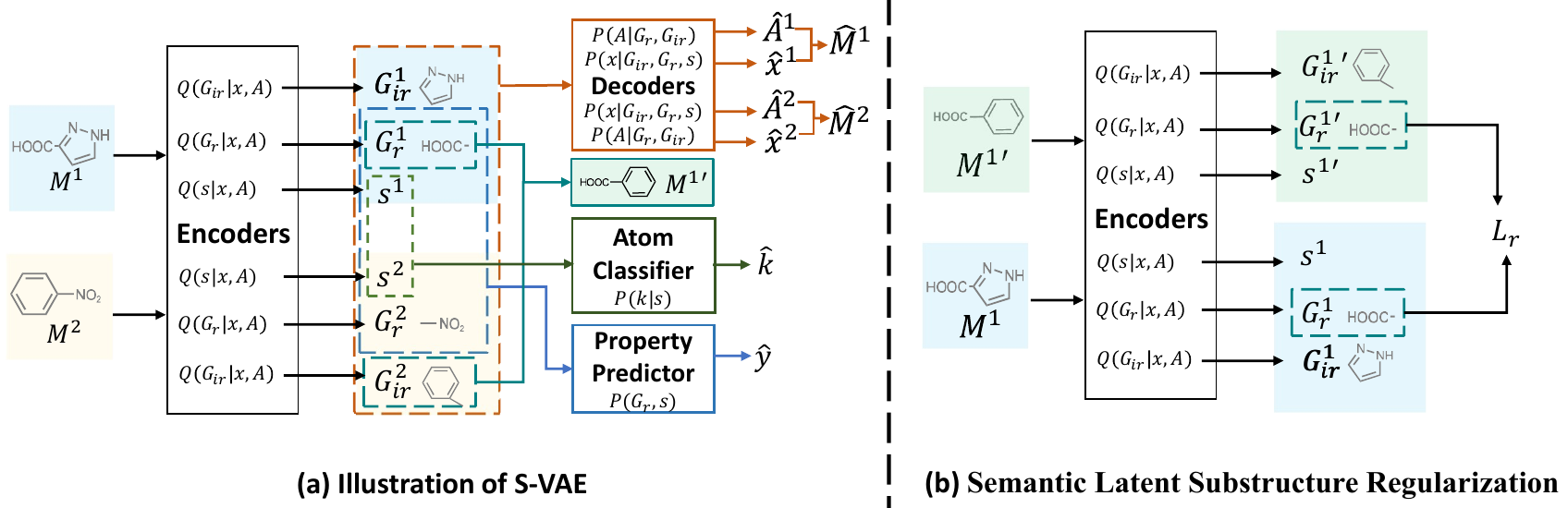}
		\caption{The illustration of the proposed semantic-component identification model. (a) The illustration of the proposed S-VAE. The ``Encoders'' block, which contains $Q(G_r|x,A),Q(G_{ir}|x,A)$ and $Q(s|A)$, is used to reconstruct the latent variables. The ``Decoders'' block, which contains $P(A|G_r, G_{ir})$ and $P(x|G_{ir},G_r,s$, is used to reconstruct the molecule data. The light green block denotes the augmented data. $\hat{M}^*$, $\hat{k}$, and $\hat{y}$ denote the reconstructed molecule data, predicted atom number, and the predicted molecular property, respectively. (b) The semantic latent substructures regularization employs the same encoders and restricts the similarity between $G_r^1$ and ${G_r^1}'$ from the original and augmented data, respectively. \textit{(Best view in color. )}}
  \label{equ:elbo}
\end{figure*}

\section{Semantic Components Identification Model}

\subsection{\textcolor{black}{Model Overview}}

Based on the theoretical results, we propose the SCI model, which contains the semantic-relevant variational auto-encoder (S-VAE) and the semantic latent substructure regularization in Figure 3. 
The S-VAE model the data generation process to reconstruct the latent variables $\bm{s}, G_r, G_{ir}$ and further disentangle the atom variables, which corresponds to Theorem 1 and is shown in Figure 3 (a). Moreover, based on the proposed S-VAE, we first generate augmented data and then apply the semantic-relevant regularization on the estimated semantic-relevant substructures for block-wise identification guarantees as shown in Figure 3 (b).



\subsection{Semantic-relevant Variational Autoencoder}
We first show the details of the semantic-relevant variational autoencoder (S-VAE). According to Theorem 1, we can identify the distribution of the atom latent variables $\bm{s}$ by modeling the proposed data generation process. 
To achieve this, we propose the semantic-relevant variational auto-encoder (S-VAE). We first model the joint distribution of $\bm{x}, y, A$ and $k$ with the help of stochastic variational inference and derive the evidence lower bound (ELBO) as shown in Equation (4) (Derivation details in Appendix C.).

  %

\begin{equation}
    \begin{split}
        \mathcal{L}_{ELBO}&=-D_{KL}(Q(G_r|\bm{x},A)||P(G_r))\\
        &-D_{KL}(Q(G_{ir}|\bm{x},A)||P(G_{ir}|G_r,\bm{s}))\\
        &-D_{KL}(Q(\bm{s}|\bm{x})||P(\bm{s}|G_r))\\
        &+E_{Q(G_r|\bm{x},A)}E_{Q(G_{ir}|\bm{x},A)}\ln P(A|G_r,G_{ir})\\
        &+E_{Q(G_r|\bm{x},A)}E_{Q(G_{ir}|\bm{x},A)}E_{Q(\bm{s}|\bm{x})}\ln P(\bm{x}|G_r,G_{ir},\bm{s})\\
        &+E_{Q(\bm{s}|\bm{x})}\ln P(k|\bm{s})\\
        &+E_{Q(G_r|\bm{x},A)}E_{Q(\bm{s}|\bm{x})}\ln P(y|G_r,\bm{s}),
    \end{split}
\end{equation}
where $D_{KL}(\cdot|\cdot)$ denotes the Kullback-Leibler divergence; $Q(G_r|\bm{x},A)$, $Q(G_{ir}|\bm{x},A)$ and $Q(\bm{s}|\bm{x})$ are used to approximate the distribution of $G_r$, $G_{ir}$ and $\bm{s}$, \textcolor{black}{which correspond to the ``Encoders'' block in the left side of Figure 3 (a)}; $P(A|G_r,G_{ir})$, $P(\bm{x}|G_r,G_{ir},\bm{s})$ denote the graph structure decoder, which correspond to the ``Decoders'' block in the right side of Figure 3 (a). And $P(k|\bm{s})$ and $P(y|G_r,\bm{s})$ the atom label decoder and the molecular property predictor, respectively. Implementation details of these components are shown as follows. 


\subsubsection{Implementation of $Q(G_r|\mathbf{x},A)$ and $Q(G_{ir}|x,A)$.} 


In this part, we aim to obtain $G_{r}$ and $G_{ir}$ by sampling from $Q(G_r|\bm{x},A)$ and $Q(G_{ir}|\bm{x},A)$, respectively. We first assume that $G_h$ and $G_n$ follow multivariate Bernoulli distributions with the parameters $\hat{\bm{B}}_r$ and $\hat{\bm{B}}_{ir}$, i.e., $Q(G_r|\bm{x},A):=P_{B}(G_r;\hat{\bm{B}}_r)$ and $Q(G_{ir}|\bm{x},A):=P_{B}(G_{ir};\hat{\bm{B}}_{ir})$.

Technologically, we estimate the parameters of the multivariate Bernoulli distribution $Q(G_r|\bm{x},A)$ and $Q(G_{ir}|\bm{x},A)$ in three steps. First, we use two layers of graph convolution networks (GCNs) to extract the node embeddings $\bm{Z}_r$ and $\bm{Z}_{ir}$, respectively. Second, we calculate the parameter matrices $\hat{\bm{B}}_r,\hat{\bm{B}}_{ir}$, which denote the probability of the existence of each edge of $G_r$ and $G_{ir}$. Third, we sample $G_r$ and $G_{ir}$ from the estimated distributions. In summary, the aforementioned three steps can be formalized as follows:
\begin{equation}
    \begin{split}
        \bm{Z}_r &= E_r(\bm{x},A,\theta_r),  \hat{\bm{B}}_r=\sigma({\bm{Z}_{r}} \bm{Z}_{r}^\mathsf{T}),  G_r \sim P(G_r;\hat{\bm{B}}_r),\\
        \bm{Z}_{ir} &= E_{ir}(\bm{x},A,\theta_{ir}),  \hat{\bm{B}}_{ir}=\sigma({\bm{Z}_{{ir}}} \bm{Z}_{{ir}}^\mathsf{T}), G_{ir}\sim P(G_{ir};\hat{\bm{B}}_{ir}),
    \end{split}
\end{equation}
where $E_r$ and $E_{ir}$ denote the feature extractors with their corresponding training parameters $\theta_r, \theta_{ir}$; $\bm{Z}_r,\bm{Z}_{ir}$ denote the feature matrices; $\sigma(\cdot)$ is the Sigmoid function. Here we employ Gumbel-Softmax \cite{jang2016categorical} to sample $G_r$ and $G_{ir}$, respectively. Hence, $D_{KL}(Q(G_{ir}|\bm{x},A)||P(G_{ir}|G_r,\bm{s}))$ and $D_{KL}(Q(G_r|\bm{x}, A)||P(G_r))$ in Equation (4) denote the Kullback-Leibler divergence of the Bernoulli distributions.

\subsubsection{Implementation of $Q(\mathbf{}{s}|\mathbf{x})$.} In this subsection, we aim to obtain $\bm{s}$ by sampling from $Q(\bm{s}|\bm{x})$. Therefore, we first assume that $\bm{s}$ follow the Gaussian distributions with parameters $\mu$ and $\sigma$, i.e., $Q(\bm{s}|\bm{x})=\mathcal{N}(\mu,\sigma)$.
Sequentially, we employ three layers of Multilayer Perceptron (MLPs) to estimate $\mu$ and $\sigma$, which are shown as follows:
\begin{equation}
\mu=E_\mu(\bm{x};\theta_\mu),\sigma=E_\sigma(\bm{x};\theta_\sigma),
\end{equation}
in which $\theta_\mu$ and $\theta_\sigma$ are the training parameters. Hence, $D_{KL}(Q(\bm{s}|\bm{x})||P(\bm{s}|G_h))$ in Equation (4) denotes the Kullback-Leibler divergence of Gaussian distributions.

\subsubsection{Implementation of $P(A|G_r,G_{ir})$.}
$P(A|G_r,G_{ir})$ is used to generate the molecule structures with the help of $G_{ir}$ and $G_r$. Formally, we have:
\begin{equation}
\hat{A} = D_A(G_r, G_{ir};\bm{\omega}_A),
\end{equation}
in which $\hat{A}$ is the reconstructed structure; $D_A$ is implemented by a MLP with Sigmoid activated function; and $\bm{\omega}_A$ denote the training parameters of $D_A$. 

\subsubsection{Implementation of $P(\mathbf{x}|G_r,G_{ir},\mathbf{s})$.}
$P(\bm{x}|G_r, G_{ir},\bm{s})$ is used to reconstruct the observed atom features, and we have:
\begin{equation}
    \hat{x} = D_x(G_r, G_{ir}, \bm{s};\bm{\omega}_{\bm{x}}),
\end{equation}
in which $\hat{x}$ denote the reconstructed atom feature; $D_x$ is also implemented by an MLP; and $\bm{\omega}_x$ denote the training parameters of $D_x$.

\subsubsection{Implementation of $P(k|\mathbf{s})$.} 
$P(k|\bm{s})$ is used to estimate the label of each atom (i.e., molecule number), we employ an MLP to predict $k$, which is shown as follows:
\begin{equation}
    \hat{k} = D_k(\bm{s};\bm{\omega}_k),
\end{equation}
where $\hat{k}$ is the predicted label; $D_K$ is implemented with an MLP layer and $\bm{\omega}_k$ are the training parameters of $D_k$.

\subsubsection{Implementation of $P(y|G_r, \mathbf{s})$.}

\textcolor{black}{$P(y|G_r, \bm{s})$ is used to predict the molecular properties with the help of $G_r$ and $\bm{s}$, we employ a GAT-based architecture to predict $\hat{y}$, which are calculated via two steps.}

\textcolor{black}{First, we employ multi-hop graph attention (GAT) layers to aggregate the features of each atom node with the help of the molecular structures. Specifically, for the $i$-th GAT layer, we consider the $i$-th neighbors for feature aggregation, which is shown as follows.}
\begin{equation}
\begin{split}
a_{jk}&=softmax(leaky\_relu(W_1\cdot[h^{i-1}_j,h^{i-1}_k]))\\
    C_{j}&=elu(\sum_{k\in N^i(j)}a_{jk}\cdot W_2 \cdot     h^{i-1}_j)\\
    h^i_j &= GRU(C_j, h^{i-1}_j;W_3),
\end{split}
\end{equation}
\textcolor{black}{where $j,k$ denote the node index, and $N^i(j)$ denote the $i$-th order neighbors of the $j$-th node; $W_1, W_2, W_3$ denotes the training parameters and $GRU$ denotes the gated recurrent units; $h_j^i$ is the node representation of the $j$-th node (atom) with $i$-th order aggregation. In our experiment, we let the maximum of $i$ be 3.}

\textcolor{black}{Second, we calculate the graph-level representation. To achieve this, we employ the idea of the hypergraph and assume that a supernode is connected with all the nodes of $G_r$, so the representation of this supernode can be considered as the graph-level representation. We also employ one layer of GAT to extract the graph-level representation and an MLP for final prediction. Formally, we have:}
\begin{equation}
    \hat{y}=D_y(GAT([\bm{s},h^i], G_r);W_4),
\end{equation}
\textcolor{black}{where $\hat{y}$ denotes the predicted label; $D_y$ denotes an MLP with training parameters $W_4$.}


\subsection{Semantic Latent Substructure Regularization}
In this subsection, we introduce the details of semantic latent substructure regularization as shown in Figure 3 (b). According to Theorem 2, we can identify the distribution of $G_r$ by employing the regularization as shown in Equation (3). 
To achieve this, we should first obtain the augmented sample $A'$, and 
then the implementation of $Q(G_r|\bm{x}, A)$ can substitute for the function $\tau$ in Equation (3).

For the augmented sample $A'$, since the proposed data generation process can be modeled by maximizing Equation (4) with mini-batch optimization, we can generate $A'$ in two steps. First, we obtain $G_{ir}'$ by shuffling $G_{ir}$ in each mini-batch. Second, using the same $G_r$, we employ the implementation of $P(A|G_r,G_{ir})$ as shown in Equation (7) to generate $A'$. These two steps can be formalized as follows:
\begin{equation}
\begin{split}
    G_{ir}'=\text{Shuffle}(G_{ir}), \quad
    A' = C_A(G_r, G_{ir}';\phi_A).
\end{split}
\end{equation}
Sequentially, we can rewrite Equation (3) by combining Equation (5) and Equation (3) as follows:
\begin{equation}
\begin{split}
 \mathcal{L}_r=&\mathbb{E}_{(A,A')\sim \hat{P}(A,A')}\left[||\hat{\bm{B}_r}-\hat{\bm{B}}'_r||^2_2\right]-H(\hat{\bm{B}}_r),\\
 &\bm{Z}_r = E_r(\bm{x},A,\theta_h),  \hat{\bm{B}_r}=\sigma({\bm{Z}_{r}} {\bm{Z}_{r}}^\mathrm{T}), \\
 &\bm{Z'}_r = E_h(\bm{x},A',\theta_h).  \hat{\bm{B'}_r}=\sigma({\bm{Z'}_{r}} {\bm{Z'}_r}^\mathrm{T}). 
\end{split}
\end{equation}



\begin{table*}[]
\centering
\caption{\textcolor{black}{The statistics of the datasets. \#Graph denotes the number of molecules (graph) in each dataset. Average \#Nodes and average \#Edges denote the average number of atoms and chemical bonds in each dataset.}}
\label{tab:dataset_statictic}
\resizebox{\textwidth}{!}{
\begin{tabular}{@{}c|ccccccc@{}}
\toprule
Category                     & Name                 & \#Graphs & Average \#Nodes & Average \#Edges & Task Type             & Split Method & Metric  \\ \midrule
\multirow{10}{*}{OGBG}       & HIV                  & 41127  & 25.5          & 54.9          & Binary Classification & scaffold     & ROC-AUC \\
 & BACE                  & 1513   & 34.1 & 73.7 & Binary Classification & scaffold & ROC-AUC \\
 & BBBP                  & 2039   & 24.1 & 51.9 & Binary Classification & scaffold & ROC-AUC \\
 & ClinTox               & 1477   & 26.2 & 55.8 & Binary Classification & scaffold & ROC-AUC \\
 & Tox21                 & 7831   & 18.6 & 38.6 & Binary Classification & scaffold & ROC-AUC \\
 & SIDER                 & 1427   & 33.6 & 70.7 & Binary Classification & scaffold & ROC-AUC \\
 & toxcast               & 8576   & 18.8 & 38.5 & Binary Classification & scaffold & ROC-AUC \\
 & esol                  & 1128   & 13.3 & 27.4 & Regression            & scaffold & RMSE    \\
 & lipo                  & 4200   & 27.0   & 59.0   & Regression            & scaffold & RMSE    \\
 & freesolv              & 642    & 8.7  & 16.8 & Regression            & scaffold & RMSE    \\ \midrule
\multirow{4}{*}{GOOD}        & \multirow{2}{*}{hiv} & 32903  & 25.3          & 54.4          & Binary Classification & scaffold     & ROC-AUC \\
 &                       & 32903  & 24.9 & 53.6 & Binary Classification & size     & ROC-AUC \\
 & \multirow{2}{*}{zinc} & 199565 & 23.1 & 49.8 & Regression            & scaffold & MAE     \\
 &                       & 199565 & 22.8 & 49.1 & Regression            & size     & MAE     \\ \midrule
\multirow{7}{*}{MoleculeNet} & BBBP                 & 2039   & 24.1          & 26.0            & Binary Classification          & Scaffold     & ROC-AUC \\
 & ClinTox               & 1478   & 26.2 & 27.9 & Binary Classification          & Scaffold & ROC-AUC \\
 & Tox21                 & 7831   & 18.6 & 19.3 & Binary Classification          & Scaffold & ROC-AUC \\
 & SIDER                 & 1427   & 33.6 & 35.4 & Binary Classification          & Scaffold & ROC-AUC \\
 & BACE                  & 1513   & 34.1 & 36.9 & Binary Classification          & Scaffold & ROC-AUC \\
 & MUV                   & 93127  & 24.2 & 26.3 & Binary Classification          & Scaffold & ROC-AUC \\
 & HIV                   & 41127  & 25.5 & 27.5 & Binary Classification          & Scaffold & ROC-AUC \\ \bottomrule
\end{tabular}}
\end{table*}

By combining the S-VAE and the semantic latent substructure regularization, we formalize the total loss of the proposed SCI method as follows:
\begin{equation}
\label{equ:total_loss}
    \mathcal{L}_{total}=-\mathcal{L}_{ELBO} + \alpha \mathcal{L}_r + \beta \mathcal{L}_h+\gamma\mathcal{L}_n,
\end{equation}
where $\mathcal{L}_h$ and $\mathcal{L}_n$ denote the L1-Norm sparsity regularization of $G_h$ and $G_n$, respectively; $\alpha$, $\beta$ and $\gamma$ denote the hyper-parameters. 

\section{Experiments}


\subsection{Setup}
\subsubsection{Realworld Datasets}
To evaluate the OOD performance, we conduct experiments on several molecular property prediction tasks from three different benchmarks as follows.

\begin{itemize}
    \item Open Graph Benchmark (OGB)\cite{hu2020open} \footnote{https://ogb.stanford.edu/} is a collection of realistic, large-scale, and diverse benchmark datasets for machine learning on graphs. It contains different types of datasets for graph data. We select the graph property prediction task, which includes 8 classification tasks and 3 regression tasks.
    \item Graph Out-of-Distribution (GOOD) benchmark dataset \cite{gui2022good} \footnote{https://github.com/divelab/GOOD} is a systematic benchmark for graph out-of-distribution problem, which contains 11 datasets with 17 domain selections, in which we choose two molecular property prediction tasks with two domain selection methods for the out-of-distribution setting.
    \item MoleculeNet \cite{wu2018moleculenet} \footnote{https://moleculenet.org/datasets-1} is a large scale benchmark for molecular machine learning. For each task, we split it with scaffold split, which splits molecules according to 
their scaffold (molecular substructure). Due to the scaffold split method, the distribution of train and test datasets is different, so it is an out-of-distribution problem of molecular property prediction.
\end{itemize}


All molecules in these datasets are pre-processed using RDKit \cite{landrum2006rdkit}. 
We consider molecules as graphs, where the nodes are atoms, and the edges are chemical bonds. Each benchmark contains the atom attributes information. For example, in the OGBG dataset, the observed atom attributes are 9-dimensional vectors, which contain atomic, chirality, and other additional atomic features such as formal charge and whether the atom is in a ring.
Based on the properties of the molecules, these datasets can be divided into three subtasks: binary classification, multi-label classification, and regression. During preprocessing, these datasets employ a scaffold or size-splitting procedure to split molecules based on their two-dimensional structural framework. Scaffold splitting is a common split method and it attempts to separate structurally different molecules into different subsets, which provides a more realistic estimate of the model performance in perspective experimental settings. This pre-processing procedure will inevitably introduce spurious correlations between functional groups due to the selection bias of the training set. The details about the statistics of each dataset are shown in Table 2.

\subsubsection{Evaluation Setting and Metrics}

We use ROC-AUC as the metric for binary classification and multi-label classification tasks and use Root Mean Square Error (RMSE) or Mean Average Error (MAE) for regression tasks. We choose the model with the best validation and evaluate the chosen model on the test set. 
For each method, we use several different random seeds and report the mean and standard error.

\subsubsection{Baselines and Model Variants}

\textcolor{black}{We compare the proposed SCI method with three kinds of baselines. Besides the conventional methods based on graph neural networks (GNN), we also take the methods that are devised for molecular property prediction tasks. Since our method uses the technique of causal inference, we also consider the causality-based methods and other methods devised for graph OOD generalization.}

\begin{table*}[ht]
\centering
\caption{The ROC-AUC of the compared methods on seven molecular property classification tasks of the OGB dataset. The values presented are averaged over four replicates with different random seeds. Values in the parenthesis denote the standard errors.}
\label{tab:ogb_cls}
\begin{tabular}{@{}c|ccccccc@{}}
\toprule
\textbf{Model}           & Molhiv         & Molbace        & Molbbbp        & Molclintox    & Moltox21       & Molsider       & Moltoxcast     \\ \midrule
\textbf{GCN}             & 0.7580(0.0197)  & 0.7689(0.0323) & 0.6974(0.0153) & 0.9027(0.0134) & 0.7456(0.0035) & 0.5843(0.0034) & 0.6421(0.0069) \\
\textbf{GAT}             & 0.7652(0.0069) & 0.8124(0.0140) & 0.6864(0.0298) & 0.8798(0.0011) & 0.7492(0.0066) & 0.5956(0.0102) & 0.6466(0.0028) \\
\textbf{GraphSAGE}       & 0.7747(0.0115) & 0.7425(0.0248) & 0.6805(0.0126) & 0.8877(0.0066) & 0.7410(0.0035) & 0.6059(0.0016) & 0.6282(0.0067) \\
\textbf{GIN}             & 0.7852(0.0158) & 0.7638(0.0387) & 0.6748(0.0063) & 0.9155(0.0212) & 0.7440(0.0040) & 0.5817(0.0124) & 0.6342(0.0102) \\
\textbf{GIN0}            & 0.7814(0.0121) & 0.7584(0.0239) & 0.6611(0.0094) & 0.9212(0.0255) & 0.7490(0.0015) & 0.5968(0.0148) & 0.6289(0.0019) \\
\textbf{SGC}             & 0.6342(0.0016) & 0.6875(0.0021) & 0.6613(0.0039) & 0.8536(0.0028) & 0.7222(0.0005) & 0.5906(0.0032) & 0.6283(0.0010)  \\
\textbf{JKNet}           & 0.7534(0.0123) & 0.7425(0.0291) & 0.6930(0.0075)  & 0.8558(0.0217) & 0.7418(0.0029) & 0.5818(0.0159) & 0.6357(0.0055) \\
\textbf{DIFFPOOL}        & 0.6408(0.0497) & 0.7525(0.0116) & 0.6935(0.0189) & 0.8241(0.0167) & 0.7325(0.0084) & 0.5758(0.0151) & 0.6217(0.0054) \\
\textbf{CMPNN}           & 0.7711(0.0071) & 0.7215(0.0490)  & 0.6403(0.0172) & 0.7947(0.0461) & 0.7048(0.0107) & 0.5799(0.0080)  & 0.6394(0.0105) \\
\textbf{DIR}             & 0.7672(0.0084) & 0.7834(0.0145) & 0.6467(0.0174) & 0.8129(0.0307) & 0.6966(0.0286) & 0.5794(0.0111) & 0.6196(0.0135) \\
\textbf{StableGNN-GCN}   & 0.7779(0.0119) & 07695(0.0327)  & 0.6882(0.0387) & 0.8798(0.0237) & 0.7312(0.0034) & 0.5915(0.0117) & 0.6329(0.0069) \\
\textbf{StableGNN-Graph} & 0.7763(0.0079) & 0.8073(0.0398) & 0.6847(0.0247) & 0.9096(0.0193) & 0.6914(0.0024) & 0.5589(0.0056) & 0.7986(0.0010) \\
\textbf{AttentiveFP}     & 0.7780(0.0195)  & 0.7767(0.0026)  & 0.6555(0.0128) & 0.8335(0.0216) & 0.7934(0.0028) & 
0.6919(0.0148) & 0.7678(0.0037) \\
\textbf{OOD-GNN}     & 0.7950(0.0080)  & 0.8130(0.0120)  & 0.7010(0.0100) & 0.9140(0.0130) & 0.7840(0.0800) & 0.6400(0.0130) & 0.7870(0.0030) \\
\textbf{GIL}             & 0.7908(0.0054) & / & / & / & / & 0.6350(0.0057) & / \\
\textbf{GREA}             & 0.7932(0.0092) & 0.8237(0.0237) & 0.6970(0.0128) & 0.8789(0.0368) & 0.7723(0.0119) & 0.6014(0.0204) & 0.6732(0.0092) \\
\textbf{FFiNet}          & 0.7722(0.0157) & 0.8147(0.0325) & 0.6962(0.0060)  & 0.8151(0.0454) & 0.8009(0.0008) & / & 0.7712(0.0048) \\
\textbf{PharmHGT}        & 0.7451(0.0080)  & 0.8204(0.0121) & 0.7275(0.0170)  & 0.9098(0.0084) & 0.7443(0.0113) & 0.6143(0.0201) & 0.8023(0.0009) \\ \midrule
\textbf{SCI} &
  \textbf{0.7959(0.7976)} &
  \textbf{0.8309(0.1168)} &
  \textbf{0.7300(0.0293)} &
  \textbf{0.9220(0.0258)} &
  \textbf{0.8027(0.0071)} &
  \textbf{0.7152(0.0055)} &
  \textbf{0.8073(0.0006)} \\ \bottomrule
\end{tabular}
\end{table*}

\textcolor{black}{For the conventional GNN-based methods, we first consider the popular graph neural architectures like GCN \cite{kipf2016semi}, GAT \cite{velickovic2017graph}, GraphSAGE \cite{hamilton2017inductive}, and GIN  \cite{xu2018powerful}, which considers the molecules as graph data and extract the graph representation by averaging the node representations. Besides, we further consider the graph neural networks for graph representations like SGC \cite{wu2019simplifying}, JKNet \cite{xu2018representation}, DIFFPOOL \cite{ying2018hierarchical}, and CMPNN \cite{song2020communicative}, which leverage different neighborhood ranges or a differentiable graph pooling module to extract the node-level or edge-level representations and further the graph-level representations.}

\textcolor{black}{Moreover, we consider the methods that are devised for molecular property prediction. For example, AttentiveFP \cite{xiong2019pushing} is a variant of graph attention networks, which learns molecular representation from atom and molecular levels. FFiNet \cite{ren2023force} is a
force field-inspired neural architecture that includes all the interactions by incorporating the functional
form of the potential energy of molecules. PharmHGT \cite{jiang2023pharmacophoric} leverages a pharmacophoric-constrained multiviews molecular representation graph to extract vital chemical information from functional substructures and chemical reactions.}


\textcolor{black}{We further consider baselines that are devised for graph out-of-distribution classification or molecular property predictions. First, we consider the OOD-GNN \cite{li2021ood}, which eliminates the statistical dependence between relevant and
irrelevant graph representation by using random Fourier features. Besides, we also consider the augmentations-based method GREA \cite{liu2022graph} to show the effectiveness of the block-identification guarantee. Moreover, we consider GIL \cite{li2022learning} and LECI \cite{gui2023joint}, which capture the invariant relationships with the help of environment inference and use the environment information to learn the causally invariant substructures, respectively. We also consider GSAT \cite{miao2022interpretable}, which employs a stochastic attention mechanism for interpretable and generalization graph learning. We further consider the causality-based methods to show the advantages of our identification results. For example, we consider DIR \cite{wu2022discovering} and StableGNN \cite{fan2023generalizing}, which leverage the structural causal model to learn the domain-invariant and domain-specific rationales of graph data and apply the techniques of stable learning on molecular to remove the influence of substructure-level spurious correlation, respectively.}

\textcolor{black}{To evaluate the effectiveness of each component of our model, we further devise the following model invariant.}
\begin{itemize}
    \item \textcolor{black}{\textbf{SCI-r}: we remove the $\mathcal{L}_r$ to evaluate the effectiveness of sub-structure regularization.}
    \item \textcolor{black}{\textbf{SCI-h}: we remove the $\mathcal{L}_h$ to evaluate the effectiveness of sparsity regularization of $G_h$.}
    \item \textcolor{black}{\textbf{SCI-n}: we remove the $\mathcal{L}_n$ to evaluate the effectiveness of sparsity regularization of $G_n$. }
    \item \textcolor{black}{\textbf{SCI-k}: we remove the atom classifier to evaluate the effectiveness of the accurate reconstruction of the atom latent variables. }
\end{itemize}


\subsection{Experiment Results}

\begin{table}[t]
\centering
\caption{The RMSE of the compared methods on three molecular property regression tasks of the OGB dataset. The values presented are averaged over four replicates with different random seeds. Values in the parenthesis denote the standard errors.}
\label{tab:ogb-reg}
\resizebox{0.49\textwidth}{!}{%
\begin{tabular}{@{}cccc@{}}
\toprule
\textbf{Models}                             & Molesol        & Mollipo        & Molfreesolv    \\ \midrule
\multicolumn{1}{c|}{\textbf{GCN}}           & 1.1274(0.0191) & 0.7822(0.0067) & 2.1519(0.0645) \\
\multicolumn{1}{c|}{\textbf{GAT}}           & 1.0359(0.0509) & 0.8090(0.0076) & 2.0072(0.1410) \\
\multicolumn{1}{c|}{\textbf{GraphSAGE}}     & 1.1220(0.0436) & 0.7771(0.0156) & 2.1944(0.1242) \\
\multicolumn{1}{c|}{\textbf{GIN}}           & 1.1042(0.0385) & 0.7477(0.0077) & 2.4035(0.2073) \\
\multicolumn{1}{c|}{\textbf{GIN0}}          & 1.0001(0.0394) & 0.7601(0.0085) & 2.3816(0.0807) \\
\multicolumn{1}{c|}{\textbf{SGC}}           & 1.1106(0.0215) & 0.9983(0.0042) & 2.5517(0.0084) \\
\multicolumn{1}{c|}{\textbf{JKNet}}         & 0.9632(0.0354) & 0.7609(0.0035) & 2.2326(0.0963) \\
\multicolumn{1}{c|}{\textbf{DIFFPOOL}}      & 0.9786(0.0344) & 0.8215(0.0060) & 3.1285(0.4856) \\
\multicolumn{1}{c|}{\textbf{CMPNN}}         & 1.2530(0.0950)   & 0.8937(0.0655) & 3.3856(0.4777) \\
\multicolumn{1}{c|}{\textbf{DIR}}           & 2.1727(0.1705) & 2.3907(0.1176) & 2.4079(0.1009) \\
\multicolumn{1}{c|}{\textbf{StableGNN-GCN}} & 0.9638(0.0292) & 0.7839(0.0165) & 3.2160(0.0707)  \\
\multicolumn{1}{c|}{\textbf{StableGNN-Graph}} & 1.0092(0.0706)          & 0.6971(0.0297)          & 4.0577(0.5516)          \\
\multicolumn{1}{c|}{\textbf{AttentiveFP}}   & 0.7603(0.0083) & 0.5117(0.0296) & 3.7196(0.2071) \\
\multicolumn{1}{c|}
{\textbf{OOD-GNN}}   & 0.8800(0.0500) & / & \textbf{1.8100(0.1400)} \\
\multicolumn{1}{c|}{\textbf{FFiNet}}        & 0.8704(0.0089) & 0.6970(0.0158)  & 3.7392(0.1749) \\
\multicolumn{1}{c|}{\textbf{PharmHGT}}      & 1.0582(0.0547) & 0.7513(0.0203) & 4.2888(0.1448) \\ \midrule
\multicolumn{1}{c|}{\textbf{SCI}}             & \textbf{0.6399(0.0041)} & \textbf{0.5081(0.0117)} & 1.9484(2.0152) \\ \bottomrule
\end{tabular}%
}
\end{table}


\begin{table*}[]
\centering
\caption{The AUC-ROC and RMSE of the compared methods on two molecular property prediction tasks of the GOOD dataset. The values presented are averaged over four replicates with different random seeds. Values in the parenthesis denote the standard errors.}
\label{tab:good_exp}
\resizebox{0.8\textwidth}{!}{%
\begin{tabular}{cc|cc|cc}
\toprule
                      &                                & \multicolumn{2}{c|}{HIV}   & \multicolumn{2}{c}{ZINC}        \\ \midrule
\multicolumn{1}{l|}{} & Model                          & scaffold     & size        & scaffold       & size           \\ \midrule
\multicolumn{1}{c|}{} & \textbf{GCN (2016)}            & 66.00(2.57)     & 55.95(0.65) & 0.4332(0.0166) & 0.2553(0.0193) \\
\multicolumn{1}{c|}{} & \textbf{GAT (2017)}            & 67.00(1.77)     & 56.88(1.97) & 0.3774(0.0085) & 0.5126(0.0340)  \\
\multicolumn{1}{c|}{} & \textbf{GraphSAGE (2017)}      & 66.16(1.59)  & 54.25(2.89) & 0.5180(0.0197)  & 0.5656(0.0849) \\
\multicolumn{1}{c|}{} & \textbf{GIN (2018)}            & 70.61(1.08)  & 56.91(3.53) & 0.1678(0.0147) & 0.3837(0.0255) \\
\multicolumn{1}{c|}{} & \textbf{GIN0 (2018)}           & 69.49(1.96)  & 54.56(1.53) & 0.1856(0.0199) & 0.2342(0.0051) \\
\multicolumn{1}{c|}{} & \textbf{SGC (2019)}            & 60.74(2.78)  & 55.74(0.94) & 0.4855(0.0110)  & 0.2297(0.0111) \\
\multicolumn{1}{c|}{} & \textbf{JKNet (2018)}          & 63.46(1.89)  & 58.99(1.21) & 0.3702(0.0138) & 0.5531(0.0212) \\
\multicolumn{1}{c|}{} & \textbf{DIFFPOOL (2018)}       & 60.16(4.16)  & 52.58(2.70)  & 0.5273(0.0692) & 0.3631(0.0235) \\
\multicolumn{1}{c|}{\multirow{-9}{*}{\textbf{\begin{tabular}[c]{@{}c@{}}Conventional \\ GNN-based\end{tabular}}}} &
  \textbf{CMPNN (2020)} &
  66.95(0.34) &
  55.53(1.75) &
  0.2865(0.0674) &
  0.5157(0.0793) \\ \midrule
\multicolumn{1}{c|}{} & \textbf{StableGNN-GCN (2023)}  & 64.92(2.61)  & 52.69(3.75) & 0.4567(0.0105) & 0.5548(0.0533) \\
\multicolumn{1}{c|}{} & \textbf{StableGNN-SAGE (2023)} & 60.80(1.99)   & 59.38(1.26) & 0.2645(0.0108) & 0.4990(0.1019)  \\
\multicolumn{1}{c|}{} & \textbf{DIR (2022)}            & 68.44(2.51)  & 57.67(3.75) & 0.3682(0.0639) & 0.4578(0.0412) \\
\multicolumn{1}{c|}{\multirow{-5}{*}{\textbf{\begin{tabular}[c]{@{}c@{}}Graph OOD \\ Generalization\end{tabular}}}}&
  \textbf{GSAT (2022)} &
  70.07(1.76) &
  { 60.73(2.39)} &
  { 0.1418(0.0077)} &
  { 0.2101(0.0095)} \\ \midrule
\multicolumn{1}{l|}{} & \textbf{AttentiveFP (2019)}    & 68.48(0.63) & 66.97(0.58) & 0.3550(0.0062)  & 0.3726(0.0162) \\
\multicolumn{1}{l|}{} & \textbf{ParmHGT (2023)}        & 66.75(2.80)   & 58.32(4.83) & 0.1441(0.0197) & 0.2369(0.0304) \\
\multicolumn{1}{l|}{} & \textbf{FFiNet (2023)}         & 65.85(3.69)  & 61.84(2.62) & 0.1660(0.0078)  & 0.2212(0.0180)  \\
\multicolumn{1}{l|}{\multirow{-4}{*}{\begin{tabular}[c]{@{}c@{}}\textbf{Molecular} \textbf{Property}\\  \textbf{Prediction}\end{tabular}}} &
  \textbf{ToxExpert (2023)} &
  68.93(1.05) &
  58.63(2.11) &
  0.3512(0.0021) &
  0.4887(0.0036) \\ \midrule
\multicolumn{1}{l|}{} &
  \textbf{SCI} &
  \textbf{72.15(2.55)} &
  \textbf{68.21(0.71)} &
  \textbf{0.1241(0.0056)} &
  \textbf{0.1979(0.0060)} \\ \bottomrule
\end{tabular}%
}
\end{table*}

\subsubsection{Experiment Results on OGB datasets}

\begin{table*}[t]
\centering
\caption{The ROC-AUC of the compared methods on seven molecular property classification tasks of the MoleculeNet datasets. The values presented are averaged over four replicates with different random seeds. Values in the parenthesis denote the standard errors.}
\label{tab:moleculenet}
\begin{tabular}{@{}c|ccccccc@{}}
\toprule
\textbf{Models}         & BBBP           & ClinTox        & Tox21          & SIDER          & BACE           & MUV             & HIV            \\ \midrule
\textbf{GCN}            & 0.5876(0.0022) & 0.7321(0.0027) & 0.5938(0.0017) & 0.6092(0.0031) & 0.6570(0.0587) & 0.5523(0.0599)  & 0.5963(0.0122) \\
\textbf{GAT}            & 0.6197(0.0050) & 0.7261(0.0062) & 0.5824(0.0009) & 0.6085(0.0043) & 0.6371(0.0223) & 0.6106(0.0462)  & 0.5888(0.0019) \\
\textbf{GraphSage}      & 0.6237(0.0078) & 0.7354(0.0059) & 0.6032(0.0004) & 0.6087(0.0038) & 0.6441(0.0154) & 0.5981(0.0472)  & 0.5948(0.0176) \\
\textbf{GIN}            & 0.6189(0.0109) & 0.7677(0.0069) & 0.5943(0.0006) & 0.6080(0.0023) & 0.6364(0.0048) & 0.6002(0.0206)  & 0.6071(0.0008) \\
\textbf{GIN0}           & 0.6123(0.0108) & 0.7733(0.0066) & 0.5883(0.0087) & 0.6131(0.0056) & 0.6332(0.0061) & 0.6004(0.0192)  & 0.6003(0.0014) \\
\textbf{SGC}            & 0.5931(0.0046) & 0.8287(0.0244) & 0.7771(0.0033) & 0.6088(0.0031) & 0.7471(0.0058) & 0.6757(0.0034)  & 0.6346(0.0191) \\
\textbf{JKNet} &
  0.5959(0.0037) &
  0.8022(0.0070) &
  \multicolumn{1}{l}{0.7976(0.0123)} &
  0.6134(0.0028) &
  0.7489(0.0018) &  
  0.7238(0.0041) &
  0.6546(0.0067) \\
\textbf{Diffpooling}    & 0.6281(0.0196) & 0.7943(0.0471) & 0.6104(0.0036) & 0.6039(0.0155) & 0.7431(0.0454) & 0.6206(0.0341)  & 0.6934(0.0111) \\
\textbf{CMPNN}          & 0.5854(0.0010) & 0.8706(0.0148) & 0.7825(0.0035) & 0.5551(0.0060) & 0.7405(0.0347) & 0.7789(0.0071)  & 0.6158(0.0045) \\
\textbf{StableGNN-GCN}  & 0.5993(0.0235) & 0.8402(0.0475) & 0.5965(0.0027) & 0.5989(0.0115) & 0.7420(0.0148) & 0.6454(0.0023)  & 0.6958(0.0097) \\
\textbf{StableGNN-SAGE} & 0.6185(0.0141) & 0.6120(0.0113) & 0.6122(0.0115) & 0.6120(0.0113) & 0.7939(0.0012) & 0.6395(0.0051)  & 0.6866(0.0402) \\
\textbf{DIR} &
  0.6442(0.0088) &
  0.8435(0.0360) &
  \multicolumn{1}{l}{0.7723(0.0019)} &
  0.6110(0.0094) &
  0.7515(0.0593) &
  \multicolumn{1}{l}{0.7172(0.0051)} &
  0.6896(0.0127) \\
\textbf{AttentiveFP}    & 0.6451(0.0169) & 0.8803(0.0168) & 0.8448(0.0015) & 0.6321(0.0125) & 0.8186(0.0045) & 0.7897(0.0163)  & \textbf{0.7662(0.0141)} \\
\textbf{PharmHGT}       & 0.7077(0.0107) & 0.8935(0.0363) & 0.8366(0.0043) & 0.6178(0.0166) & 0.8166(0.0136) & 0.7536(0.0324) & 0.7163(0.0072) \\
\textbf{FFiNet}         & 0.6904(0.0136) & 0.8590(0.0347) & 0.8213(0.0029) & 0.6438(0.0088) & 0.7967(0.0372) & 0.7142(0.0135) & 0.7276(0.0063) \\ \midrule
\textbf{SCI} &
  \textbf{0.7120(0.0069)} &
  \textbf{0.9072(0.0327)} &
  \textbf{0.8481(0.0029)} &
  \textbf{0.6455(0.0007)} &
  \textbf{0.8237(0.0054)} &
  \multicolumn{1}{l}{\textbf{0.7956(0.0019)}} &
  {0.7546(0.0110)} \\ \bottomrule
\end{tabular}
\end{table*}

Experiment results on seven molecular property prediction classification and regression datasets are shown in Table \ref{tab:ogb_cls} and \ref{tab:ogb-reg}, respectively.  We conduct
the Wilcoxon signed-rank test on the reported accuracies,
our method significantly outperforms the baselines, with
a p-value threshold of 0.05. According to the experiment results, we can draw the following conclusions:
\begin{itemize}
    \item The proposed SCI model outperforms all other baselines on most of the datasets, which is attributed to both the proposed causal generation process and the identifiability theories of semantic-relevant components.
    \item Some GNN-based graph classification methods like SGC and GIN do not perform well on the molecular property prediction tasks, which means that these methods contain poor generalization ability, since they consider the whole molecular graph. And DIFFPOOL achieves better results, this is because it learns the hierarchical representations and might filter the property-irrelevant structural information.  
    \item The recently proposed methods for OOD graph data like StableGNN and DIR also perform well in most of the tasks, this is because these methods leverage the causal models to remove the spurious correlations.
    \item We also find that the proposed method does not achieve a distinct performance on Moltox21 and Molhiv compared with other methods. 
    These results can be attributed to two reasons. First, the moltoxcast dataset contains more than 600 types of molecular properties, it is so complex that all the methods achieve close results in this dataset. Second, the proposed method is based on the assumption of well modeling the proposed causal mechanism, but this dataset size is too small for our model to learn an accurate causal mechanism.
    \item \textcolor{black}{According to the experiment results on the regression tasks in Table \ref{tab:ogb-reg}, we can find that our method achieves the best results on most of the datasets and obtains comparable performance on the Molfreesolv dataset, reflecting that the proposed identification theories can benefit both the classification and the regression tasks.}
\end{itemize}

\begin{figure*}[htbp]
\centering
\subfigure[The AUC-ROC performance of the different SCI model invariant on the BBBP, ClinTox, and Tox21 datasets from the MoleculeNet Benchmark.]{
\begin{minipage}[t]{0.31\textwidth}
\centering
\includegraphics[width=6cm]{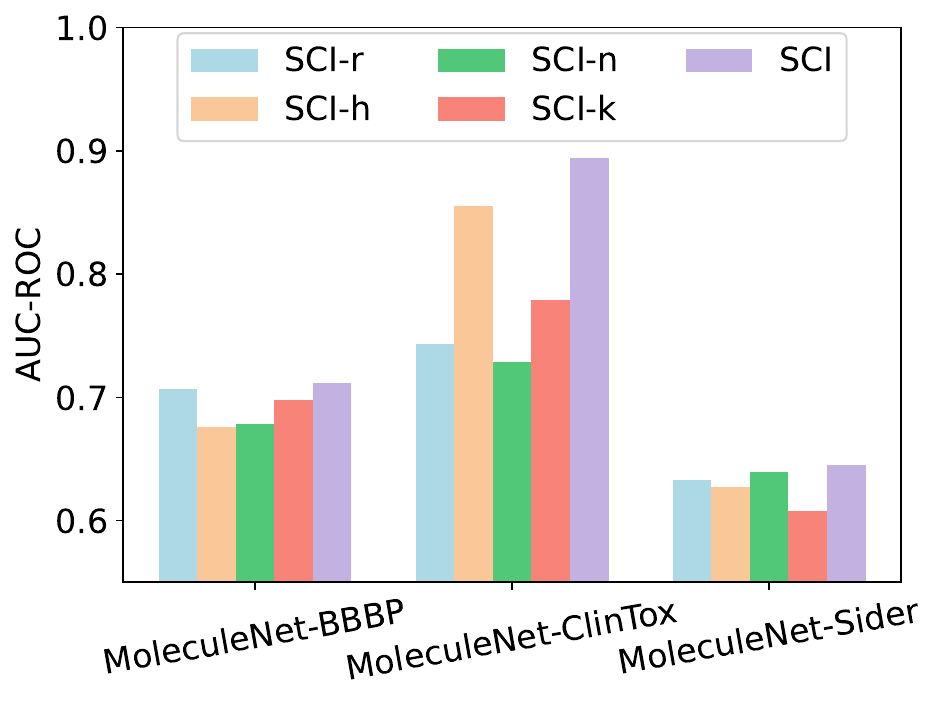}
\end{minipage}}
\hspace{1mm}
\subfigure[The AUC-ROC performance of the different SCI model invariant on the MoleculeNet-MUV, MoleculeNet-BACE, and OGBG-Tox21 datasets.]{
\begin{minipage}[t]{0.31\textwidth}
\centering
\includegraphics[width=6cm]{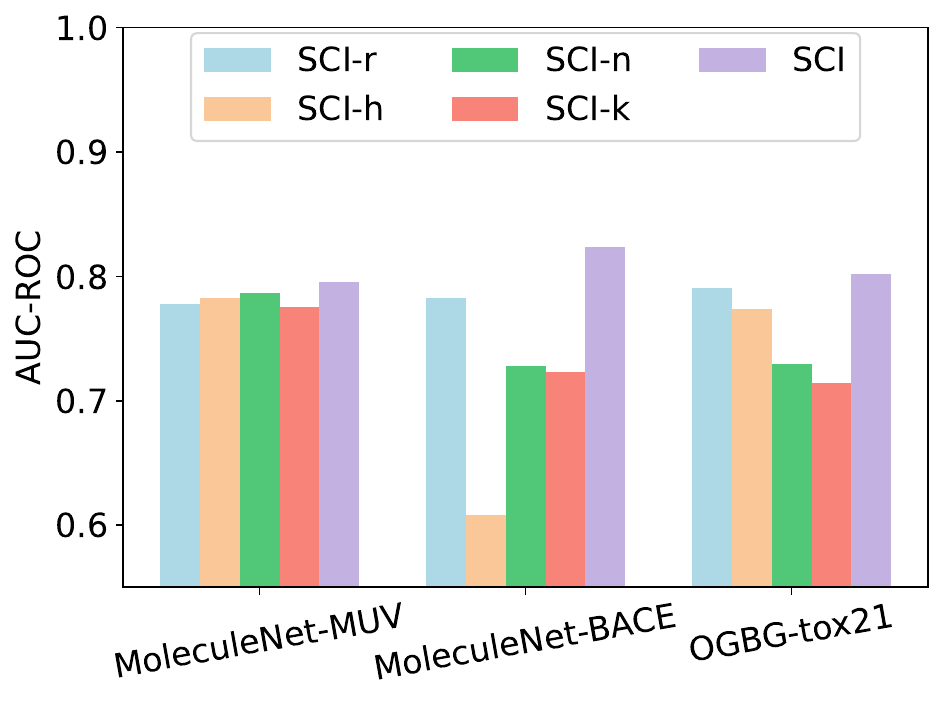}
\end{minipage}}
\hspace{1mm}
\subfigure[The AUC-ROC performance of the different SCI model invariant on the HIV, BBBP, and Toxcast datasets from the OGBG benchmark.]{
\begin{minipage}[t]{0.32\textwidth}
\centering
\includegraphics[width=6cm]{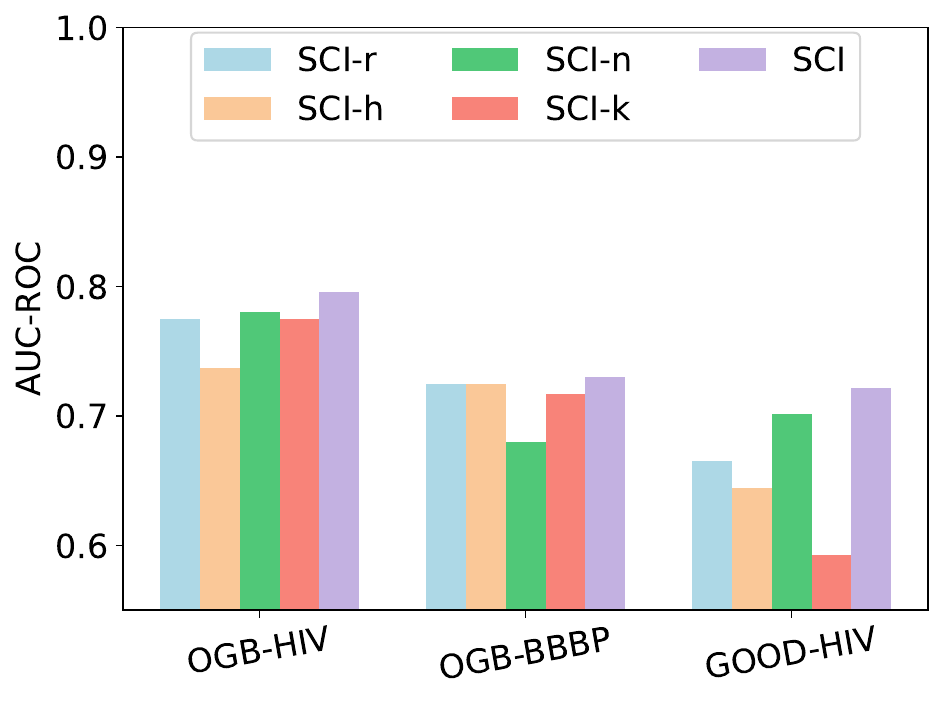}
\end{minipage}}
\caption{Ablation experiments on the different datasets from the MoleculeNet, OGBG, and GOOD benchmarks, respectively. }
\label{fig:ablation}
\end{figure*}

\subsubsection{\textcolor{black}{Experiment Results on GOOD datasets}}
\textcolor{black}{Experiment results on GOOD datasets are shown in Table \ref{tab:good_exp}. We conduct
the Wilcoxon signed-rank test on the reported accuracies,
our method significantly outperforms the baselines, with
a p-value threshold of 0.05. According to the experiment results, we can learn the following lessons.}
\begin{itemize}
    \item \textcolor{black}{Our method achieves the best results on all the split methods. According to the statistic information in Table \ref{tab:dataset_statictic}, the size of the GOOD dataset is larger, meaning that the assumptions of the proposed identification theory are easier to meet, hence the proposed SCI model can achieve better performance.}
    \item \textcolor{black}{Since the GOOD dataset is more complex, the distribution shift between the train and the test datasets might be larger, and the conventional GNN-based methods can hardly achieve ideal performance.}
    \item \textcolor{black}{In the meanwhile, compared with the conventional GCN-based methods, the methods that are devised for graph OOD problems can address the distribution shift and achieve comparable results. }
\end{itemize}

\begin{figure*}[htbp]
\centering
\subfigure[The AUC-ROC performance of the proposed SCI with different values of $\alpha$ with the range $(0.1\sim 1.5)$]{
\begin{minipage}[t]{0.31\textwidth}
\centering
\includegraphics[width=5.7cm]{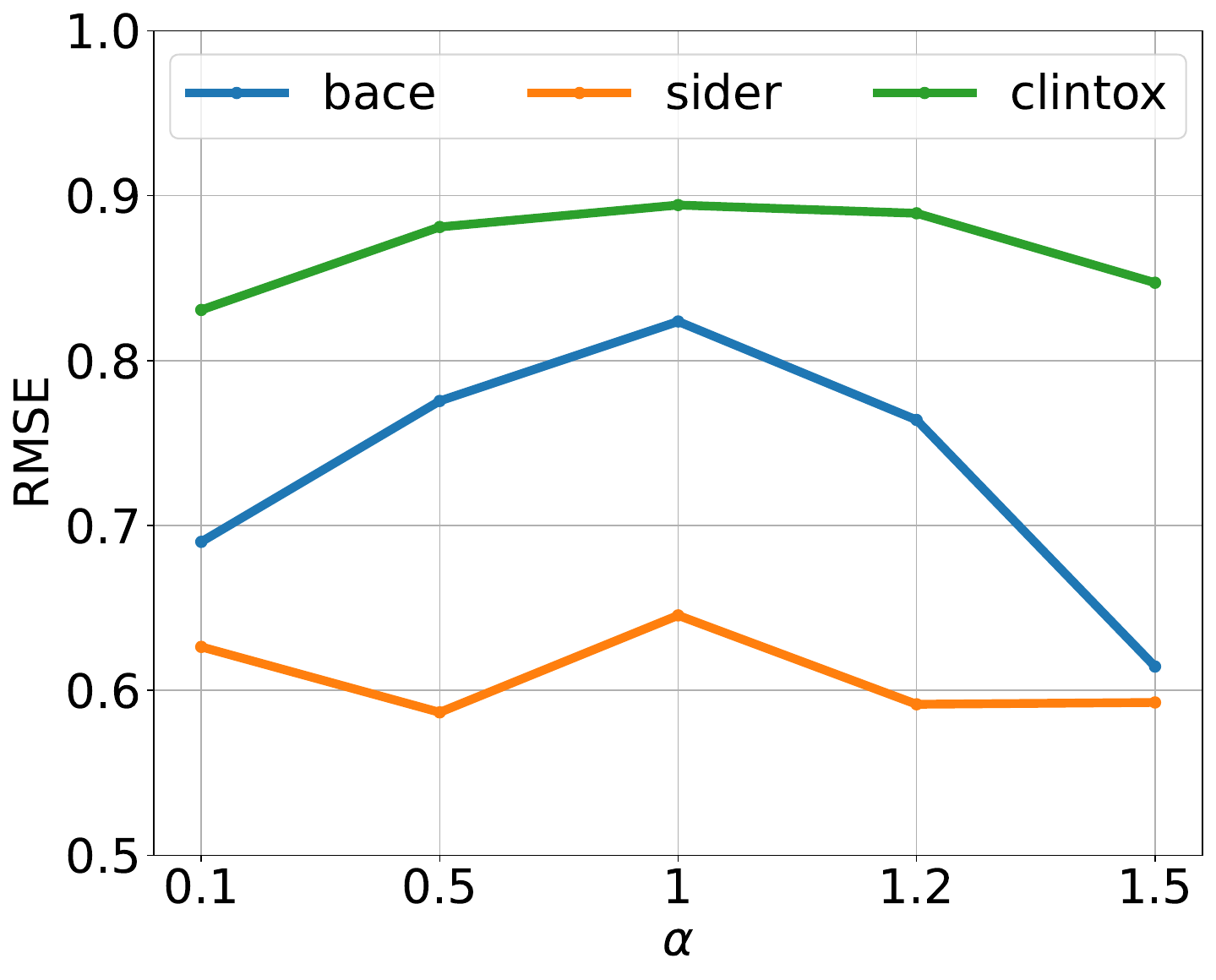}
\end{minipage}}
\hspace{1mm}
\subfigure[The AUC-ROC performance of the proposed SCI with different values of $\beta$ with the range $(0.5\sim 2.5)$]{
\begin{minipage}[t]{0.31\textwidth}
\centering
\includegraphics[width=5.7cm]{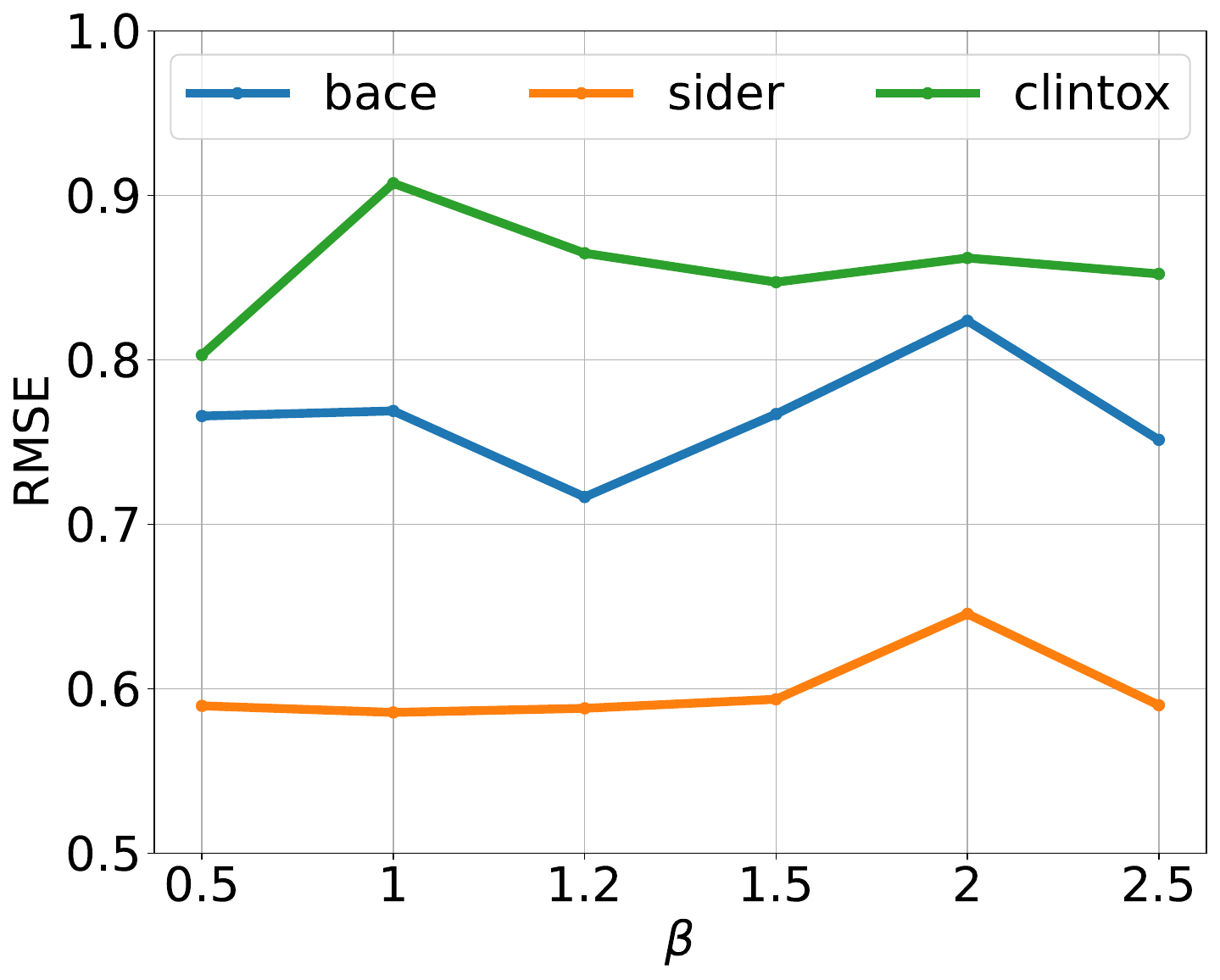}
\end{minipage}}
\hspace{1mm}
\subfigure[The AUC-ROC performance of the proposed SCI with different values of $\gamma$ with the range $(0.5\sim 2.5)$]{
\begin{minipage}[t]{0.32\textwidth}
\centering
\includegraphics[width=5.7cm]{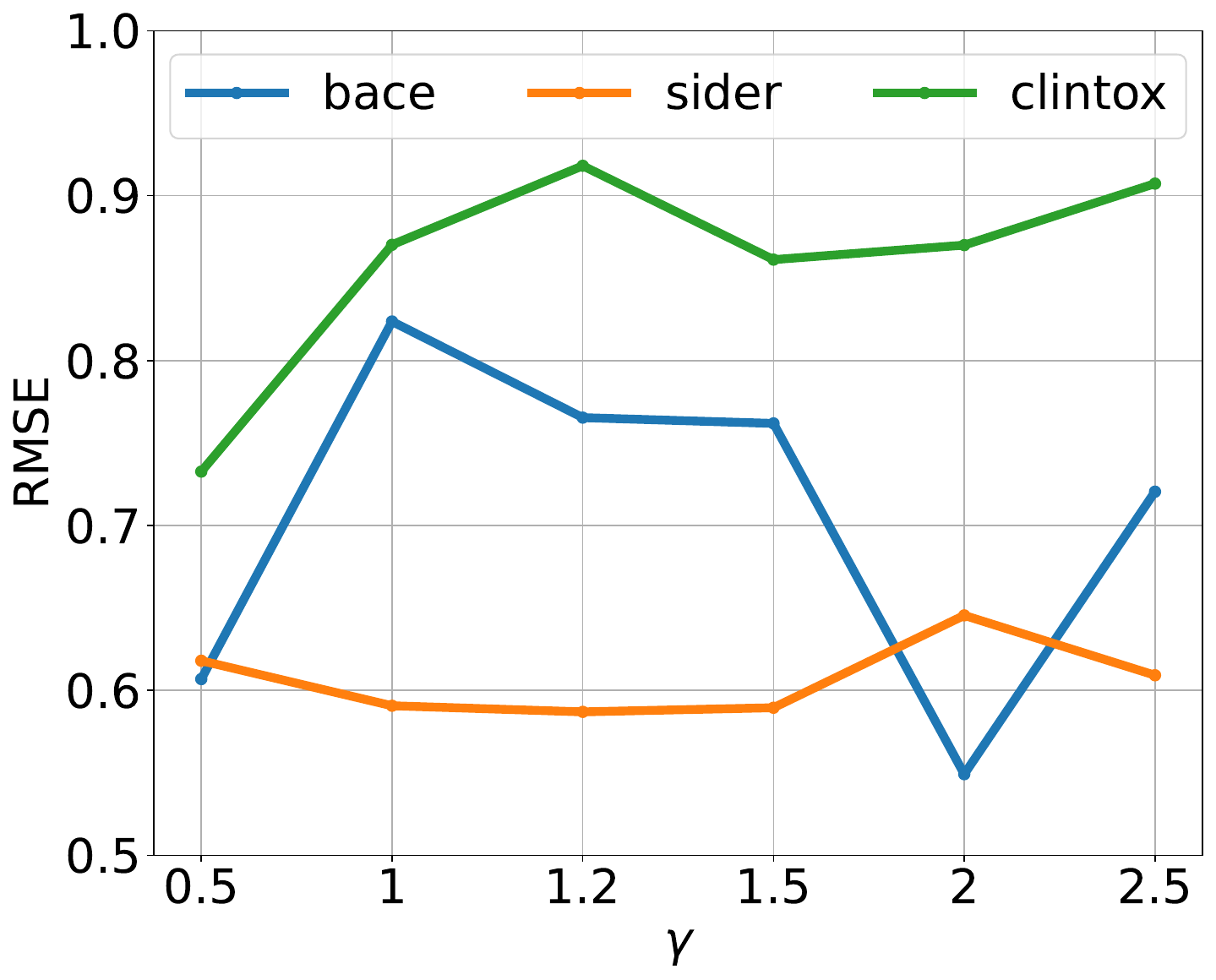}
\end{minipage}}
\caption{The experiment results of the sensitive analysis of the hyperparameters $\alpha$, $\beta$ and $\gamma$. (a) The AUC-ROC results of the SCI with different values of $\alpha$. (b) The AUC-ROC results of the SCI with different values of $\beta$. (c) The AUC-ROC results of the SCI with different values of $\gamma$.}
\label{fig:sensitity}
\end{figure*}

\begin{figure}[t]
		\centering	\includegraphics[width=0.9\columnwidth, height=7cm]{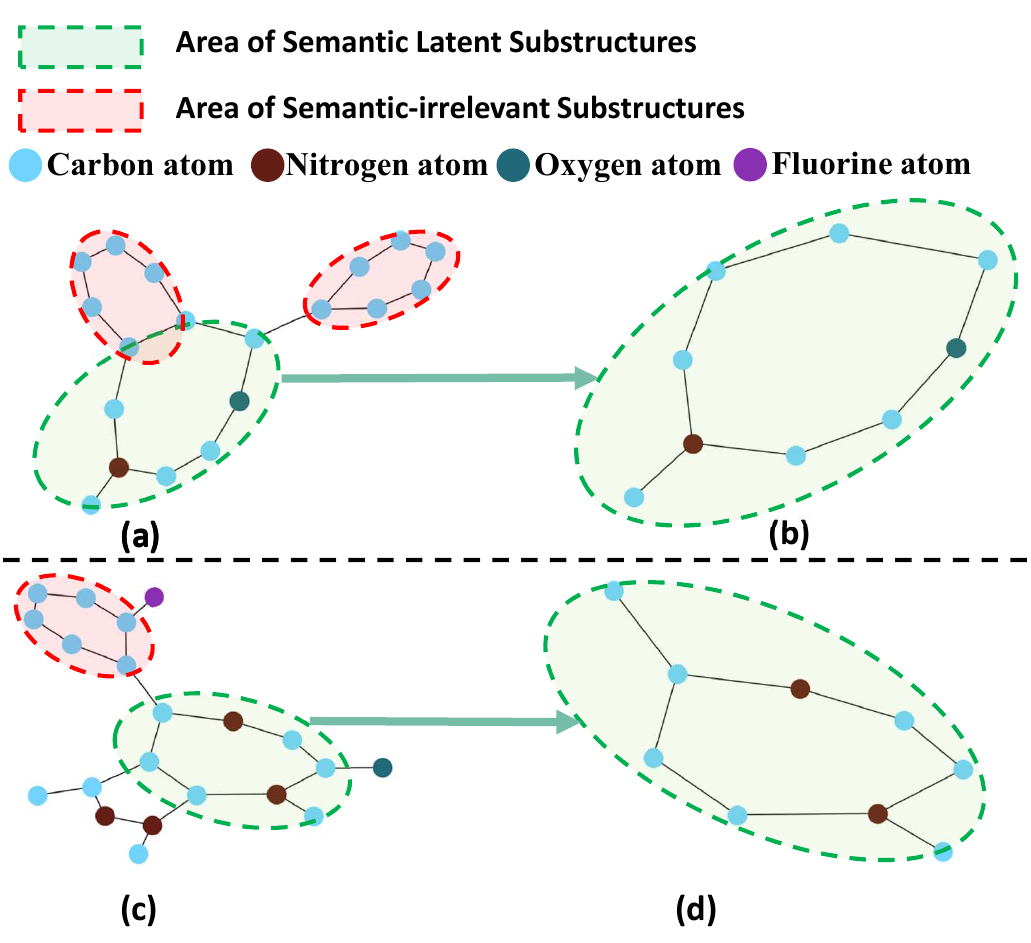}
 \caption{Case studies of the proposed MCI model (The hydrogen atoms and the details of bonds are ignored for convenience.). (a)(c) denote the original molecules and (b)(d) denote the corresponding $G_h$. The nodes with different colors denote different types of atoms.\textit{(Best view in color.)}}
 \label{fig:case}
\end{figure}

\begin{figure*}[t]
		\centering	\includegraphics[width=1.9\columnwidth]{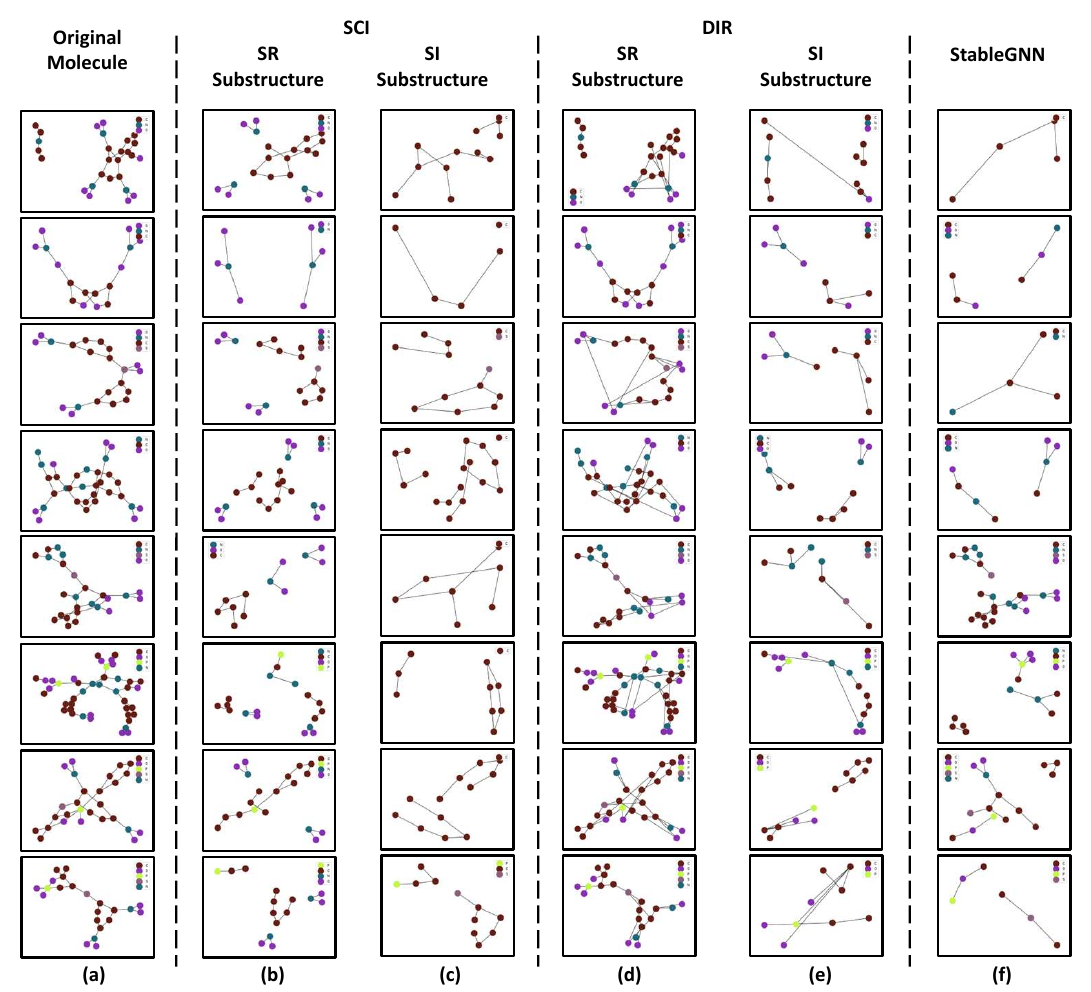}
 \caption{Visualization of examples in GOOD-HIV datasets, nodes with different colors denote different atoms, and edges denote different chemical bonds. (a) denotes the original molecule structures. (b)(c) denotes the SR and SI substructures extracted from our SCI model. (d)(e) denotes the SR and SI substructures extracted from our DIR model. (f) denotes the SR extracted from our StableGNN model. \textit{(Best view in color. )}}
\label{fig:visual}
\end{figure*}

\subsubsection{\textcolor{black}{Experiment Results on MoleculeNet datasets}}
\textcolor{black}{Experiment results on MoleculeNet datasets are shown in Table \ref{tab:moleculenet}.  We conduct
the Wilcoxon signed-rank test on the reported accuracies,
our method significantly outperforms the baselines, with
a p-value threshold of 0.05. According to the experiment results, we can learn that:}
\begin{itemize}
    \item \textcolor{black}{According to the experiment results, we can find that the proposed SCI model also achieves state-of-the-art performance in most of the tasks and achieves comparable performance in the HIV datasets.}
    \item \textcolor{black}{The recently proposed AttentiveFP also performs well, this is because they learn the invariant structural information by combining the attention mechanism and the chemical reactions, which motivates us to leverage the expert knowledge.}
\end{itemize}

\subsubsection{Ablation Study}
To evaluate the effectiveness of the semantic latent substructure regularization, sparsity regularization, and the atom variables, we devise the SCI-r, SCI-h, SCI-n, and SCI-k. Figure \ref{fig:ablation} (a)(b)(c) shows the experimental results on the different datasets. 
According to the experiment results shown in Figure \ref{fig:ablation}, we draw several observations as follows.

\begin{itemize}
    \item By comparing the standard SCI and SCI-r, the standard SCI with the semantic latent substructure regularization performs better, especially on the molbace datasets. This is because the semantic latent substructure can be well extracted with theoretical guarantees.
    \item The standard SCI outperforms the SCI-h, this is because the sparsity regularization benefits the extraction of semantic latent substructures, which coincides with the prior sparsity of semantic latent substructures.
    \item The standard SCI also outperforms the SCI-n, since restricting the sparsity of $G_{ir}$ can well generate augmented samples, the proposed SCI model can well extract the semantic latent substructures $G_r$.
    \item \textcolor{black}{We also find that the atom latent variables play an important role in the performance by comparing the experiment results of SCI and SCI-k, meaning that the identification of the atom latent variables can strengthen the prediction performance.}
\end{itemize}






\subsection{\textcolor{black}{Sensitive Analysis of Hyper-parameters}}

\textcolor{black}{To explore the importance of $\alpha$, $\beta$, and $\gamma$ in Equation (\ref{equ:total_loss}), we conduct experiments for the sensitivity of these hyperparameters.}

\textcolor{black}{First, to explore the importance of $\alpha$, we try different values of $\alpha$, and the experiment results of Bace, Sider, and Clintox datasets in MoleculeNet benchmark are shown in Figure \ref{fig:sensitity}(a). According to the experiment results, we can draw the following conclusions: 1) the proposed SCI model outperforms most of the baselines under different values of $\alpha$. 2) When the value of $\alpha$ is around 1.0, our SCI model achieves the best results, this is because an appropriate hyperparameter leads to better extraction of semantic-relevant substructures, which further benefits the model performance.} 

\textcolor{black}{We further evaluate how the different values of $\beta$ and $\gamma$ affect the model performance. To achieve this, we try different values of $\beta$ and $\gamma$ as shown in Figure \ref{fig:sensitity}(b) (c). According to the experiment results, we can draw the following conclusions. 1) We can find that the experiment results of SCI($\beta=2$) are better than that of SCI-h, which shows the effectiveness of the reasonable substructure sparseness regularization term. 2) With appropriate values of $\beta$ and $\gamma$, the model achieves the best performance, meaning that the reasonable sparsity of semantic-relevant substructures benefits the model performance. 3) With the increase of the value of $\beta$, the performance drops. This is because the too-heavy penalty of the sparseness of the semantic-relevant substructure might lead to the loss of semantic information, which further influences the model generalization.}


\subsection{Case Study}




We further provide two intuitive case studies to verify the soundness of the proposed method in Figure \ref{fig:case}. The molecules shown in Figure \ref{fig:case} (a)(c) denote the original structures in the OGBG-bbbp dataset, and the substructures shown in Figure \ref{fig:case} (b)(d) denote the corresponding $G_h$ extracted by the proposed method.
From the aforementioned results, We draw several interesting observations.

First, the proposed MCI method is able to learn the property-relevant functional groups from training data. According to our observation from the compounds from \cite{carpenter2014method}, the compounds containing the carbocycle with nitrogen atoms may have the ability to cross the blood-brain barrier (BBB), coinciding with the results shown in Figure 4(b)(d). Hence, the proposed method has the potential to provide explanations for model prediction and references for biochemistry. Second, the benzene ring (the red area in Figure 6(a)(c)) is a common functional group among organic compounds and may be property-irrelevant, which might result in spurious correlations. According to the experiment results shown in Figure 6, we can find that the proposed \textbf{SCI} model can remove these spurious correlations.

\subsection{\textcolor{black}{Visualization}}

\textcolor{black}{To further investigate our approach, we perform visualization of the semantic-relevant and semantic-irrelevant substructures of the proposed SCI, DIR, and StableGNN. Note that StableGNN can only extract the semantic-relevant substructures. Visualizations in Figure \ref{fig:visual} can not only provide an extra advantage of specific interpretation for the molecular property prediction results but also potentially enhance the practical human comprehension in the field of chemistry. According to the visualization results, we can draw the following conclusions.}

\begin{itemize}
    \item \textcolor{black}{Compared with the original molecular structures, we can find that the semantic-relevant substructures are sparse and basic structures, reflecting that our method can generate identifiable latent substructures.}
    \item \textcolor{black}{We can also find that the proposed SCI model can generate reasonable molecular substructures, which contain the basic functional groups like in chemistry. For example, our method can easily extract substructures like ``$-NO_2$'' (i.e., the substructure with two purple nodes and one blue node), which might provide potential motivation for chemical science.}
    \item \textcolor{black}{Moreover, our SCI model can remove some less important atoms, for example, the $S$ atom does not appear in both SR and SI structures. Moreover, the SR and SI substructures are not complementary, this is because there is not any restriction in the model to make these complementary. However, the SI substructures can still provide any interpretation for prediction results.}
    \item \textcolor{black}{ Compared with Figure \ref{fig:visual}(d)(e)  we can find that DIR can hardly extract the reasonable semantic-relevant substructures (some of the semantic-relevant substructures are similar to the original structures.) and might generate some inexistent or unreasonable semantic-irrelevant substructures.}
    \item \textcolor{black}{Compared with Figure \ref{fig:visual}(f), the StableGNN can also extract some meaningful semantic-relevant extract, but some of them might be too sparse, which might result in the loss of semantic information and further the suboptimal prediction performance.}
\end{itemize}




\section{Conclusion}
This paper presents a semantic components identification model that addresses the out-of-distribution shift problem in molecular property prediction. To achieve it, two identification theorems are provided to guarantee that the semantic-relevant substructures and the atom latent variables are identifiable by modeling the proposed causal mechanism for molecular data. The theoretical results further guide us in devising a practical model. 
The success of our proposed approach extends beyond improved molecular property prediction. It also yields valuable insights, explanations, and references that can be applied to the field of biochemistry. By integrating causality and chemistry, our method represents a meaningful step forward in understanding the intricate relationships between these two domains. Moreover, the study emphasizes the potential depth of utilizing domain knowledge, presenting a stimulating prospect for future research in this interdisciplinary field.


\ifCLASSOPTIONcompsoc
  \section*{Acknowledgments}
\else
  \section*{Acknowledgment}
\fi
 The authors would like to thank Zhifan Jiang, Kaitao Zhen, Haiqin Huang, and Haozhi Chen from the Guangdong University of Technology for their help in this work.

\bibliographystyle{IEEEtran}
\bibliography{main}

\begin{thebibliography}{10}
\providecommand{\url}[1]{#1}
\csname url@samestyle\endcsname
\providecommand{\newblock}{\relax}
\providecommand{\bibinfo}[2]{#2}
\providecommand{\BIBentrySTDinterwordspacing}{\spaceskip=0pt\relax}
\providecommand{\BIBentryALTinterwordstretchfactor}{4}
\providecommand{\BIBentryALTinterwordspacing}{\spaceskip=\fontdimen2\font plus
\BIBentryALTinterwordstretchfactor\fontdimen3\font minus
  \fontdimen4\font\relax}
\providecommand{\BIBforeignlanguage}[2]{{%
\expandafter\ifx\csname l@#1\endcsname\relax
\typeout{** WARNING: IEEEtran.bst: No hyphenation pattern has been}%
\typeout{** loaded for the language `#1'. Using the pattern for}%
\typeout{** the default language instead.}%
\else
\language=\csname l@#1\endcsname
\fi
#2}}
\providecommand{\BIBdecl}{\relax}
\BIBdecl

\bibitem{jumper2021highly}
J.~Jumper, R.~Evans, A.~Pritzel, T.~Green, M.~Figurnov, O.~Ronneberger,
  K.~Tunyasuvunakool, R.~Bates, A.~{\v{Z}}{\'\i}dek, A.~Potapenko
  \emph{et~al.}, ``Highly accurate protein structure prediction with
  alphafold,'' \emph{Nature}, vol. 596, no. 7873, pp. 583--589, 2021.

\bibitem{atz2021geometric}
K.~Atz, F.~Grisoni, and G.~Schneider, ``Geometric deep learning on molecular
  representations,'' \emph{Nature Machine Intelligence}, vol.~3, no.~12, pp.
  1023--1032, 2021.

\bibitem{de2019synthetic}
A.~F. de~Almeida, R.~Moreira, and T.~Rodrigues, ``Synthetic organic chemistry
  driven by artificial intelligence,'' \emph{Nature Reviews Chemistry}, vol.~3,
  no.~10, pp. 589--604, 2019.

\bibitem{xia2022pre}
J.~Xia, Y.~Zhu, Y.~Du, and S.~Z. Li, ``Pre-training graph neural networks for
  molecular representations: retrospect and prospect,'' in \emph{ICML 2022 2nd
  AI for Science Workshop}, 2022.

\bibitem{vamathevan2019applications}
J.~Vamathevan, D.~Clark, P.~Czodrowski, I.~Dunham, E.~Ferran, G.~Lee, B.~Li,
  A.~Madabhushi, P.~Shah, M.~Spitzer \emph{et~al.}, ``Applications of machine
  learning in drug discovery and development,'' \emph{Nature reviews Drug
  discovery}, vol.~18, no.~6, pp. 463--477, 2019.

\bibitem{ji2023drugood}
Y.~Ji, L.~Zhang, J.~Wu, B.~Wu, L.~Li, L.-K. Huang, T.~Xu, Y.~Rong, J.~Ren,
  D.~Xue \emph{et~al.}, ``Drugood: Out-of-distribution dataset curator and
  benchmark for ai-aided drug discovery--a focus on affinity prediction
  problems with noise annotations,'' in \emph{Proceedings of the AAAI
  Conference on Artificial Intelligence}, vol.~37, no.~7, 2023, pp. 8023--8031.

\bibitem{lu2022tankbind}
W.~Lu, Q.~Wu, J.~Zhang, J.~Rao, C.~Li, and S.~Zheng, ``Tankbind:
  Trigonometry-aware neural networks for drug-protein binding structure
  prediction,'' \emph{Advances in neural information processing systems},
  vol.~35, pp. 7236--7249, 2022.

\bibitem{liu2023interpretable}
Y.~L. Liu, Y.~Wang, O.~Vu, R.~Moretti, B.~Bodenheimer, J.~Meiler, and T.~Derr,
  ``Interpretable chirality-aware graph neural network for quantitative
  structure activity relationship modeling in drug discovery,'' in
  \emph{Proceedings of the AAAI Conference on Artificial Intelligence},
  vol.~37, no.~12, 2023, pp. 14\,356--14\,364.

\bibitem{xu2023graph}
M.~Xu, M.~Liu, W.~Jin, S.~Ji, J.~Leskovec, and S.~Ermon, ``Graph and geometry
  generative modeling for drug discovery,'' in \emph{Proceedings of the 29th
  ACM SIGKDD Conference on Knowledge Discovery and Data Mining}, 2023, pp.
  5833--5834.

\bibitem{butler2018machine}
K.~T. Butler, D.~W. Davies, H.~Cartwright, O.~Isayev, and A.~Walsh, ``Machine
  learning for molecular and materials science,'' \emph{Nature}, vol. 559, no.
  7715, pp. 547--555, 2018.

\bibitem{neese2009prediction}
F.~Neese, ``Prediction of molecular properties and molecular spectroscopy with
  density functional theory: From fundamental theory to exchange-coupling,''
  \emph{Coordination Chemistry Reviews}, vol. 253, no. 5-6, pp. 526--563, 2009.

\bibitem{velivckovic2017graph}
P.~Veli{\v{c}}kovi{\'c}, G.~Cucurull, A.~Casanova, A.~Romero, P.~Lio, and
  Y.~Bengio, ``Graph attention networks,'' \emph{arXiv preprint
  arXiv:1710.10903}, 2017.

\bibitem{battaglia2018relational}
P.~W. Battaglia, J.~B. Hamrick, V.~Bapst, A.~Sanchez-Gonzalez, V.~Zambaldi,
  M.~Malinowski, A.~Tacchetti, D.~Raposo, A.~Santoro, R.~Faulkner
  \emph{et~al.}, ``Relational inductive biases, deep learning, and graph
  networks,'' \emph{arXiv preprint arXiv:1806.01261}, 2018.

\bibitem{lu2019molecular}
C.~Lu, Q.~Liu, C.~Wang, Z.~Huang, P.~Lin, and L.~He, ``Molecular property
  prediction: A multilevel quantum interactions modeling perspective,'' in
  \emph{Proceedings of the AAAI Conference on Artificial Intelligence},
  vol.~33, no.~01, 2019, pp. 1052--1060.

\bibitem{li2022geomgcl}
S.~Li, J.~Zhou, T.~Xu, D.~Dou, and H.~Xiong, ``Geomgcl: Geometric graph
  contrastive learning for molecular property prediction,'' in
  \emph{Proceedings of the AAAI Conference on Artificial Intelligence},
  vol.~36, no.~4, 2022, pp. 4541--4549.

\bibitem{hu2020pretraining}
\BIBentryALTinterwordspacing
W.~Hu, B.~Liu, J.~Gomes, M.~Zitnik, P.~Liang, V.~Pande, and J.~Leskovec,
  ``Strategies for pre-training graph neural networks,'' in \emph{International
  Conference on Learning Representations}, 2020. [Online]. Available:
  \url{https://openreview.net/forum?id=HJlWWJSFDH}
\BIBentrySTDinterwordspacing

\bibitem{rong2020self}
Y.~Rong, Y.~Bian, T.~Xu, W.~Xie, Y.~Wei, W.~Huang, and J.~Huang,
  ``Self-supervised graph transformer on large-scale molecular data,''
  \emph{Advances in Neural Information Processing Systems}, vol.~33, pp.
  12\,559--12\,571, 2020.

\bibitem{ying2018hierarchical}
Z.~Ying, J.~You, C.~Morris, X.~Ren, W.~Hamilton, and J.~Leskovec,
  ``Hierarchical graph representation learning with differentiable pooling,''
  \emph{Advances in neural information processing systems}, vol.~31, 2018.

\bibitem{li2022ood}
H.~Li, X.~Wang, Z.~Zhang, and W.~Zhu, ``Ood-gnn: Out-of-distribution
  generalized graph neural network,'' \emph{IEEE Transactions on Knowledge and
  Data Engineering}, 2022.

\bibitem{liu2022graph}
G.~Liu, T.~Zhao, J.~Xu, T.~Luo, and M.~Jiang, ``Graph rationalization with
  environment-based augmentations,'' in \emph{Proceedings of the 28th ACM
  SIGKDD Conference on Knowledge Discovery and Data Mining}, 2022, pp.
  1069--1078.

\bibitem{li2022out}
H.~Li, X.~Wang, Z.~Zhang, and W.~Zhu, ``Out-of-distribution generalization on
  graphs: A survey,'' \emph{arXiv preprint arXiv:2202.07987}, 2022.

\bibitem{fan2023generalizing}
S.~Fan, X.~Wang, C.~Shi, P.~Cui, and B.~Wang, ``Generalizing graph neural
  networks on out-of-distribution graphs,'' \emph{IEEE Transactions on Pattern
  Analysis and Machine Intelligence}, 2023.

\bibitem{fang2022invariant}
Z.~Fang, Z.~Zhang, G.~Song, Y.~Zhang, D.~Li, J.~Hao, and X.~Wang, ``Invariant
  factor graph neural networks,'' in \emph{2022 IEEE International Conference
  on Data Mining (ICDM)}.\hskip 1em plus 0.5em minus 0.4em\relax IEEE, 2022,
  pp. 933--938.

\bibitem{wu2022discovering}
Y.-X. Wu, X.~Wang, A.~Zhang, X.~He, and T.-S. Chua, ``Discovering invariant
  rationales for graph neural networks,'' \emph{arXiv preprint
  arXiv:2201.12872}, 2022.

\bibitem{pearl2009causality}
J.~Pearl, \emph{Causality}.\hskip 1em plus 0.5em minus 0.4em\relax Cambridge
  university press, 2009.

\bibitem{liu2021transferable}
Z.~Liu, L.~Lin, Q.~Jia, Z.~Cheng, Y.~Jiang, Y.~Guo, and J.~Ma, ``Transferable
  multilevel attention neural network for accurate prediction of quantum
  chemistry properties via multitask learning,'' \emph{Journal of Chemical
  Information and Modeling}, vol.~61, no.~3, pp. 1066--1082, 2021.

\bibitem{zhang2021motif}
Z.~Zhang, Q.~Liu, H.~Wang, C.~Lu, and C.-K. Lee, ``Motif-based graph
  self-supervised learning for molecular property prediction,'' \emph{Advances
  in Neural Information Processing Systems}, vol.~34, pp. 15\,870--15\,882,
  2021.

\bibitem{lee2023exploring}
S.~Lee, J.~Jo, and S.~J. Hwang, ``Exploring chemical space with score-based
  out-of-distribution generation,'' in \emph{International Conference on
  Machine Learning}.\hskip 1em plus 0.5em minus 0.4em\relax PMLR, 2023, pp.
  18\,872--18\,892.

\bibitem{stark20223d}
H.~St{\"a}rk, D.~Beaini, G.~Corso, P.~Tossou, C.~Dallago, S.~G{\"u}nnemann, and
  P.~Li{\`o}, ``3d infomax improves gnns for molecular property prediction,''
  in \emph{International Conference on Machine Learning}.\hskip 1em plus 0.5em
  minus 0.4em\relax PMLR, 2022, pp. 20\,479--20\,502.

\bibitem{li2022kpgt}
H.~Li, D.~Zhao, and J.~Zeng, ``Kpgt: knowledge-guided pre-training of graph
  transformer for molecular property prediction,'' in \emph{Proceedings of the
  28th ACM SIGKDD Conference on Knowledge Discovery and Data Mining}, 2022, pp.
  857--867.

\bibitem{wieder2020compact}
O.~Wieder, S.~Kohlbacher, M.~Kuenemann, A.~Garon, P.~Ducrot, T.~Seidel, and
  T.~Langer, ``A compact review of molecular property prediction with graph
  neural networks,'' \emph{Drug Discovery Today: Technologies}, vol.~37, pp.
  1--12, 2020.

\bibitem{doi:10.1021/ci100050t}
\BIBentryALTinterwordspacing
D.~Rogers and M.~Hahn, ``Extended-connectivity fingerprints,'' \emph{Journal of
  Chemical Information and Modeling}, vol.~50, no.~5, pp. 742--754, 2010, pMID:
  20426451. [Online]. Available: \url{https://doi.org/10.1021/ci100050t}
\BIBentrySTDinterwordspacing

\bibitem{10.1093/bioinformatics/btab195}
\BIBentryALTinterwordspacing
Z.~Zhang, J.~Guan, and S.~Zhou, ``{FraGAT: a fragment-oriented multi-scale
  graph attention model for molecular property prediction},''
  \emph{Bioinformatics}, vol.~37, no.~18, pp. 2981--2987, 03 2021. [Online].
  Available: \url{https://doi.org/10.1093/bioinformatics/btab195}
\BIBentrySTDinterwordspacing

\bibitem{yang2022learning}
N.~Yang, K.~Zeng, Q.~Wu, X.~Jia, and J.~Yan, ``Learning substructure invariance
  for out-of-distribution molecular representations,'' \emph{Advances in Neural
  Information Processing Systems}, vol.~35, pp. 12\,964--12\,978, 2022.

\bibitem{guo2020graseq}
Z.~Guo, W.~Yu, C.~Zhang, M.~Jiang, and N.~V. Chawla, ``Graseq: graph and
  sequence fusion learning for molecular property prediction,'' in
  \emph{Proceedings of the 29th ACM international conference on information \&
  knowledge management}, 2020, pp. 435--443.

\bibitem{chen2022learning}
Y.~Chen, Y.~Zhang, Y.~Bian, H.~Yang, M.~Kaili, B.~Xie, T.~Liu, B.~Han, and
  J.~Cheng, ``Learning causally invariant representations for
  out-of-distribution generalization on graphs,'' \emph{Advances in Neural
  Information Processing Systems}, vol.~35, pp. 22\,131--22\,148, 2022.

\bibitem{10027780}
Z.~Fang, Z.~Zhang, G.~Song, Y.~Zhang, D.~Li, J.~Hao, and X.~Wang, ``Invariant
  factor graph neural networks,'' in \emph{2022 IEEE International Conference
  on Data Mining (ICDM)}, 2022, pp. 933--938.

\bibitem{cai2021graph}
R.~Cai, F.~Wu, Z.~Li, P.~Wei, L.~Yi, and K.~Zhang, ``Graph domain adaptation: A
  generative view,'' \emph{arXiv preprint arXiv:2106.07482}, 2021.

\bibitem{shen2021towards}
Z.~Shen, J.~Liu, Y.~He, X.~Zhang, R.~Xu, H.~Yu, and P.~Cui, ``Towards
  out-of-distribution generalization: A survey,'' \emph{arXiv preprint
  arXiv:2108.13624}, 2021.

\bibitem{zhang2021deep}
X.~Zhang, P.~Cui, R.~Xu, L.~Zhou, Y.~He, and Z.~Shen, ``Deep stable learning
  for out-of-distribution generalization,'' in \emph{Proceedings of the
  IEEE/CVF Conference on Computer Vision and Pattern Recognition}, 2021, pp.
  5372--5382.

\bibitem{zhang2022multi}
W.~Zhang, X.~Zhang, M.-L. Zhang \emph{et~al.}, ``Multi-instance causal
  representation learning for instance label prediction and out-of-distribution
  generalization,'' \emph{Advances in Neural Information Processing Systems},
  vol.~35, pp. 34\,940--34\,953, 2022.

\bibitem{chen2021hiddencut}
J.~Chen, D.~Shen, W.~Chen, and D.~Yang, ``Hiddencut: Simple data augmentation
  for natural language understanding with better generalizability,'' in
  \emph{Proceedings of the 59th Annual Meeting of the Association for
  Computational Linguistics and the 11th International Joint Conference on
  Natural Language Processing (Volume 1: Long Papers)}, 2021, pp. 4380--4390.

\bibitem{li2022graphde}
Z.~Li, Q.~Wu, F.~Nie, and J.~Yan, ``Graphde: A generative framework for
  debiased learning and out-of-distribution detection on graphs,''
  \emph{Advances in Neural Information Processing Systems}, vol.~35, pp.
  30\,277--30\,290, 2022.

\bibitem{zhao2020uncertainty}
X.~Zhao, F.~Chen, S.~Hu, and J.-H. Cho, ``Uncertainty aware semi-supervised
  learning on graph data,'' \emph{Advances in Neural Information Processing
  Systems}, vol.~33, pp. 12\,827--12\,836, 2020.

\bibitem{liu2023flood}
Y.~Liu, X.~Ao, F.~Feng, Y.~Ma, K.~Li, T.-S. Chua, and Q.~He, ``Flood: A
  flexible invariant learning framework for out-of-distribution generalization
  on graphs,'' in \emph{Proceedings of the 29th ACM SIGKDD Conference on
  Knowledge Discovery and Data Mining}, 2023, pp. 1548--1558.

\bibitem{sui2022causal}
Y.~Sui, X.~Wang, J.~Wu, M.~Lin, X.~He, and T.-S. Chua, ``Causal attention for
  interpretable and generalizable graph classification,'' in \emph{Proceedings
  of the 28th ACM SIGKDD Conference on Knowledge Discovery and Data Mining},
  2022, pp. 1696--1705.

\bibitem{sun2022does}
R.~Sun, H.~Dai, and A.~W. Yu, ``Does gnn pretraining help molecular
  representation?'' \emph{Advances in Neural Information Processing Systems},
  vol.~35, pp. 12\,096--12\,109, 2022.

\bibitem{li2021ood}
H.~Li, X.~Wang, Z.~Zhang, and W.~Zhu, ``Ood-gnn: Out-of-distribution
  generalized graph neural network,'' \emph{arXiv preprint arXiv:2112.03806},
  2021.

\bibitem{10.1145/3534678.3539347}
\BIBentryALTinterwordspacing
G.~Liu, T.~Zhao, J.~Xu, T.~Luo, and M.~Jiang, ``Graph rationalization with
  environment-based augmentations,'' in \emph{Proceedings of the 28th ACM
  SIGKDD Conference on Knowledge Discovery and Data Mining}, ser. KDD
  '22.\hskip 1em plus 0.5em minus 0.4em\relax New York, NY, USA: Association
  for Computing Machinery, 2022, p. 1069–1078. [Online]. Available:
  \url{https://doi.org/10.1145/3534678.3539347}
\BIBentrySTDinterwordspacing

\bibitem{scholkopf2021toward}
B.~Sch{\"o}lkopf, F.~Locatello, S.~Bauer, N.~R. Ke, N.~Kalchbrenner, A.~Goyal,
  and Y.~Bengio, ``Toward causal representation learning,'' \emph{Proceedings
  of the IEEE}, vol. 109, no.~5, pp. 612--634, 2021.

\bibitem{kumar2017variational}
A.~Kumar, P.~Sattigeri, and A.~Balakrishnan, ``Variational inference of
  disentangled latent concepts from unlabeled observations,'' \emph{arXiv
  preprint arXiv:1711.00848}, 2017.

\bibitem{locatello2019challenging}
F.~Locatello, S.~Bauer, M.~Lucic, G.~Raetsch, S.~Gelly, B.~Sch{\"o}lkopf, and
  O.~Bachem, ``Challenging common assumptions in the unsupervised learning of
  disentangled representations,'' in \emph{international conference on machine
  learning}.\hskip 1em plus 0.5em minus 0.4em\relax PMLR, 2019, pp. 4114--4124.

\bibitem{locatello2019disentangling}
F.~Locatello, M.~Tschannen, S.~Bauer, G.~R{\"a}tsch, B.~Sch{\"o}lkopf, and
  O.~Bachem, ``Disentangling factors of variation using few labels,''
  \emph{arXiv preprint arXiv:1905.01258}, 2019.

\bibitem{zheng2022identifiability}
Y.~Zheng, I.~Ng, and K.~Zhang, ``On the identifiability of nonlinear ica with
  unconditional priors,'' in \emph{ICLR2022 Workshop on the Elements of
  Reasoning: Objects, Structure and Causality}, 2022.

\bibitem{trauble2021disentangled}
F.~Tr{\"a}uble, E.~Creager, N.~Kilbertus, F.~Locatello, A.~Dittadi, A.~Goyal,
  B.~Sch{\"o}lkopf, and S.~Bauer, ``On disentangled representations learned
  from correlated data,'' in \emph{International Conference on Machine
  Learning}.\hskip 1em plus 0.5em minus 0.4em\relax PMLR, 2021, pp.
  10\,401--10\,412.

\bibitem{hyvarinen2002independent}
A.~Hyvarinen, J.~Karhunen, and E.~Oja, ``Independent component analysis,''
  \emph{Studies in informatics and control}, vol.~11, no.~2, pp. 205--207,
  2002.

\bibitem{hyvarinen2013independent}
A.~Hyv{\"a}rinen, ``Independent component analysis: recent advances,''
  \emph{Philosophical Transactions of the Royal Society A: Mathematical,
  Physical and Engineering Sciences}, vol. 371, no. 1984, p. 20110534, 2013.

\bibitem{zhang2008minimal}
K.~Zhang and L.~Chan, ``Minimal nonlinear distortion principle for nonlinear
  independent component analysis,'' \emph{Journal of Machine Learning
  Research}, vol.~9, no. Nov, pp. 2455--2487, 2008.

\bibitem{zhang2007kernel}
------, ``Kernel-based nonlinear independent component analysis,'' in
  \emph{Independent Component Analysis and Signal Separation: 7th International
  Conference, ICA 2007, London, UK, September 9-12, 2007. Proceedings 7}.\hskip
  1em plus 0.5em minus 0.4em\relax Springer, 2007, pp. 301--308.

\bibitem{xiemulti}
S.~Xie, L.~Kong, M.~Gong, and K.~Zhang, ``Multi-domain image generation and
  translation with identifiability guarantees,'' in \emph{The Eleventh
  International Conference on Learning Representations}.

\bibitem{comon1994independent}
P.~Comon, ``Independent component analysis, a new concept?'' \emph{Signal
  processing}, vol.~36, no.~3, pp. 287--314, 1994.

\bibitem{hyvarinen2016unsupervised}
A.~Hyvarinen and H.~Morioka, ``Unsupervised feature extraction by
  time-contrastive learning and nonlinear ica,'' \emph{Advances in neural
  information processing systems}, vol.~29, 2016.

\bibitem{hyvarinen2017nonlinear}
------, ``Nonlinear ica of temporally dependent stationary sources,'' in
  \emph{Artificial Intelligence and Statistics}.\hskip 1em plus 0.5em minus
  0.4em\relax PMLR, 2017, pp. 460--469.

\bibitem{hyvarinen2019nonlinear}
A.~Hyvarinen, H.~Sasaki, and R.~Turner, ``Nonlinear ica using auxiliary
  variables and generalized contrastive learning,'' in \emph{The 22nd
  International Conference on Artificial Intelligence and Statistics}.\hskip
  1em plus 0.5em minus 0.4em\relax PMLR, 2019, pp. 859--868.

\bibitem{khemakhem2020variational}
I.~Khemakhem, D.~Kingma, R.~Monti, and A.~Hyvarinen, ``Variational autoencoders
  and nonlinear ica: A unifying framework,'' in \emph{International Conference
  on Artificial Intelligence and Statistics}.\hskip 1em plus 0.5em minus
  0.4em\relax PMLR, 2020, pp. 2207--2217.

\bibitem{halva2021disentangling}
H.~H{\"a}lv{\"a}, S.~Le~Corff, L.~Leh{\'e}ricy, J.~So, Y.~Zhu, E.~Gassiat, and
  A.~Hyvarinen, ``Disentangling identifiable features from noisy data with
  structured nonlinear ica,'' \emph{Advances in Neural Information Processing
  Systems}, vol.~34, pp. 1624--1633, 2021.

\bibitem{halva2020hidden}
H.~H{\"a}lv{\"a} and A.~Hyvarinen, ``Hidden markov nonlinear ica: Unsupervised
  learning from nonstationary time series,'' in \emph{Conference on Uncertainty
  in Artificial Intelligence}.\hskip 1em plus 0.5em minus 0.4em\relax PMLR,
  2020, pp. 939--948.

\bibitem{kong2022partial}
L.~Kong, S.~Xie, W.~Yao, Y.~Zheng, G.~Chen, P.~Stojanov, V.~Akinwande, and
  K.~Zhang, ``Partial disentanglement for domain adaptation,'' in
  \emph{International Conference on Machine Learning}.\hskip 1em plus 0.5em
  minus 0.4em\relax PMLR, 2022, pp. 11\,455--11\,472.

\bibitem{yao2022temporally}
W.~Yao, G.~Chen, and K.~Zhang, ``Temporally disentangled representation
  learning,'' \emph{arXiv preprint arXiv:2210.13647}, 2022.

\bibitem{yao2021learning}
W.~Yao, Y.~Sun, A.~Ho, C.~Sun, and K.~Zhang, ``Learning temporally causal
  latent processes from general temporal data,'' \emph{arXiv preprint
  arXiv:2110.05428}, 2021.

\bibitem{li2023subspace}
Z.~Li, R.~Cai, G.~Chen, B.~Sun, Z.~Hao, and K.~Zhang, ``Subspace identification
  for multi-source domain adaptation,'' \emph{arXiv preprint arXiv:2310.04723},
  2023.

\bibitem{von2021self}
J.~Von~K{\"u}gelgen, Y.~Sharma, L.~Gresele, W.~Brendel, B.~Sch{\"o}lkopf,
  M.~Besserve, and F.~Locatello, ``Self-supervised learning with data
  augmentations provably isolates content from style,'' \emph{Advances in
  neural information processing systems}, vol.~34, pp. 16\,451--16\,467, 2021.

\bibitem{zimmermann2021contrastive}
R.~S. Zimmermann, Y.~Sharma, S.~Schneider, M.~Bethge, and W.~Brendel,
  ``Contrastive learning inverts the data generating process,'' in
  \emph{International Conference on Machine Learning}.\hskip 1em plus 0.5em
  minus 0.4em\relax PMLR, 2021, pp. 12\,979--12\,990.

\bibitem{jang2016categorical}
E.~Jang, S.~Gu, and B.~Poole, ``Categorical reparameterization with
  gumbel-softmax,'' \emph{arXiv preprint arXiv:1611.01144}, 2016.

\bibitem{hu2020open}
W.~Hu, M.~Fey, M.~Zitnik, Y.~Dong, H.~Ren, B.~Liu, M.~Catasta, and J.~Leskovec,
  ``Open graph benchmark: Datasets for machine learning on graphs,''
  \emph{Advances in neural information processing systems}, vol.~33, pp.
  22\,118--22\,133, 2020.

\bibitem{gui2022good}
S.~Gui, X.~Li, L.~Wang, and S.~Ji, ``Good: A graph out-of-distribution
  benchmark,'' \emph{Advances in Neural Information Processing Systems},
  vol.~35, pp. 2059--2073, 2022.

\bibitem{wu2018moleculenet}
Z.~Wu, B.~Ramsundar, E.~N. Feinberg, J.~Gomes, C.~Geniesse, A.~S. Pappu,
  K.~Leswing, and V.~Pande, ``Moleculenet: a benchmark for molecular machine
  learning,'' \emph{Chemical science}, vol.~9, no.~2, pp. 513--530, 2018.

\bibitem{landrum2006rdkit}
G.~Landrum \emph{et~al.}, ``Rdkit: Open-source cheminformatics,'' 2006.

\bibitem{kipf2016semi}
T.~N. Kipf and M.~Welling, ``Semi-supervised classification with graph
  convolutional networks,'' \emph{arXiv preprint arXiv:1609.02907}, 2016.

\bibitem{velickovic2017graph}
P.~Velickovic, G.~Cucurull, A.~Casanova, A.~Romero, P.~Lio, Y.~Bengio
  \emph{et~al.}, ``Graph attention networks,'' \emph{stat}, vol. 1050, no.~20,
  pp. 10--48\,550, 2017.

\bibitem{hamilton2017inductive}
W.~Hamilton, Z.~Ying, and J.~Leskovec, ``Inductive representation learning on
  large graphs,'' \emph{Advances in neural information processing systems},
  vol.~30, 2017.

\bibitem{xu2018powerful}
K.~Xu, W.~Hu, J.~Leskovec, and S.~Jegelka, ``How powerful are graph neural
  networks?'' \emph{arXiv preprint arXiv:1810.00826}, 2018.

\bibitem{wu2019simplifying}
F.~Wu, A.~Souza, T.~Zhang, C.~Fifty, T.~Yu, and K.~Weinberger, ``Simplifying
  graph convolutional networks,'' in \emph{International conference on machine
  learning}.\hskip 1em plus 0.5em minus 0.4em\relax PMLR, 2019, pp. 6861--6871.

\bibitem{xu2018representation}
K.~Xu, C.~Li, Y.~Tian, T.~Sonobe, K.-i. Kawarabayashi, and S.~Jegelka,
  ``Representation learning on graphs with jumping knowledge networks,'' in
  \emph{International conference on machine learning}.\hskip 1em plus 0.5em
  minus 0.4em\relax PMLR, 2018, pp. 5453--5462.

\bibitem{song2020communicative}
Y.~Song, S.~Zheng, Z.~Niu, Z.-H. Fu, Y.~Lu, and Y.~Yang, ``Communicative
  representation learning on attributed molecular graphs.'' in \emph{IJCAI},
  vol. 2020, 2020, pp. 2831--2838.

\bibitem{xiong2019pushing}
Z.~Xiong, D.~Wang, X.~Liu, F.~Zhong, X.~Wan, X.~Li, Z.~Li, X.~Luo, K.~Chen,
  H.~Jiang \emph{et~al.}, ``Pushing the boundaries of molecular representation
  for drug discovery with the graph attention mechanism,'' \emph{Journal of
  medicinal chemistry}, vol.~63, no.~16, pp. 8749--8760, 2019.

\bibitem{ren2023force}
G.-P. Ren, Y.-J. Yin, K.-J. Wu, and Y.~He, ``Force field-inspired molecular
  representation learning for property prediction,'' \emph{Journal of
  Cheminformatics}, vol.~15, no.~1, p.~17, 2023.

\bibitem{jiang2023pharmacophoric}
Y.~Jiang, S.~Jin, X.~Jin, X.~Xiao, W.~Wu, X.~Liu, Q.~Zhang, X.~Zeng, G.~Yang,
  and Z.~Niu, ``Pharmacophoric-constrained heterogeneous graph transformer
  model for molecular property prediction,'' \emph{Communications Chemistry},
  vol.~6, no.~1, p.~60, 2023.

\bibitem{li2022learning}
H.~Li, Z.~Zhang, X.~Wang, and W.~Zhu, ``Learning invariant graph
  representations for out-of-distribution generalization,'' \emph{Advances in
  Neural Information Processing Systems}, vol.~35, pp. 11\,828--11\,841, 2022.

\bibitem{gui2023joint}
S.~Gui, M.~Liu, X.~Li, Y.~Luo, and S.~Ji, ``Joint learning of label and
  environment causal independence for graph out-of-distribution
  generalization,'' \emph{arXiv preprint arXiv:2306.01103}, 2023.

\bibitem{miao2022interpretable}
S.~Miao, M.~Liu, and P.~Li, ``Interpretable and generalizable graph learning
  via stochastic attention mechanism,'' in \emph{International Conference on
  Machine Learning}.\hskip 1em plus 0.5em minus 0.4em\relax PMLR, 2022, pp.
  15\,524--15\,543.

\bibitem{carpenter2014method}
T.~S. Carpenter, D.~A. Kirshner, E.~Y. Lau, S.~E. Wong, J.~P. Nilmeier, and
  F.~C. Lightstone, ``A method to predict blood-brain barrier permeability of
  drug-like compounds using molecular dynamics simulations,'' \emph{Biophysical
  journal}, vol. 107, no.~3, pp. 630--641, 2014.

\bibitem{potapczynski2020invertible}
A.~Potapczynski, G.~Loaiza-Ganem, and J.~P. Cunningham, ``Invertible gaussian
  reparameterization: Revisiting the gumbel-softmax,'' \emph{Advances in Neural
  Information Processing Systems}, vol.~33, pp. 12\,311--12\,321, 2020.

\bibitem{darmois1951analyse}
G.~Darmois, ``Analyse des liaisons de probabilit{\'e},'' in \emph{Proc. Int.
  Stat. Conferences 1947}, 1951, p. 231.

\bibitem{hyvarinen1999nonlinear}
A.~Hyv{\"a}rinen and P.~Pajunen, ``Nonlinear independent component analysis:
  Existence and uniqueness results,'' \emph{Neural networks}, vol.~12, no.~3,
  pp. 429--439, 1999.

\end{thebibliography}

\ifCLASSOPTIONcaptionsoff
  \newpage
\fi

\title{Appendix for ``Identifying Semantic Component for Robust Molecular Property Prediction''}
%
%
%
%

\author{Zijian Li,
        Zunhong Xu, 
        Ruichu Cai*,Zhenhui Yang, Yuguang Yan, Zhifeng Hao ~\IEEEmembership{Senior Member,~IEEE,} Guangyi Chen and Kun Zhang
}
\onecolumn
%
%

\markboth{Journal of \LaTeX\ Class Files,~Vol.~14, No.~8, August~2015}%
{Shell \MakeLowercase{\textit{et al.}}: Bare Demo of IEEEtran.cls for Computer Society Journals}
\maketitle

\IEEEdisplaynontitleabstractindextext

%
\appendix
\textbf{A. Proof of Atom Latent Variables Identification}
\begin{theorem}
    \textbf{\textit{(Atom Latent Variables ($\bm{s}$) Identification) }}
    We follow the causal mechanism shown in Figure 2 and make the following assumptions:
    \begin{itemize}
        \item A1 (\underline{Smooth and Positive Density}): The probability density function of atom latent variables is smooth and positive, i.e. $P(\bm{s}|\bm{x}) > 0$.
        \item A2 (\underline{Conditional independence}): Conditioned on $\bm{x}$, each $s_i$ is independent of any other $s_j$ for $i,j\in [n],i\neq j$, i.e, $log P(\bm{s}|\bm{x})=\sum_i^n P(s_i|\bm{x})$.
        \item A3 (\underline{Linear independence}): For any $\bm{s} \in \mathcal{S} \subseteq \mathbb{R}^n$, where $\mathcal{S}$ is the range of $\bm{s}$ and $n$ is the dimension of $\bm{s}$, there exist $2n+1$ values of $\bm{x}$, i.e., $\bm{x}_j$ with $j=0,1,...,2n$, such that the $2n$ vectors $\textbf{\textsc{w}}(\bm{s},x_j)-\textbf{\textsc{w}}(\bm{s},x_0)$ with $j=1,...2n$, are linearly independent, where vector $\textbf{\textsc{w}}(\bm{s},\bm{x})$ is formalized as follows:
        \begin{equation}\nonumber
        \begin{split}
             \textbf{\textsc{w}}(\bm{s},x_j)=(&\frac{\partial\log P(s_0|\bm{x})}{\partial s_0},... \frac{\partial\log P(s_n|\bm{x})}{\partial s_n},...\frac{\partial^2\log P(s_0|\bm{x})}{\partial^2 s_0},...\frac{\partial^2\log P(s_n|\bm{x})}{\partial^2 s_n}).
        \end{split}
        \end{equation}
    \end{itemize}
    If a learned generative model $(\hat{f}_{\bm{s}},\hat{f}_h,\hat{f}_n,\hat{g}_A,\hat{g}_y,\hat{g}_k)$ assumes the same generation process shown in Figure 2 and matches the ground-truth conditional distribution, i.e. $P(\hat{k}|\bm{x})=P(k|\bm{x})$ and $\hat{k}$ denote the estimated variables, then the identifiability of the atom latent variables $\bm{s}$ is ensured, i.e., the ground-truth atom latent variables can be learned.
\end{theorem}
\begin{proof}
Since the learned generative model $(\hat{f_s},\hat{f_h},\hat{f_n},\hat{g_A},\hat{g_y},\hat{g_k})$ assumes the same generation process shown in Figure 2 and $P(k|x)=P(\hat{k}|x)$, we have:
\begin{equation}
\label{equ:the1_1}
\begin{split}
    P(\hat{k}|x) = P(k|x) &\Longleftrightarrow P(g_k^{-1}(\hat{k})|x)|\bm{J}_{g^{-1}_k}| = P(\bm{s}|x)|\bm{J}_{g^{-1}_k}|\\ &\Longleftrightarrow P(g_k^{-1} \circ \hat{g}_k(\hat{\bm{s}})|x)|\bm{J}_{g^{-1}_k}| = P(\bm{s}|x)|\bm{J}_{g^{-1}_k}|\\ &\Longleftrightarrow P(\phi(\hat{\bm{s}})|x) =  P(\bm{s}|x)
\end{split}
\end{equation}
in which $\hat{g}_k^{-1}$ denotes the estimated invertible function and $\phi:=g_k^{-1}\circ \hat{g}_k$ denotes the transformation between the ground truth atom latent variables and the estimated one. $|\bm{J}_{\hat{g}_k^{-1}}|$ denotes the absolute value of Jacobian matrix determinant of $\hat{g}_k^{-1}$. It is noted that $\phi:=\hat{g}_k^{-1}\circ g_k$ is invertible and $|\bm{J}_{\hat{g}_k^{-1}}|\neq 0$ since $\hat{g}_k^{-1}$ and $g_k$ are invertible.

According to the conditional independent assumption (A2), we have:
\begin{equation}
\label{equ:the1_2}
    P(\hat{\bm{s}}|x) =  \prod_{i=1}^{n}P(\hat{\bm{s}_i}|x), P(\bm{s}|x) =  \prod_{i=1}^{n}P(\bm{s}_i|x),
\end{equation}
\end{proof}
and we further have:
\begin{equation}
\label{equ:the1_3}
\begin{split}
    \log P(\hat{\bm{s}}|x) &=  \sum_{i=1}^{n}\log P(\hat{\bm{s}}_i|x)\\ \log P(\bm{s}|x) &=  \sum_{i=1}^{n} \log P(\bm{s}_i|x).
\end{split}
\end{equation}

Combining Equation (\ref{equ:the1_3})
with Equation (\ref{equ:the1_1}), we further have:
\begin{equation}
\label{equ:the1_4}
\begin{split}
P(\bm{s}|x)\cdot|\bm{J}_{\phi}|&=P(\hat{\bm{s}}|x)  \Longleftrightarrow \sum_{i=1}^n \log P(\bm{s}_i|x) + \log |\bm{J}_{\phi}|= \sum_{i=1}^n \log P(\hat{\bm{s}}_i|x),
\end{split}
\end{equation}
where $\bm{J}_{\phi}$ is the Jacobian matrix of the transformation associated with $\phi$. 

We employ the following notation to define the first and second derivatives of $\phi_s$ and $\log P(\bm{s}_i|x)$, respectively.
\begin{equation}
\label{equ:the1_5}
\centering
\begin{split}
    &\phi'_{i,(j)}:=\frac{\partial s_i}{\partial \hat{s}_j}, \qquad\qquad \qquad \quad \phi''_{i,(j,l)}:=\frac{\partial^2 s_i}{\partial \hat{s}_j\partial \hat{s}_l}\\
    &\eta'_i(s_i,x):=\frac{\partial\log P(s_i|x) }{\partial s_i }\qquad \; \eta''_i(s_i,x):=\frac{\partial^2 \log P(s_i|x)}{(\partial s_i)^2}
\end{split}
\end{equation}

Then we derive Equation (\ref{equ:the1_6}) by differentiating both sides of Equation (\ref{equ:the1_4}) twice w.r.t $\hat{s}_j$ and $\hat{s}_l$ where $j,l \in [n]$ and $j \neq l$.
\begin{equation}
\label{equ:the1_6}
\begin{split}
    \sum_{i=1}^n \left(\eta''_i(s_i,x)\cdot \phi'_{i,(j)}\cdot \phi'_{i,(l)} + \eta'_i(s_i,x)\cdot \phi''_{i,(j,l)} \right) + \frac{\partial^2 \log |\bm{J}_h|}{\partial\hat{s}_j\partial\hat{s}_l}=0.
\end{split}
\end{equation}

According to the linear independence assumption, there exist $2n+1$ values of $x=\{x_0,x_1,\cdots,x_{2n}\}$, we have $2n+1$ equations by different values of $x$ in Equation (\ref{equ:the1_6}). Subtracting each equation corresponding to $x_1, \cdots, x_{2n}$ with the equation corresponding to $x+0$, we have the following equation:
\begin{equation}
\label{equ:the1_7}
\begin{split}
    \sum_{i=1}^n\left( (\eta''_i(s_i,x_t) - \eta''_i(s_i,x_0))\cdot  \phi'_{i,(j)} \phi'_{i,(l)}\right. \left. + (\eta'_i(s_i,x_t)-\eta'_i(s_i,x_0))\cdot \phi''_{i,(j,l)}\right)=0,
\end{split}
\end{equation}
in which $t=1,\cdots,2n$. We can further consider Equation (\ref{equ:the1_7}) as homogeneous linear equations with the linearly independent restriction. As a result, the only solution is $\phi'_{i,(j)} \phi'_{i,(l)}=0$ and $ \phi''_{i,(j,l)}=0$ for $i=1,\cdot,n$ and $j,l\in[n],j \neq l$. 

Then we have the Jacobian of $h$ shown as follows:
\begin{equation}
    \bm{J}_h = 
    \begin{bmatrix}
        \frac{\partial s_1}{\partial \hat{s_1}} \cdots \frac{\partial s_1}{\partial \hat{s_n}}\\
        \vdots \quad \ddots \quad  \vdots\\
        \frac{\partial {s_n}}{\partial \hat{s_1}} \cdots \frac{\partial s_n}{\partial \hat{s_n}}.
    \end{bmatrix}
\end{equation}

Since $\phi'_{i,(j)} \phi'_{i,(l)}=0$, there is at most one element $j\in[n]$, resulting in $\phi'_{i,(j)}\neq 0$. Therefore, there is at most one non-zero element in each row of the Jacobian matrix $\bm{J}_\phi$. 
According to the inverse function theorem, since $\phi$ is invertible, the inverse matrix of $J_\phi$ exists and equals $J_{\phi^{-1}}$. Since $J_\phi$ is invertible, $J_\phi$ is full-rank, implying that there is exactly one non-zero element in each row of the Jacobian matrix $\bm{J}_\phi$ and in each column of the Jacobian matrix $\bm{J}_\phi$. 
Therefore, given any true atom latent variable $s_i$, there exist a corresponding estimated variable $\hat{s}_j$ and an invertible function $\phi: \mathbb{R}\rightarrow\mathbb{R}$, such that $\hat{s}_j=\phi(s_i)$.\newline\newline


\noindent\textbf{B. Proof of Semantic Latent Substructures}

\begin{theorem}
\textbf{\textit{(Semantic Latent Substructure distribution ($\bm{B}_r$) Identification) }}
    We follow the causal generation process shown in Figure 2 and make the following assumptions:
    \begin{itemize}
        \item  A1 (\underline{Smooth and Invertible Generation Process}): $g_A:G_{ir}, G_{r} \rightarrow A$ is smooth and invertible with a smooth inverse.
        \item A2 (\underline{Smooth, Continuous and Positive Density}): $P(G_{ir}, G_{r})$ is a smooth, continuous density with $P(G_{ir}, G_{r})>0$ almost everywhere.
        \item A3 (\underline{Smooth and Positive Conditional Probability}) The conditional probability density function $P(G_{ir}'|G_{ir})$ is smooth \textit{w.r.t} both $G_{ir}$ and $G_{ir}'$; for any $G_{ir}$, $P(\cdot|G_{ir})>0$ in some open, non-empty subset containing $G_{ir}$.
        \item A4 (\underline{Identical data generation process}) A learned generative model $(\hat{f}_{\bm{s}},\hat{f}_h,\hat{f}_n,\hat{g}_A,\hat{g}_y,\hat{g}_k)$ assumes the same generation process shown in Equation (1).
    \end{itemize}
    Let $v$ be the node number and let $\tau: A \rightarrow (0,1)^{v\times v}$ be any smooth function. Then $\bm{B}_r$ can be identified by minimizing the following restriction:
    \begin{equation}
        \label{equ:the2_1}\mathcal{L}_r=\mathbb{E}_{(A,A')\sim \hat{P}(A,A')}\left[||\tau(A)-\tau(A')||^2_2\right]-H(\tau(A)),
    \end{equation}
    where $A'$ denotes the augmented sample from the proposed causal generation process and $H(\cdot)$ denotes the differential entropy of the random variables $\tau(A)$. When $\bm{B}_r$ is identified, we can obtain $G_r$ by sampling from $P_B(G_r;\bm{B}_r)$.
\end{theorem}

\begin{proof}

According to the causal generation process shown in Figure 2, we consider $G_{ir}, G_r$ as two $v^2$-dimension latent variables, i.e. $G_{ir}, G_r \in (0,1)^{v^2}$, where $v$ is the node number of molecular structure. 
So we can assume that the truth semantic latent graphs follow the multivariate Bernoulli distribution with the parameters of $\bm{B}_r \in (0,1)^{v^2}$, i.e. $G_r \sim P_B(G_r;\bm{B}_r)$, where $v$ is the node number of $G_r$. 
Similarly, we assume the noise latent substructures $G_{ir}$ are samples from $P_B(G_{ir};\bm{B}_{ir})$, where $\bm{B}_{ir}$ are the parameters. 
The proof of the identification of $\bm{B}_r$ consists of  the following three main steps:
\begin{itemize}
    \item First, we demonstrate that $\hat{\bm{B}}_r$ extracted by a smooth function by minimizing Equation (\ref{equ:the2_1}) is related to the true $\bm{B}_r,\bm{B}_{ir}$ through a smooth mapping $\psi$, i.e., $\hat{\bm{B}}_r=\psi(\bm{B}_r,\bm{B}_{ir})$. 
    \item Second, we show that $\hat{\bm{B}}_r=\psi(\bm{B}_r,\bm{B}_{ir})$ can only depend on the true $\bm{B}_r$ and not on $\bm{B}_{ir}$, i.e., $\hat{\bm{B}}_r=\psi(\bm{B})_r$
    \item Based on the result from \cite{zimmermann2021contrastive}, we show that $\psi$ must be a bijection.
\end{itemize}


\noindent\textbf{Step 1.} 
In this step, we aim to prove that there is a transformation between $\hat{\bm{B}}_r$ and the true $\bm{B}_r, \bm{B}_{ir}$. 
According to the A1, $g_A:G_{ir},G_r \rightarrow A$ is smooth and invertible. 
So we let $g_{A,r}^{-1}:A \rightarrow G_{r}$ be the reverse of the $g_A$ with the restriction of the first $v^2$ dimension, i.e. $G_r=g_{A,r}^{-1}(A)$ and $G_r'=g_{A,r}^{-1}(A')$. 

Then we further let $g_{A,r}^{-1}=g_{B_r}\circ \kappa_r$, where $g_{B_r}: A\rightarrow \bm{B}_{r}$ denotes the process of estimating $P_B(G_r;\bm{B}_{r})$ from $A$ and $\kappa: \bm{B}\rightarrow \{0,1\}^{v^2}$ denotes the process of sampling a graph from $P_B(\cdot;\bm{B})$. We can implement $\kappa$ via the reparameterization trick from the invertible Gaussian family 
\cite{potapczynski2020invertible}, so $\kappa$ is invertible. According to A1, since $g_A$ is invertible, so $g_{A,r}^{-1}$ and $g_{B_r}$ are invertible.


Sequentially, we build a function $\bm{\mathbf{d}}:\bm{B}_r \rightarrow (0, 1)^{v^2}$, which maps $\bm{B}_r$ to uniform random variables on $(0, 1)^{v^2}$ with the help of a recursive construction known as the \textit{Darmois construction} \cite{darmois1951analyse,hyvarinen1999nonlinear}. Hence, for any random variables $\mathcal{B}_r^i$, function $d$ can be formalized as follows:
\begin{equation}
\begin{split}
    \label{equ:the2_2}d_i(\bm{B}_r):=F_i(\bm{B}_r^i|\bm{B}_r^{1:i-1})=P(\mathcal{B}_r^i\leq \bm{B}_r^i|\bm{B}_r^{1:i-1}), i=1\cdots, v^2
\end{split}
\end{equation}
in which $\bm{B}_r^i$ denotes the $i$-th element of $\bm{B}_r$ and $\bm{B}_r^{1:i-1}$ denote the elements with the indices from $1$ to $i-1$; $F$ denotes the conditional cumulative distribution function (CDF) of $\bm{B}_r^i$ given $\bm{B}_r^{1:i-1}$. Then we define a transformation as shown in Equation (\ref{equ:the2_3}).

\begin{equation}
\label{equ:the2_3}
    \tau^*= \bm{\mathbf{d}} \circ g_{B_r}^{-1}: \mathcal{A} \rightarrow (0,1)^{v^2},
\end{equation}
in which $\mathcal{A}$ denotes the space of parameters $A$.

Then we substitute $\tau^*$ into Equation (\ref{equ:the2_1}), so we have:
\begin{equation}
\label{equ:the2_4}
\begin{split}
    \mathcal{L}_r=&\mathbb{E}_{(A,A')\sim \hat{P}(A,A')}\left[||\tau(A)-\tau(A')||^2_2\right]-H(\tau(A))\\
    =&\mathbb{E}_{(A,A')\sim \hat{P}(A,A')}\left[||\bm{\mathbf{d}} \circ g_{B_r}^{-1}(A)-\bm{\mathbf{d}} \circ g_{B_r}^{-1}(A')||^2_2\right]-H(\bm{\mathbf{d}} \circ g_{B_r}^{-1}(A))\\
    =&\mathbb{E}_{(A,A')\sim \hat{P}(A,A')}\left[||\bm{\mathbf{d}}(\bm{B}_r)-\bm{\mathbf{d}}(\bm{B}_r')||^2_2\right]-H(\bm{\mathbf{d}}(\bm{B}_r)).
\end{split}
\end{equation}

Since $A'$ is generated by $A'=g_A(G_{ir}',G_r)$, $G_r=g_{A,r}^{-1}(A)=g_{A,r}^{-1}(A')=G_r'$, so $\bm{B}_r=\bm{B}_r'$ and $||\bm{\mathbf{d}}(\bm{B}_r)-\bm{\mathbf{d}}(\bm{B}_r')||^2_2$ can be minimized to zero. 
Therefore, there is a transformation $\tau\circ g_A \circ \kappa: \bm{B}_r,\bm{B}_{ir}\rightarrow(0,1)^{v^2}$ between $\hat{\bm{B}_r}$ and the true $\bm{B}_r,\bm{B}_{ir}$, which is shown as follows:
\begin{equation}
\label{equ:the2_5}
\hat{\bm{B}}_r=\tau\circ g_A\circ \kappa(\bm{B}_r,\bm{B}_{ir})=\tau\circ g_A\circ \kappa(\bm{B}_r,\bm{B}_{ir}')
\end{equation}

In the next step, we need to show how $\psi=\tau\circ g_A \circ \kappa$ can only be a function of $\bm{B}_r$ instead of depending on $\bm{B}_{ir}$.
\newline

\noindent \textbf{Step 2.}
In this step, we prove that $\tau\circ g_A \circ \kappa$ can only be a function of $\bm{B}_r$ with a contradiction. 

We first let $\psi=\tau \circ g_A$ and suppose a contradiction that $\psi(\bm{B}_r, \bm{B}_{ir})$ depends on some component of the noise latent substructures, which can be formalized as follows:
\begin{equation}
\label{equ:the2_6}
\begin{split}
    \exists i\in \{1,\cdots,v^2\},& (\bm{B}_r^*, \bm{B}_{ir}^*) \in (0,1)^{v^2}\!\times \! (0,1)^{v^2},  s.t. \frac{\partial\psi }{\partial \bm{B}_{ir,j}}(\bm{B}_r^*,\bm{B}_{ir}^*) \neq 0.
\end{split}
\end{equation}
According to Equation (\ref{equ:the2_6}), we assume that the partial derivative of $\psi$ w.r.t some variables $\bm{B}_{ir,j}$ in the latent noise substructures are non-zero at some point $(\bm{B}_r^*,\bm{B}_{ir}^*)\in (0,1)^{v^2}\!\times \! (0,1)^{v^2}$.

Since $\psi$ is composed of smooth functions, $\psi$ is smooth and has continuous partial derivatives. Therefore, $\frac{\partial \psi}{\partial \bm{B}_{ir,j}}$ must be non-zero in a neighbourhood of $(\bm{B}_r^*, \bm{B}_{ir}^*)$. In other words, $\exists \epsilon>0\quad s.t. \quad \bm{B}_{ir,j}\longmapsto\psi(\bm{B}_r*,(\bm{B}_{ir,-j}^*,\bm{B}_{ir,j}^*))$ is strictly monotonic on $(\bm{B}_{ir,j}^*-\epsilon,\bm{B}_{ir,j}^*+\epsilon)$, where $\bm{B}_{ir,-j}$ denotes the value of remaining variables except $\bm{B}_{ir,j}$.

Sequentially, we let $\mathcal{B}_r$ and $\mathcal{B}_{ir}$ be the space of $\bm{B}_r$ and $\bm{B}_{ir}$, respectively. Then we define the auxiliary function $\delta: \mathcal{B}_{r} \times \mathcal{B}_{ir} \times \mathcal{B}_{ir} \rightarrow \mathbb{R}_{\geq 0}$ as follows:
\begin{equation}
\label{equ:the2_7}
    \delta(\bm{B}_r, \bm{B}_{ir}, \bm{B}_{ir}'):=|\psi(\bm{B}_r, \bm{B}_{ir})-\psi(\bm{B}_r, \bm{B}_{ir}')|\geq 0.
\end{equation}

To obtain a contradiction to Equation (\ref{equ:the2_5}), from \textbf{Step 1} under the assumption (\ref{equ:the2_6}), it remains to show that $\delta$ from Equation (\ref{equ:the2_7}) is strictly positive with a probability greater than zero.

First, the strict monotonicity of $\quad \bm{B}_{ir,j}\longmapsto\psi(\bm{B}_r^*,(\bm{B}_{ir,-j}^*,\bm{B}_{ir,j}^*))$ implies that 
\begin{equation}
\label{equ:the2_8}
\begin{split}
    \delta(\bm{B}_r^*,(\bm{B}_{ir,-j}^*,\bm{B}_{ir,j}),({\bm{B}'}_{ir,-j}^*,{\bm{B}'}_{ir,j}))>0, \\ \forall (\bm{B}_{ir,j},{\bm{B}'}_{ir,j})\in (\bm{B}_{ir,j}^*, \bm{B}_{ir,j}^*+\epsilon) \times (\bm{B}_{ir,j}^*-\epsilon, \bm{B}_{ir,j}^*).
\end{split}
\end{equation}

Since $\delta$ is a composition of continuous function, $\delta$ is continuous. For the open set $\mathbb{R}_{>0}$ under a continuous function, the pre-images of this function are always open. Therefore, the pre-images $\mathcal{U} \in \mathcal{B}_r \times \mathcal{B}_{ir} \times \mathcal{B}_r$ of $\delta$ are also open. According to Equation (\ref{equ:the2_8}), we further have:
\begin{equation}
\label{equ:the2_9}
\begin{split}
    \{\bm{B}_r^*\} \times \left(\{\bm{B}_{ir,-j}^*\}\times(\bm{B}_{ir,j}^*, \bm{B}_{ir,j}^*+\epsilon)\right)\times  \left(\{\bm{B}_{ir,-j}^*\}\times (\bm{B}_{ir,j}^*-\epsilon, \bm{B}_{ir,j}^*)\right) \subseteq \mathcal{U},
\end{split}
\end{equation}
so $\mathcal{U}$ is not an empty set.

According to A3, for any $\bm{B}_{ir} \in \mathcal{B}_{ir}$, there is an open subset $\mathcal{O}(\bm{B}_{ir}) \subseteq \mathcal{B}_{ir}$ containing $\bm{B}_{ir}$, such that $P(\bm{B}_{ir}'|\bm{B}_{ir})>0$ for $\forall \bm{B}_{ir}' \in \mathcal{O}(\bm{B}_{ir})$. Then we define the following space:
\begin{equation}
\label{equ:the2_10}
    \mathcal{R}:=\mathcal{B}_r \times \mathcal{B}_{ir} \times \mathcal{O}(\bm{B}_{ir}),
\end{equation}
which is a topological subspace of $\mathcal{B}_r\times\mathcal{B}_{ir}\times\mathcal{B}_{ir}$. According to A2, $P(\bm{B}_r,\bm{B}_{ir})>0$. According to A3, $P(\cdot|\bm{B}_{ir})$ is fully supported on $\mathcal{O}(\bm{B}_{ir})$ for any $\bm{B}_{ir} \in \mathcal{B}_{ir}$. Therefore, the measure $\mu(\bm{B}_r,\bm{B}_{ir},\bm{B}_{ir}')$ has fully supported, strictly-positive density on $\mathcal{R}$ w.r.t. a strictly positive measure on $\mathcal{R}$.

Since $\mathcal{U}$ is open, the interaction $\mathcal{U}\cap \mathcal{R} \subseteq \mathcal{R}$ is also open in $\mathcal{R}$. Given $\bm{B}_{ir}^*$ in Equation (\ref{equ:the2_7}), there exists $\epsilon'>0$ such that $\{\bm{B}_{ir,-j}\}\times(\bm{B}_{ir,h}^*-\epsilon', \bm{B}_{ir,j}^*)\subset \mathcal{O}(\bm{B}_{ir}^*)$. For $\epsilon''=min(\epsilon,\epsilon')>0$, we have:
\begin{equation}
\label{equ:the2_11}
\begin{split}
\{\bm{B}_r^*\}\times\left(\{\bm{B}_{ir,-j}^*\}\times (\bm{B}_{ir,j}^*,\bm{B}_{ir,j}^*+\epsilon)\right) \times \left(\{\bm{B}_{ir,-j}^*\}\right. \left. \times(\bm{B}_{ir,j}^*-\epsilon'',\bm{B}_{ir,j})\right)\subset \mathcal{R}.
\end{split}
\end{equation}
Combining Equation (\ref{equ:the2_9}) and Equation (\ref{equ:the2_11}), we can find that the left-hand side of Equation (\ref{equ:the2_11}) is also a subset of $\mathcal{U}$. Therefore, the interaction $\mathcal{U}\cap\mathcal{R}$ is non-empty.

According to the aforementioned analysis, we have:
\begin{equation}
    \label{equ:the2_12}\delta(\bm{B}_{ir},\bm{B}_r,\bm{B}_{ir}')>0 \quad s.t. \quad \bm{B}_{ir},\bm{B}_r,\bm{B}_r'\in \mathcal{U}\cap\mathcal{R}.
\end{equation}
Equation (\ref{equ:the2_12}) further implies
\begin{equation}
    \label{equ:the2_13}\psi(\bm{B}_r,\bm{B}_{ir})\neq\psi(\bm{B}_r,\bm{B}_{ir}'), 
\end{equation}
which contradicts the conclusion as shown in Equation (\ref{equ:the2_5}). This concludes the proof that $\hat{\bm{B}}_r$ is related to the true $\bm{B}_r$ via a smooth transformation $\psi=\tau\circ g_A \circ \kappa$, i.e., $\hat{\bm{B}}_r=\psi(\bm{B}_r)$.
\newline

\noindent \textbf{Step 3.} Finally, we show that the mapping $\hat{\bm{B}}_r=\tau\circ g_A\circ \kappa(\bm{B}_r)$ is invertible. To this end, we make use of the following result from \cite{zimmermann2021contrastive}.

\begin{proposition}
\label{prop1}
(Proposition 5 of \cite{zimmermann2021contrastive}). Let $\mathcal{M}$ and $\mathcal{N}$ be simply connected and oriented $\mathcal{C}^1$ manifolds without boundaries and $F:\mathcal{M}\rightarrow\mathcal{N}$ be a differentiable map. Further, let the random variable $m\in \mathcal{M}$ be distributed according to $m\sim P(m)$ for a regular density function P, i.e., $0<P<\infty$. If the pushforward $P_{\#h}(m)$ of $P$ through $F$ is also a regular density, i.e, $0<P_{\#h}<\infty$, the $F$ is a bijection.
\end{proposition}

We apply this result to the simply connected and oriented $\mathcal{C}^1$ manifolds without boundaries $\mathcal{M}=\bm{B}_r$ and $\mathcal{N}=(0,1)^{v^2}$, and the smooth function $\psi: \bm{B}_r\rightarrow (0,1)^{v^2}$ which maps the random variables $\bm{B}_r$ to uniform random variables $\hat{\bm{B}}_r$. 

Since both $P_B(G_r;\bm{B}_r)$ and the uniform distribution are regular densities in the sense of Prop. \ref{prop1}, we can conclude that $\psi$ is a bijection, i.e., invertible.

Therefore, since $\hat{\bm{B}}_r$ is related to the truth parameters $\bm{B}_r$ via a smooth invertible mapping $\psi$, we can model $P_B(G_r;\bm{B}_r)$ and further $G_r$.\newline
\end{proof}

\noindent\textbf{C. Proof of Evidence Lower Bound}

\begin{equation}
    \begin{split}
        \ln P(\bm{x},k,A,y)\geq & -D_{KL}(Q(G_r|\bm{x},A)||P(G_r))-D_{KL}(Q(G_{ir}|\bm{x},A)||P(G_{ir}|G_r,\bm{s})) -D_{KL}(Q(\bm{s}|\bm{x})||P(\bm{s}|G_r))\\
        &+E_{Q(G_r|\bm{x},A)}E_{Q(G_{ir}|\bm{x},A)}\ln P(A|G_r,G_{ir})+E_{Q(G_r|\bm{x},A)}E_{Q(G_{ir}|\bm{x},A)}E_{Q(\bm{s}|\bm{x})}\ln P(\bm{x}|G_r,G_{ir},\bm{s})\\
        &+E_{Q(\bm{s}|\bm{x})}\ln P(k|\bm{s})+E_{Q(G_r|\bm{x},A)}E_{Q(\bm{s}|\bm{x})}\ln P(y|G_r,\bm{s})
    \end{split}
\end{equation}

\begin{proof}
The proof of the ELBO is composed of three steps. First, we factorize the conditional distribution according to the Bayes theorem.

\begin{equation}
    \begin{split}
        \ln P(\bm{x},k,A,y)&=\ln \frac{P(\bm{x},k,A,y,G_r,G_{ir},\bm{s})}{P(G_r,G_{ir},\bm{s}|\bm{x},k,A,y)}=\ln \frac{P(\bm{x},k,A,y,G_r,G_{ir},\bm{s})}{P(G_r|\bm{x},k,A,y)P(G_{ir},\bm{s}|\bm{x},k,A,y,G_r)}\\
        &=\ln \frac{P(\bm{x},k,A,y,G_r,G_{ir},\bm{s})}{P(G_r|\bm{x},k,A,y)P(G_{ir}|\bm{x},A,G_r)P(\bm{s}|\bm{x},k,y,G_r)}\\
    \end{split}
\end{equation}

Second, we add the expectation operator on both sides of the equation and reformalize the equation as follows:

\begin{equation}
    \begin{split}
        \ln P(\bm{x},k,A,y)
        &=D_{KL}(Q(G_r|\bm{x},A)||P(G_r|\bm{x},k,A,y)+D_{KL}(Q(G_{ir}|\bm{x},A)||P(G_{ir}|\bm{x},A,G_r))\\
        &+D_{KL}(Q(\bm{s}|\bm{x})||P(\bm{s}|\bm{x},k,y,G_r))+\ln \frac{P(\bm{x},k,A,y,G_r,G_{ir},\bm{s})}{Q(G_r|\bm{x},A)Q(G_{ir}|\bm{x},A)Q(\bm{s}|\bm{x})}\\
    \end{split}
\end{equation}

Third, we obtain the last equality with the help of $D_{KL}(\cdot|\cdot) \geq 0$

\begin{equation}
    \small
    \begin{split}
        \ln P(\bm{x},k,A,y)
        \geq& \ln \frac{P(\bm{x},k,A,y,G_r,G_{ir},\bm{s})}{Q(G_r|\bm{x},A)Q(G_{ir}|\bm{x},A)Q(\bm{s}|\bm{x})}
        = \ln \frac{P(A|G_r,G_{ir})P(\bm{x},k,y,G_r,G_{ir},\bm{s})}{Q(G_r|\bm{x},A)Q(G_{ir}|\bm{x},A)Q(\bm{s}|\bm{x})}\\
        =& \ln \frac{P(A|G_r,G_{ir})P(\bm{x}|G_r,G_{ir},\bm{s})P(k,y,G_r,G_{ir},\bm{s})}{Q(G_r|\bm{x},A)Q(G_{ir}|\bm{x},A)Q(\bm{s}|\bm{x})}
        = \ln \frac{P(A|G_r,G_{ir})P(\bm{x}|G_r,G_{ir},\bm{s})P(k|\bm{s})P(y,G_r,G_{ir},\bm{s})}{Q(G_r|\bm{x},A)Q(G_{ir}|\bm{x},A)Q(\bm{s}|\bm{x})}\\
        =& \ln \frac{P(A|G_r,G_{ir})P(\bm{x}|G_r,G_{ir},\bm{s})P(k|\bm{s})P(y|G_r,\bm{s})}{Q(G_r|\bm{x},A)Q(G_{ir}|\bm{x},A)}+\ln \frac{P(G_r,G_{ir},\bm{s})}{Q(\bm{s}|\bm{x})}\\
        =& \ln \frac{P(A|G_r,G_{ir})P(\bm{x}|G_r,G_{ir},\bm{s})P(k|\bm{s})P(y|G_r,\bm{s})}{Q(G_r|\bm{x},A)Q(G_{ir}|\bm{x},A)}+\ln\frac{P(G_{ir}|G_r,\bm{s})P(G_r,\bm{s})}{Q(\bm{s}|\bm{x})}\\
        =& \ln \frac{P(A|G_r,G_{ir})P(\bm{x}|G_r,G_{ir},\bm{s})P(k|\bm{s})P(y|G_r,\bm{s})}{Q(G_r|\bm{x},A)Q(G_{ir}|\bm{x},A)}+\ln \frac{P(G_{ir}|G_r,\bm{s})P(\bm{s}|G_r)P(G_r)}{Q(\bm{s}|\bm{x})}\\
        =&-D_{KL}(Q(G_r|\bm{x},A)||P(G_r))-D_{KL}(Q(G_{ir}|\bm{x},A)||P(G_{ir}|G_r,\bm{s}))-D_{KL}(Q(\bm{s}|\bm{x})||P(\bm{s}|G_r))\\&+E_{Q(G_r|\bm{x},A)}E_{Q(G_r|\bm{x},A)}\ln P(A|G_r,G_{ir})+E_{Q(G_r|\bm{x},A)}E_{Q(G_{ir}|\bm{x},A)}E_{Q(\bm{s}|\bm{x})}\ln P(\bm{x}|G_r,G_{ir},\bm{s})\\
        &+E_{Q(\bm{s}|\bm{x})}\ln P(k|\bm{s})+E_{Q(G_r|\bm{x},A)}E_{Q(\bm{s}|\bm{x})}\ln P(y|G_r,\bm{s})
    \end{split}
\end{equation}

\end{proof}


\end{document}


\maketitle

\section{Proof of Atom Latent Variables Identification }
\begin{theorem}
    \textbf{\textit{(Atom Latent Variables ($\bm{s}$) Identification) }}
    We follow the causal mechanism shown in Figure 2 and make the following assumptions:
    \begin{itemize}
        \item A1 (\underline{Smooth and Positive Density}): The probability density function of atom latent variables is smooth and positive, i.e. $P(\bm{s}|\bm{x}) > 0$.
        \item A2 (\underline{Conditional independence}): Conditioned on $\bm{x}$, each $s_i$ is independent of any other $s_j$ for $i,j\in [n],i\neq j$, i.e, $log P(\bm{s}|\bm{x})=\sum_i^n P(s_i|\bm{x})$.
        \item A3 (\underline{Linear independence}): For any $\bm{s} \in \mathcal{S} \subseteq \mathbb{R}^n$, where $\mathcal{S}$ is the range of $\bm{s}$ and $n$ is the dimension of $\bm{s}$, there exist $2n+1$ values of $\bm{x}$, i.e., $\bm{x}_j$ with $j=0,1,...,2n$, such that the $2n$ vectors $\textbf{\textsc{w}}(\bm{s},x_j)-\textbf{\textsc{w}}(\bm{s},x_0)$ with $j=1,...2n$, are linearly independent, where vector $\textbf{\textsc{w}}(\bm{s},\bm{x})$ is formalized as follows:
        \begin{equation}\nonumber
        \begin{split}
             \textbf{\textsc{w}}(\bm{s},x_j)=(&\frac{\partial\log P(s_0|\bm{x})}{\partial s_0},... \frac{\partial\log P(s_n|\bm{x})}{\partial s_n},...\\&\frac{\partial^2\log P(s_n|\bm{x})}{\partial^2 s_n},...\frac{\partial^2\log P(s_n|\bm{x})}{\partial^2 s_n}).
        \end{split}
        \end{equation}
    \end{itemize}
    If a learned generative model $(\hat{f_{\bm{s}}},\hat{f_h},\hat{f_n},\hat{g_A},\hat{g_y},\hat{g_k})$ assumes the same generation process shown in Figure 2 and matches the truth conditional distribution, i.e. $P(\hat{k}|\bm{x})=P(k|\bm{x})$ and $\hat{k}$ denote the estimated variables, then the identifiability of the atom latent variables $\bm{s}$ is ensured, i.e., the ground-truth atom latent variables can be learned.
\end{theorem}
\begin{proof}
Since the learned generative model $(\hat{f_s},\hat{f_h},\hat{f_n},\hat{g_A},\hat{g_y},\hat{g_k})$ assumes the same generation process shown in Figure 2 and $P(k|x)=P(\hat{k}|x)$, we have:
\begin{equation}
\begin{split}
    P(\hat{k}|x) = P(k|x) &\Longleftrightarrow P(g_k^{-1}(\hat{k})|x)|\bm{J}_{g^{-1}_k}| = P(\bm{s}|x)|\bm{J}_{g^{-1}_k}|\\ \Longleftrightarrow &P(g_k^{-1} \circ \hat{g}_k(\hat{\bm{s}})|x)|\bm{J}_{g^{-1}_k}| = P(\bm{s}|x)|\bm{J}_{g^{-1}_k}|\\ \Longleftrightarrow &P(\phi(\hat{\bm{s}})|x) =  P  (\bm{s}|x)
\end{split}
\end{equation}
in which $\hat{g}_k^{-1}$ denotes the estimated invertible function and $\phi:=g_k^{-1}\circ \hat{g}_k$ denotes the transformation between the ground truth atom latent variables and the estimated one. $|\bm{J}_{\hat{g}_k^{-1}}|$ denotes the absolute value of Jacobian matrix determinant of $\hat{g}_k^{-1}$. It is noted that $\phi:=\hat{g}_k^{-1}\circ g_k$ is invertible and $|\bm{J}_{\hat{g}_k^{-1}}|\neq 0$ since $\hat{g}_k^{-1}$ and $g_k$ are invertible.

According to the conditional independent assumption (A2), we have:
\begin{equation}
    P(\hat{\bm{s}}|x) =  \prod_{i=1}^{n}P(\hat{\bm{s}_i}|x), P(\bm{s}|x) =  \prod_{i=1}^{n}P(\bm{s}_i|x),
\end{equation}
\end{proof}
and we further have:
\begin{equation}
\begin{split}
    \log P(\hat{\bm{s}}|x) &=  \sum_{i=1}^{n}\log P(\hat{\bm{s}_i}|x)\\ \log P(\bm{s}|x) &=  \sum_{i=1}^{n} \log P(\bm{s}_i|x).
\end{split}
\end{equation}

Combining Equation (3) with Equation (1), we further have:
\begin{equation}
\begin{split}
P(\bm{s}|x)\cdot|\bm{J}_{\phi}|&=P(\hat{\bm{s}}|x)  \Longleftrightarrow\\ \sum_{i=1}^n \log P(\bm{s}_i|x) + \log |\bm{J}_{\phi}|&= \sum_{i=1}^n \log P(\hat{\bm{s}_i}|x),
\end{split}
\end{equation}
where $\bm{J}_{\phi}$ is the Jacobian matrix of the transformation associated with $\phi$. 

We employ the following notation to define the first and second derivatives of $\phi_s$ and $\log P(\bm{s}_i|x)$, respectively.
\begin{equation}
\centering
\begin{split}
    &\phi'_{i,(j)}:=\frac{\partial s_i}{\partial \hat{s_j}}, \qquad\qquad \qquad \quad \phi''_{i,(j,l)}:=\frac{\partial^2 s_i}{\partial \hat{s_j}\partial \hat{s_l}}\\
    &\eta'_i(s_i,x):=\frac{\partial\log P(s_i|x) }{\partial s_i }\qquad \; \eta''_i(s_i,x):=\frac{\partial^2 \log P(s_i|x)}{(\partial s_i)^2}
\end{split}
\end{equation}

Then we derive Equation (6) by differentiating both sides of Equation (4) twice w.r.t $\hat{s_j}$ and $\hat{s_l}$ where $j,l \in [n]$ and $j \neq l$.
\begin{equation}
\begin{split}
    \sum_{i=1}^n \left(\eta''_i(s_i,x)\cdot \phi'_{i,(j)}\cdot \phi'_{i,(l)} + \eta'_i(s_i,x)\cdot \phi''_{i,(j,l)} \right) \\+ \frac{\partial^2 \log |\bm{J}_h|}{\partial\hat{s}_j\partial\hat{s}_l}=0.
\end{split}
\end{equation}

According to the linear independence assumption, there exist $2n+1$ values of $x=\{x_0,x_1,\cdots,x_{2n}\}$, we have $2n+1$ equations by different values of $x$ in Equation (6). Subtracting each equation corresponding to $x_1, \cdots, x_{2n}$ with the equation corresponding to $x+0$, we have the following equation:
\begin{equation}
\begin{split}
    \sum_{i=1}^n\left( (\eta''_i(s_i,x_t) - \eta''_i(s_i,x_0))\cdot  \phi'_{i,(j)} \phi'_{i,(l)}\right.\\ \left. + (\eta'_i(s_i,x_t)-\eta'_i(s_i,x_0))\cdot \phi''_{i,(j,l)}\right)=0,
\end{split}
\end{equation}
in which $t=1,\cdots,2n$. We can further consider Equation (7) as homogeneous linear equations with the linearly independent restriction. As a result, the only solution is $\phi'_{i,(j)} \phi'_{i,(l)}=0$ and $ \phi''_{i,(j,l)}=0$ for $i=1,\cdot,n$ and $j,l\in[n],j \neq l$. 

Then we have the Jacobian of $h$ shown as follows:
\begin{equation}
    \bm{J}_h = 
    \begin{bmatrix}
        \frac{\partial s_1}{\partial \hat{s_1}} \cdots \frac{\partial s_1}{\partial \hat{s_n}}\\
        \vdots \quad \ddots \quad  \vdots\\
        \frac{\partial {s_n}}{\partial \hat{s_1}} \cdots \frac{\partial s_n}{\partial \hat{s_n}}.
    \end{bmatrix}
\end{equation}

Since $\phi'_{i,(j)} \phi'_{i,(l)}=0$, there is at most one element $j\in[n]$, resulting in $\phi'_{i,(j)}\neq 0$. Therefore, there is at most one non-zero element in each row of the Jacobian matrix $\bm{J}_\phi$. 
According to the inverse function theorem, since $\phi$ is invertible, the inverse matrix of $J_\phi$ exists and equals $J_{\phi^{-1}}$. Since $J_\phi$ is invertible, $J_\phi$ is full-rank, implying that there is exactly one non-zero element in each row of the Jacobian matrix $\bm{J}_\phi$ and in each column of the Jacobian matrix $\bm{J}_\phi$. 
Therefore, given any true atom latent variable $s_i$, there exist a corresponding estimated variable $\hat{s_j}$ and an invertible function $\phi: \mathbb{R}\rightarrow\mathbb{R}$, such that $\hat{s_j}=\phi(s_i)$.


\section{Proof of Semantic Latent Substructures}

\begin{theorem}
\textbf{\textit{(Semantic Latent Substructure distribution ($\bm{B}_h$) Identification) }}
    We follow the causal generation process shown in Figure 3 and make the following assumptions:
    \begin{itemize}
        \item  A1 (\underline{Smooth and Invertible Generation Process}): $g_A:G_n, G_h \rightarrow A$ is smooth and invertible with a smooth inverse.
        \item A2 (\underline{Smooth, Continuous and Positive Density}): $P(G_h, G_n)$ is a smooth, continuous density with $P(G_h, G_n)>0$ almost everywhere.
        \item A3 (\underline{Smooth and Positive Conditional Probability}) The conditional probability density function $P(G_n'|G_n)$ is smooth \textit{w.r.t} both $G_n$ and $G_n'$; for any $G_n$, $P(\cdot|G_n)>0$ in some open, non-empty subset containing $G_n$.
        \item A4 (\underline{Identical data generation process}) A learned generative model $(\hat{f_{\bm{s}}},\hat{f_h},\hat{f_n},\hat{g_A},\hat{g_y},\hat{g_k})$ assumes the same generation process shown in Equation (1).
    \end{itemize}
    Let $v$ be the node number and let $\tau: A \rightarrow (0,1)^{v\times v}$ be any smooth function. Then $\bm{B}_h$ can be identified by minimizing the following restriction:
    \begin{equation}
        \mathcal{L}_r=\mathbb{E}_{(A,A')\sim \hat{P}(A,A')}\left[||\tau(A)-\tau(A')||^2_2\right]-H(\tau(A)),
    \end{equation}
    where $A'$ denotes the augmented sample from the proposed causal generation process and $H(\cdot)$ denotes the differential entropy of the random variables $\tau_x(A)$. When $\bm{B}_h$ is identified, we can obtain $G_h$ by sampling from $P(G_h;\bm{B}_h)$.
\end{theorem}

\begin{proof}

According to the causal generation process shown in Figure 2, we consider $G_n, G_h$ as two $v^2$-dimension latent variables, i.e. $G_n, G_h \in (0,1)^{v^2}$, where $v$ is the node number of molecular structure. 
So we can assume that the truth semantic latent graphs follow the multivariate Bernoulli distribution with the parameters of $\bm{B}_h \in (0,1)^{v^2}$, i.e. $G_h \sim P(G_h;\bm{B}_h)$, where $v$ is the node number of $G_h$. 
Similarly, we assume the noise latent substructures $G_n$ are samples from $P(G_n;\bm{B}_n)$, where $\bm{B}_n$ are the parameters. 
The proof of the identification of $\bm{B}_h$ consists of  the following three main steps:
\begin{itemize}
    \item First, we demonstrate that $\hat{\bm{B}_h}$ extracted by a smooth function by minimizing Equation (9) is related to the true $\bm{B}_h,\bm{B}_n$ through a smooth mapping $\psi$, i.e., $\hat{\bm{B}_h}=\psi(\bm{B}_h,\bm{B}_n)$. 
    \item Second, we show that $\hat{\bm{B}_h}=\psi(\bm{B}_h,\bm{B}_n)$ can only depend on the true $\bm{B}_h$ and not on $\bm{B}_n$, i.e., $\hat{\bm{B}_h}=\psi(\bm{B}_h)$
    \item Based on the result from \cite{zimmermann2021contrastive}, we show that $\psi$ must be a bijection.
\end{itemize}


\noindent\textbf{Step 1.} 
In this step, we aim to prove that there is a transformation between $\hat{\bm{B}_h}$ and the true $\bm{B}_h, \bm{B}_n$. 
According to the A1, $g_A:G_n,G_h \rightarrow A$ is smooth and invertible. 
So we let $g_{A,h}^{-1}:A \rightarrow G_h$ be the reverse of the $g_A$ with the restriction of the first $v^2$ dimension, i.e. $G_h=g_{A,h}^{-1}(A)$ and $G_h'=g_{A,h}^{-1}(A')$. 

Then we further let $g_{A,h}^{-1}=g_{B_h}\circ \kappa_h$, where $g_{B_h}: A\rightarrow \bm{B}_{h}$ denotes the process of estimating $P(G_h;\bm{B}_{h})$ from $A$ and $\kappa: \bm{B}\rightarrow \{0,1\}^{v^2}$ denotes the process of sampling a graph from $P(\cdot;\bm{B})$. We can implement $\kappa$ via the reparameterization trick from the invertible Gaussian family 
\cite{potapczynski2020invertible}, so $\kappa$ is invertible. According to A1, since $g_A$ is invertible, so $g_{A,h}^{-1}$ and $g_{B_h}$ are invertible.


Sequentially, we build a function $\bm{\mathbf{d}}:\bm{B}_h \rightarrow (0, 1)^{v^2}$, which maps $\bm{B}_h$ to uniform random variables on $(0, 1)^{v^2}$ with the help of a recursive construction known as the \textit{Darmois construction} \cite{darmois1951analyse,hyvarinen1999nonlinear}. Hence, for any random variables $\mathcal{B}_h^i$, function $d$ can be formalized as follows:
\begin{equation}
\begin{split}
    d_i(\bm{B}_h):=F_i(\bm{B}_h^i|\bm{B}_h^{1:i-1})=P(\mathcal{B}_h^i\leq \bm{B}_h^i|\bm{B}_h^{1:i-1}), \\i=1\cdots, v^2
\end{split}
\end{equation}
in which $\bm{B}_h^i$ denotes the $i$-th element of $\bm{B}_h$ and $\bm{B}_h^{1:i-1}$ denote the elements with the indices from $1$ to $i-1$; $F$ denotes the conditional cumulative distribution function (CDF) of $\bm{B}_h^i$ given $\bm{B}_h^{1:i-1}$. Then we define a transformation as shown in Equation (11).

\begin{equation}
    \tau^*= \bm{\mathbf{d}} \circ g_{B_h}^{-1}: \mathcal{A} \rightarrow (0,1)^{v^2},
\end{equation}
in which $\mathcal{A}$ denotes the space of parameters $A$.

Then we substitute $\tau^*$ into Equation (9), so we have:
\begin{equation}
\begin{split}
    \mathcal{L}_r=&\mathbb{E}_{(A,A')\sim \hat{P}(A,A')}\left[||\tau(A)-\tau(A')||^2_2\right]-H(\tau(A))\\
    =&\mathbb{E}_{(A,A')\sim \hat{P}(A,A')}\left[||\bm{\mathbf{d}} \circ g_{B_h}^{-1}(A)-\bm{\mathbf{d}} \circ g_{B_h}^{-1}(A')||^2_2\right]\\&-H(\bm{\mathbf{d}} \circ g_{B_h}^{-1}(A))\\
    =&\mathbb{E}_{(A,A')\sim \hat{P}(A,A')}\left[||\bm{\mathbf{d}}(\bm{B}_h)-\bm{\mathbf{d}}(\bm{B}_h')||^2_2\right]-H(\bm{\mathbf{d}}(\bm{B}_h)).
\end{split}
\end{equation}

Since $A'$ is generated by $A'=g_A(G_n',G_h)$, $G_h=g_{A,h}^{-1}(A)=g_{A,h}^{-1}(A')=G_h'$, so $\bm{B}_h=\bm{B}_h'$ and $||\bm{\mathbf{d}}(\bm{B}_h)-\bm{\mathbf{d}}(\bm{B}_h')||^2_2$ can be minimized to zero. 
Therefore, there is a transformation $\tau\circ g_A \circ \kappa: \bm{B}_h,\bm{B}_n\rightarrow(0,1)^{v^2}$ between $\hat{\bm{B}_h}$ and the true $\bm{B}_h,\bm{B}_n$, which is shown as follows:
\begin{equation}
    \hat{\bm{B}_h}=\tau\circ g_A\circ \kappa(\bm{B}_h,\bm{B}_n)=\tau\circ g_A\circ \kappa(\bm{B}_h,\bm{B}_n')
\end{equation}

In the next step, we need to show how $\psi=\tau\circ g_A \circ \kappa$ can only be a function of $\bm{B}_h$ instead of depending on $\bm{B}_n$.
\newline

\noindent \textbf{Step 2.}
In this step, we prove that $\tau\circ g_A \circ \kappa$ can only be a function of $\bm{B}_h$ with a contradiction. 

We first let $\psi=\tau \circ g_A$ and suppose a contradiction that $\psi(\bm{B}_h, \bm{B}_n)$ depends on some component of the noise latent substructures, which can be formalized as follows:
\begin{equation}
\begin{split}
    \exists i\in \{1,\cdots,v^2\},& (\bm{B}_h^*, \bm{B}_n^*) \in (0,1)^{v^2}\!\times \! (0,1)^{v^2}, \\ s.t. &\frac{\partial\psi}{\partial \bm{B}_{n,i}}(\bm{B}_h^*,\bm{B}_n^*) \neq 0.
\end{split}
\end{equation}
According to Equation (15), we assume that the partial derivative of $\psi$ w.r.t some variables $\bm{B}_{n,i}$ in the latent noise substructures are non-zero at some point $(\bm{B}_h^*,\bm{B}_n^*)\in (0,1)^{v^2}\!\times \! (0,1)^{v^2}$.

Since $\psi$ is composed of smooth functions, $\psi$ is smooth and has continuous partial derivatives. Therefore, $\frac{\partial \psi}{\partial \bm{B}_{n,i}}$ must be non-zero in a neighbourhood of $(\bm{B}_h^*, \bm{B}_n^*)$. In other words, $\exists \epsilon>0\quad s.t. \quad \bm{B}_{n,i}\longmapsto\psi(\bm{B}_h*,(\bm{B}_{n,-i}^*,\bm{B}_{n,i}^*))$ is strictly monotonic on $(\bm{B}_{n,i}^*-\epsilon,\bm{B}_{n,i}^*+\epsilon)$, where $\bm{B}_{n,-i}$ denotes the value of remaining variables except $\bm{B}_{n,i}$.

Sequentially, we let $\mathcal{B}_h$ and $\mathcal{B}_n$ be the space of $\bm{B}_h$ and $\bm{B}_n$, respectively. Then we define the auxiliary function $\delta: \mathcal{B}_h \times \mathcal{B}_n \times \mathcal{B}_n \rightarrow \mathbb{R}_{\geq 0}$ as follows:
\begin{equation}
    \delta(\bm{B}_h, \bm{B}_n, \bm{B}_n'):=|\psi(\bm{B}_h, \bm{B}_n)-\psi(\bm{B}_h, \bm{B}_n')|\geq 0.
\end{equation}

To obtain a contradiction to Equation (13), from \textbf{Step 1} under the assumption (14), it remains to show that $\delta$ from Equation (15) is strictly positive with a probability greater than zero.

First, the strict monotonicity of $\quad \bm{B}_{n,i}\longmapsto\psi(\bm{B}_h*,(\bm{B}_{n,-i}^*,\bm{B}_{n,i}^*))$ implies that 
\begin{equation}
\begin{split}
    \delta(\bm{B}_h^*,(\bm{B}_{n,-i}^*,\bm{B}_{n,i}),({\bm{B}'}_{n,-i}^*,{\bm{B}'}_{n,i}))>0, \\ \forall (\bm{B}_{n,i},{\bm{B}'}_{n,i})\in (\bm{B}_{n,i}^*, \bm{B}_{n,i}^*+\epsilon) \times (\bm{B}_{n,i}^*-\epsilon, \bm{B}_{n,i}^*).
\end{split}
\end{equation}

Since $\delta$ is a composition of continuous function, $\delta$ is continuous. For the open set $\mathbb{R}_{>0}$ under a continuous function, the pre-images of this function are always open. Therefore, the pre-images $\mathcal{U} \in \mathcal{B}_h \times \mathcal{B}_n \times \mathcal{B}_h$ of $\delta$ are also open. According to Equation (16), we further have:
\begin{equation}
\begin{split}
    \{\bm{B}_h^*\} \times \left(\{\bm{B}_{n,-i}^*\}\times(\bm{B}_{n,i}^*, \bm{B}_{n,i}^*+\epsilon)\right)\times \\ \left(\{\bm{B}_{n,-i}^*\}\times (\bm{B}_{n,i}^*-\epsilon, \bm{B}_{n,i}^*)\right) \subseteq \mathcal{U},
\end{split}
\end{equation}
so $\mathcal{U}$ is not an empty set.

According to A3, for any $\bm{B}_n \in \mathcal{B}_n$, there is an open subset $\mathcal{O}(\bm{B}_n) \subseteq \mathcal{B}_n$ containing $\bm{B}_n$, such that $P(\bm{B}_n'|\bm{B}_n)>0$ for $\forall \bm{B}_n' \in \mathcal{O}(\bm{B}_n)$. Then we define the following space:
\begin{equation}
    \mathcal{R}:=\mathcal{B}_h \times \mathcal{B}_n \times \mathcal{O}(\bm{B}_n),
\end{equation}
which is a topological subspace of $\mathcal{B}_h\times\mathcal{B}_n\times\mathcal{B}_n$. According to A2, $P(\bm{B}_h,\bm{B}_n)>0$. According to A3, $P(\cdot|\bm{B}_n)$ is fully supported on $\mathcal{O}(\bm{B}_n)$ for any $\bm{B}_n \in \mathcal{B}_n$. Therefore, the measure $\mu(\bm{B}_h,\bm{B}_n,\bm{B}_n')$ has fully supported, strictly-positive density on $\mathcal{R}$ w.r.t. a strictly positive measure on $\mathcal{R}$.

Since $\mathcal{U}$ is open, the interaction $\mathcal{U}\cap \mathcal{R} \subseteq \mathcal{R}$ is also open in $\mathcal{R}$. Given $\bm{B}_n^*$ in Equation (15), there exists $\epsilon'>0$ such that $\{\bm{B}_{n,-i}\}\times(\bm{B}_{n,i}^*-\epsilon', \bm{B}_{n,i}^*)\subset \mathcal{O}(\bm{B}_n^*)$. For $\epsilon''=min(\epsilon,\epsilon')>0$, we have:
\begin{equation}
\begin{split}
\{\bm{B}_h^*\}\times\left(\{\bm{B}_{n,-i}^*\}\times (\bm{B}_{n,i}^*,\bm{B}_{n,i}^*+\epsilon)\right) \times \left(\{\bm{B}_{n,-i}^*\}\right.\\ \left. \times(\bm{B}_{n,i}^*-\epsilon'',\bm{B}_{n,i})\right)\subset \mathcal{R}.
\end{split}
\end{equation}
Combining Equation (17) and Equation (19), we can find  that the left-hand side of Equation (19) is also a subset of $\mathcal{U}$. Therefore, the interaction $\mathcal{U}\cap\mathcal{R}$ is non-empty.

According to the aforementioned analysis, we have:
\begin{equation}
    \delta(\bm{B}_n,\bm{B}_h,\bm{B}_n')>0 \quad s.t. \quad \bm{B}_n,\bm{B}_h,\bm{B}_h'\in \mathcal{U}\cap\mathcal{R}.
\end{equation}
Equation (20) further implies
\begin{equation}
    \psi(\bm{B}_h,\bm{B}_n)\neq\psi(\bm{B}_h,\bm{B}_n'), 
\end{equation}
which contradicts the conclusion as shown in Equation (13). This concludes the proof that $\hat{\bm{B}}_h$ is related to the true $\bm{B}_h$ via a smooth transformation $\psi=\tau\circ g_A \circ \kappa$, i.e., $\hat{\bm{B}_h}=\psi(\bm{B}_h)$.
\newline

\noindent \textbf{Step 3.} Finally, we show that the mapping $\hat{\bm{B}_h}=\tau\circ g_A\circ \kappa(\bm{B}_h)$ is invertible. To this end, we make use of the following result from \cite{zimmermann2021contrastive}.

\begin{proposition}
\label{prop1}
(Proposition 5 of \cite{zimmermann2021contrastive}). Let $\mathcal{M}$ and $\mathcal{N}$ be simply connected and oriented $\mathcal{C}^1$ manifolds without boundaries and $F:\mathcal{M}\rightarrow\mathcal{N}$ be a differentiable map. Further, let the random variable $m\in \mathcal{M}$ be distributed according to $m\sim P(m)$ for a regular density function P, i.e., $0<P<\infty$. If the pushforward $P_{\#h}(m)$ of $P$ through $F$ is also a regular density, i.e, $0<P_{\#h}<\infty$, the $F$ is a bijection.
\end{proposition}

We apply this result to the simply connected and oriented $\mathcal{C}^1$ manifolds without boundaries $\mathcal{M}=\bm{B}_h$ and $\mathcal{N}=(0,1)^{v^2}$, and the smooth function $\psi: \bm{B}_h\rightarrow (0,1)^{v^2}$ which maps the random variables $\bm{B}_h$ to uniform random variables $\hat{\bm{B}_h}$.

Since both $P(G_h;\bm{B}_h)$ and the uniform distribution are regular densities in the sense of Prop. \ref{prop1}, we can conclude that $\psi$ is a bijection, i.e., invertible.

Therefore, since $\hat{\bm{B}_h}$ is related to the truth parameters $\bm{B}_h$ via a smooth invertible mapping $\psi$, we can model $P(G_h;\bm{B}_h)$ and further $G_h$.
\end{proof}

\section{Proof of Evidence Lower Bound}
\begin{equation}
    \begin{split}
        \ln &P(\bm{x},k,A,y)\\
        &\geq -D_{KL}(Q(G_h|\bm{x},A)||P(G_h))\\
        &-D_{KL}(Q(G_n|\bm{x},A)||P(G_n|G_h,\bm{s}))\\
        &-D_{KL}(Q(\bm{s}|\bm{x})||P(\bm{s}|G_h))\\
        &+E_{Q(G_h|\bm{x},A)}E_{Q(G_n|\bm{x},A)}\ln P(A|G_h,G_n)\\
        &+E_{Q(G_h|\bm{x},A)}E_{Q(G_n|\bm{x},A)}E_{Q(\bm{s}|\bm{x})}\ln P(\bm{x}|G_h,G_n,\bm{s})\\
        &+E_{Q(\bm{s}|\bm{x})}\ln P(k|\bm{s})\\
        &+E_{Q(G_h|\bm{x},A)}E_{Q(\bm{s}|\bm{x})}\ln P(y|G_h,\bm{s})
    \end{split}
\end{equation}

\begin{proof}
The proof of the ELBO is composed of three steps. First, we factorize the conditional distribution according to the Bayes theorem.

\begin{equation}
    \begin{split}
        \ln &P(\bm{x},k,A,y)\\
        &=\ln \frac{P(\bm{x},k,A,y,G_h,G_n,\bm{s})}{P(G_h,G_n,\bm{s}|\bm{x},k,A,y)}\\
        &=\ln \frac{P(\bm{x},k,A,y,G_h,G_n,\bm{s})}{P(G_h|\bm{x},k,A,y)P(G_n,\bm{s}|\bm{x},k,A,y,G_h)}\\
        &=\ln \frac{P(\bm{x},k,A,y,G_h,G_n,\bm{s})}{P(G_h|\bm{x},k,A,y)P(G_n|\bm{x},A,G_h)P(\bm{s}|\bm{x},k,y,G_h)}\\
    \end{split}
\end{equation}

Second, we add the expectation operator on both sides of the equation and reformalize the equation as follows:

\begin{equation}
    \begin{split}
        \ln P(&\bm{x},k,A,y)\\
        =&D_{KL}(Q(G_h|\bm{x},A)||P(G_h|\bm{x},k,A,y)\\
        &+D_{KL}(Q(G_n|\bm{x},A)||P(G_n|\bm{x},A,G_h))\\
        &+D_{KL}(Q(\bm{s}|\bm{x})||P(\bm{s}|\bm{x},k,y,G_h))\\
        &+\ln \frac{P(\bm{x},k,A,y,G_h,G_n,\bm{s})}{Q(G_h|\bm{x},A)Q(G_n|\bm{x},A)Q(\bm{s}|\bm{x})}\\
    \end{split}
\end{equation}

Third, we obtain the last equality with the help of $D_{KL}(\cdot|\cdot) \geq 0$

\begin{equation}
    \small
    \begin{split}
        &\ln P(\bm{x},k,A,y)\\
        \geq& \ln \frac{P(\bm{x},k,A,y,G_h,G_n,\bm{s})}{Q(G_h|\bm{x},A)Q(G_n|\bm{x},A)Q(\bm{s}|\bm{x})}\\
        =& \ln \frac{P(A|G_h,G_n)P(\bm{x},k,y,G_h,G_n,\bm{s})}{Q(G_h|\bm{x},A)Q(G_n|\bm{x},A)Q(\bm{s}|\bm{x})}\\
        =& \ln \frac{P(A|G_h,G_n)P(\bm{x}|G_h,G_n,\bm{s})P(k,y,G_h,G_n,\bm{s})}{Q(G_h|\bm{x},A)Q(G_n|\bm{x},A)Q(\bm{s}|\bm{x})}\\
        =& \ln \frac{P(A|G_h,G_n)P(\bm{x}|G_h,G_n,\bm{s})P(k|\bm{s})P(y,G_h,G_n,\bm{s})}{Q(G_h|\bm{x},A)Q(G_n|\bm{x},A)Q(\bm{s}|\bm{x})}\\
        =& \ln \frac{P(A|G_h,G_n)P(\bm{x}|G_h,G_n,\bm{s})P(k|\bm{s})P(y|G_h,\bm{s})}{Q(G_h|\bm{x},A)Q(G_n|\bm{x},A)}\\
        &+\ln \frac{P(G_h,G_n,\bm{s})}{Q(\bm{s}|\bm{x})}\\
        =& \ln \frac{P(A|G_h,G_n)P(\bm{x}|G_h,G_n,\bm{s})P(k|\bm{s})P(y|G_h,\bm{s})}{Q(G_h|\bm{x},A)Q(G_n|\bm{x},A)}\\
        &+\ln\frac{P(G_n|G_h,\bm{s})P(G_h,\bm{s})}{Q(\bm{s}|\bm{x})}\\
        =& \ln \frac{P(A|G_h,G_n)P(\bm{x}|G_h,G_n,\bm{s})P(k|\bm{s})P(y|G_h,\bm{s})}{Q(G_h|\bm{x},A)Q(G_n|\bm{x},A)}\\
        &+\ln \frac{P(G_n|G_h,\bm{s})P(\bm{s}|G_h)P(G_h)}{Q(\bm{s}|\bm{x})}\\
        =&-D_{KL}(Q(G_h|\bm{x},A)||P(G_h))\\
        &-D_{KL}(Q(G_n|\bm{x},A)||P(G_n|G_h,\bm{s}))\\
        &-D_{KL}(Q(\bm{s}|\bm{x})||P(\bm{s}|G_h))\\
        &+E_{Q(G_h|\bm{x},A)}E_{Q(G_n|\bm{x},A)}\ln P(A|G_h,G_n)\\
        &+E_{Q(G_h|\bm{x},A)}E_{Q(G_n|\bm{x},A)}E_{Q(\bm{s}|\bm{x})}\ln P(\bm{x}|G_h,G_n,\bm{s})\\
        &+E_{Q(\bm{s}|\bm{x})}\ln P(k|\bm{s})\\
        &+E_{Q(G_h|\bm{x},A)}E_{Q(\bm{s}|\bm{x})}\ln P(y|G_h,\bm{s})
    \end{split}
\end{equation}

\end{proof}